\documentclass{article} 
\usepackage{iclr2027_conference_preprint,times}

\usepackage{amsmath,amsfonts,bm}

\def\eqref#1{equation~\ref{#1}}

\def\1{\bm{1}}

\DeclareMathAlphabet{\mathsfit}{\encodingdefault}{\sfdefault}{m}{sl}
\SetMathAlphabet{\mathsfit}{bold}{\encodingdefault}{\sfdefault}{bx}{n}

\usepackage{hyperref}
\usepackage{url}
\hypersetup{
colorlinks,
linkcolor=blue,
citecolor=blue,
anchorcolor=red,
}
\usepackage{url}            
\usepackage{booktabs}       
\usepackage{amsfonts}       
\usepackage{nicefrac}       
\usepackage{microtype}      
\usepackage{xcolor}         
\usepackage{graphicx}
\usepackage{amssymb}
\usepackage{amsmath}
\usepackage{amsthm}
\usepackage{bm}
\usepackage{algorithmicx,algorithm}
\usepackage[noend]{algpseudocode}
\usepackage{wrapfig}
\usepackage{threeparttable}
\usepackage{multirow}
\usepackage{colortbl}
\usepackage{enumitem}
\usepackage{circledsteps}
\usepackage{booktabs}
\usepackage{tcolorbox}
\usepackage{listings}
\usepackage{makecell}
\theoremstyle{plain}
\newtheorem{definition}{Definition}
\newtheorem{proposition}{Proposition}
\theoremstyle{definition}
\newtheorem{assumption}{Assumption}

\definecolor{darkred}{rgb}{0.8, 0, 0}
\definecolor{myPurple}{HTML}{7030A0}
\definecolor{MistBlue}{HTML}{D1E4F0}
\definecolor{SkyBlue}{HTML}{7FA9CD}
\definecolor{SteelBlue}{HTML}{4060A3}
\definecolor{myRed}{HTML}{BD3C33}
\definecolor{myOrange}{HTML}{E5935B}
\definecolor{myYellow}{HTML}{FEEAAC}
\definecolor{upColor}{RGB}{220,20,120}

\title{Structure-Internalized Rule Language Model for Faithful Knowledge Graph Reasoning}

\author{%
Xingrui Zhuo\textsuperscript{1,2}, Jiapu Wang\textsuperscript{3}, Manzong Huang\textsuperscript{1,2}, Gongqing Wu\textsuperscript{1,2}\thanks{Corresponding author}, Xindong Wu\textsuperscript{1,2}\footnotemark[\value{footnote}]\\
\textsuperscript{1}Key Laboratory of Knowledge Engineering with Big Data (Hefei University of Technology), \\Ministry of Education, China\\
\textsuperscript{2}School of Computer Science and Information Engineering, Hefei University of Technology, China\\
\textsuperscript{3}Nanjing University of Science and Technology, China\\
\texttt{zxr@mail.hfut.edu.cn, jiapu.wang@njust.edu.cn,}\\ \texttt{manzonghuang@mail.hfut.edu.cn, wugq@hfut.edu.cn, xwu@hfut.edu.cn}
}

\iclrfinalcopy 
\begin{document}

\maketitle
\begin{abstract}
Knowledge Graph Reasoning (KGR) aims to discover latent facts by leveraging the structural evidence available in KGs, posing a challenge to the structural semantic understanding capability of KGR models. 
Recent studies have demonstrated that Large Language Models (LLMs) can achieve remarkable progress on KGR tasks via flexible in-context learning. 
However, the inherent representation inconsistency between KG structural context and LLM parametric knowledge remains inadequately addressed. This limitation prevents LLMs from effectively perceiving reasoning evidence that aligns with KG constraints, which undermines both the effectiveness and faithfulness of reasoning. We refer to this problem as \textit{reasoning evidence perception drift} of LLMs over KGs. 
To address this problem, we propose a Structure-Internalized Rule Language Model (SIRLM), which centers on structural rule generation to couple the parametric learning of structural knowledge with the faithfulness evaluation of reasoning logic, enabling LLMs to anchor tightly to KG-grounded evidence. 
Specifically, we first design a Structure-Internalized Rule Generator (SIRG), which incorporates an in-context learning block augmented with a structural relation memory to coordinate structural and parametric knowledge. 
Furthermore, we equip SIRG with a KG tokenizer based on structural invariance learning and a neuro-symbolic reasoner based on rule-constrained message propagation. These components provide SIRG with learnable structural representations and faithful rule-execution feedback, respectively. 
Our SIRLM can be seamlessly integrated into standard LLM training paradigms, such as SFT and GRPO. 
Extensive experiments against 17 state-of-the-art KGR methods on 36 datasets demonstrate the significant superiority of SIRLM\footnote{Our source codes are available at \url{https://github.com/lazyloafer/SIRLM}}.
\end{abstract}
\section{Introduction}
\label{Introduction}
\vspace{-2mm}
Knowledge Graphs (KGs)~\citep{KGSurvey} represent real-world facts through structural triplets composed of entities and relations, maintaining a unified framework for storing and querying large-scale knowledge. However, KGs are generally incomplete, which leaves many potential facts unobserved. Consequently, Knowledge Graph Reasoning (KGR)~\citep{KGRSurvey} has been proposed to leverage the available KG context to infer missing facts, which provides more sufficient evidence for knowledge-driven applications~\citep{Relink,rog}.
\par Traditional KGR methods mainly focus on knowledge embedding~\citep{RotatE} and logical rule learning~\citep{DRUM}, which achieve efficient fact prediction via structural representation learning and interpretable reasoning~\citep{GaussianPath}, respectively. However, these methods generally rely on specific contextual instances within a KG. When confronted with an unfamiliar KG, they often require complex rule mining or model retraining, which makes it difficult for these methods to generalize to more challenging KGR scenarios~\citep{ANet,InGram}.
\par Recently, Large Language Models (LLMs)~\citep{LLMSurvey}, pre-trained on massive natural language corpora, have demonstrated outstanding performance in various reasoning tasks. Latest studies have confirmed that LLMs can achieve breakthrough success on KGR tasks~\citep{KGGPT,LLM-DA,QIPP}. With the support of flexible in-context learning capabilities, LLMs integrate inherent parametric knowledge with structural context of KGs and exhibit remarkable knowledge emergence~\citep{Roadmap}, thereby discovering new facts in KGs that are not explicitly observed. More importantly, compared with traditional KGR methods, LLM-based approaches show stronger transferability when facing previously unseen KG structures in inductive reasoning settings~\citep{PROLINK,MKGL,KRLM}.
\begin{figure*}[t!]
  \centering
  \includegraphics[width=\linewidth]{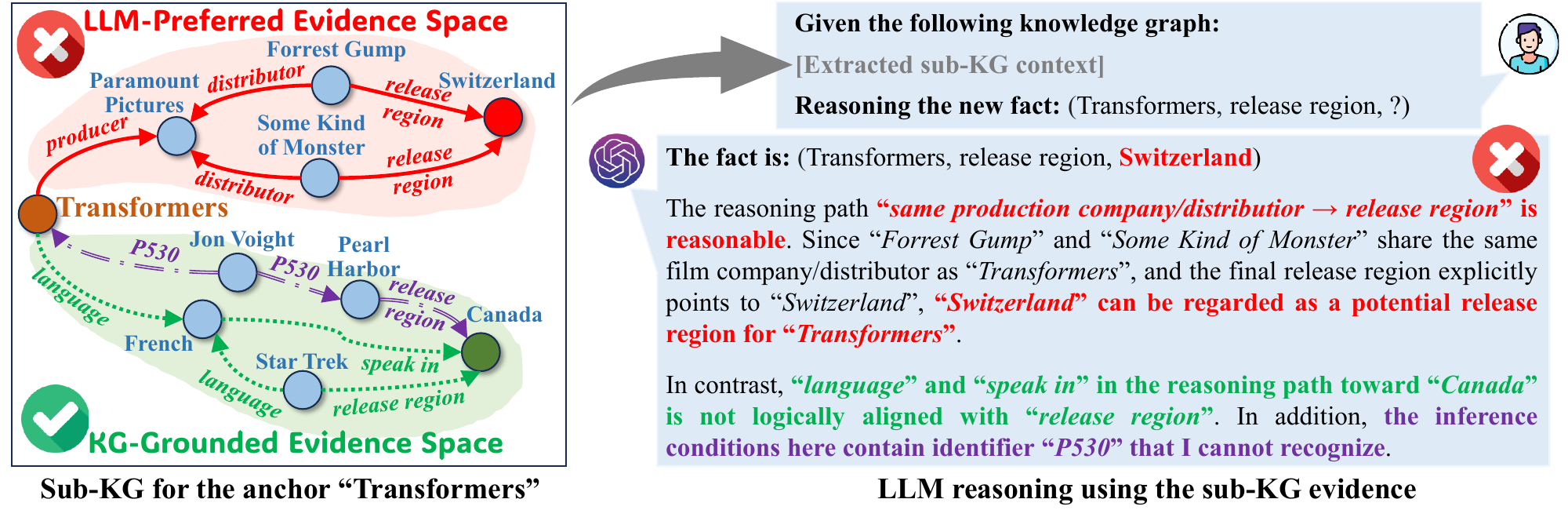}
\vspace{-7mm}
  \caption{\label{fig_challenge}Case of reasoning evidence perception drift of LLMs over KGs. LLM select incorrect reasoning evidence based on its semantic preference, resulting in its inability to follow KG-grounded reasoning logic to arrive at correct reasoning conclusions.}
\vspace{-6mm}
\end{figure*}
\par Despite significant accomplishments, existing LLM-based KGR methods still struggle to address the representation inconsistency between structural and parametric knowledge~\citep{KGLLMGap}, which constrains the ability of LLMs to grasp structural KG context during reasoning. This limitation is reflected in the fact that most LLMs tend to select reasoning evidence from a KG based on their pre-trained semantic preferences, rather than strictly following the structural constraints induced from the KG. As a result, LLMs are inclined to generate reasoning logic that appears semantically plausible but is difficult to ground in a KG, leading to the reasoning bias and the faithfulness degradation of results. We refer to this issue as the \textbf{\textit{reasoning evidence perception drift}} problem of LLMs over KGs, meaning that inconsistencies in knowledge representation hinder LLMs from identifying reasonable evidence from KGs, thereby limiting the effectiveness of LLM reasoning.
\par Figure \ref{fig_challenge} illustrates the harm that such a perception shift brings to LLMs. In this case, LLMs tend to accept the assertion “\textit{films produced or distributed by the same company are likely to be released in the same region}” based on their semantic preferences. As a result, the evidence space is anchored to the sub-KG marked in \textcolor{red}{\textbf{red}} in Figure \ref{fig_challenge}. However, the structural rule that can be explicitly induced from the KG, namely \textit{release\_region(X, Y) = language(X, Z) $\land$ speak\_in(X, Z)}, is instead misjudged as a weakly related evidence space (marked in \textcolor{green}{\textbf{green}} in Figure \ref{fig_challenge}). In addition, for relation paths represented by anonymized identifiers (marked in \textcolor{myPurple}{\textbf{purple}} in Figure \ref{fig_challenge}), the absence of natural language semantics often leads LLMs to deliberately ignore this structural evidence that could potentially support correct reasoning. This preference for pre-trained semantic cues and the neglect of structural context cause the LLM's reasoning process to deviate from the correct KG-grounded evidence space, ultimately producing erroneous results.
\par To address the aforementioned limitation, we propose a Structure-Internalized Rule Language Model (SIRLM), which enables LLMs to accurately perceive KG-grounded evidence for faithful reasoning. 
Overall, SIRLM consists of a KG tokenizer, a Structure-Internalized Rule Generator (SIRG) and a neuro-symbolic reasoner. In this framework, SIRG serves as the core module to generate structurally grounded rules. Therefore, we propose an in-context learning block based on a Structural Relation Memory (SRM) mechanism, which allows SIRG to align pre-trained knowledge with structural context during rule generation. Complementing SIRG, the KG tokenizer is built upon the Structural Invariance Learning (SIL) principle~\citep{ULTRA,MOTIF,TRIX}, which encodes KG elements into universal structure representations, enabling LLMs to grasp them in the form of new tokens without specific textual annotations. Furthermore, we design the neuro-symbolic reasoner as a rule executor based on Rule-Constraint Message Propagation (RCMP), which feeds back a rule faithfulness metric composed of reasoning conclusions to SIRG, forming a closed-loop optimization that reinforces structurally consistent reasoning.
\par In terms of model training, SIRLM can seamlessly integrate with standard Supervised Fine-Tuning (SFT) and post-training frameworks of LLMs. This capability endows SIRLM with significant transferable advantages across a wide range of cross-scenario KGR tasks.
\par Our main contributions can be summarized as follows:
\begin{enumerate}[left=0pt]
\item[$\bullet$] We propose a structure-internalized rule language model to mitigate the reasoning evidence perception drift of LLMs over KGs, enabling reliable LLM-based KG reasoning.
\item[$\bullet$] We design a structure-internalized rule generator with a structural relation memory to coordinate LLM knowledge with KG context, together with a SIL-based KG tokenizer and a RCMP-based rule reasoner for universal structure representation learning and faithful rule execution feedback. 
\item[$\bullet$] SIRLM can seamlessly integrate with standard LLM training frameworks such as SFT and GRPO, achieving generalizable reasoning across unknown KGs.
\item[$\bullet$] Extensive experimental results conducted on 36 datasets demonstrate that SIRLM exhibits remarkable reasoning capabilities in both transductive and inductive KGR scenarios.
\end{enumerate}
\vspace{-4mm}
\section{Related Work}
\label{Related Work}
\vspace{-2mm}
KGR studies have long focused on mining latent facts from KGs, providing interpretable knowledge support for downstream tasks. Traditional KGR methods adopt structural representation learning techniques by projecting entities and relations into appropriate embedding spaces, such as Euclidean~\citep{TransE}, complex~\citep{complex}, or manifold spaces~\citep{ManifoldE}, to infer latent associations between entities and relations. To ensure interpretability in reasoning, some studies introduce rule-based approaches that leverage techniques such as rule mining~\citep{RNNLogic,DRUM}, relational path retrieval~\citep{MINERVA}, or Markov decision processes~\citep{PRA} to perform KGR while providing explicit reasoning paths.
\par As KGs scale up, early embedding-based and rule-based methods struggle to adapt to more complex knowledge structures. Consequently, researchers have incorporated Graph Neural Networks (GNN) into KGR~\citep{NodePiece,GraIL,ANet}. Methods such as NBFNet~\citep{NBFNet}, RED-GNN~\citep{RED-GNN}, and InGram~\citep{InGram} exploit the message-passing mechanism of GNNs to capture fine-grained structural context about entities and relations, providing stronger evidence for comprehensive reasoning. Building on this, some studies~\citep{ULTRA,MOTIF,TRIX,KGICL} focus on improving the generalization of KGR models to out-of-distribution graph structures. These approaches aim to discover reusable minimal subgraph patterns within KGs and learn universal structural representations of entities and relations, \textit{i.e.}, Structural Invariance Learning (SIL)~\citep{ULTRA}. This technique enables unknown KG elements to be projected into learned structural patterns, thereby enhancing model performance in complex reasoning scenarios.
\par In recent years, LLM-based KGR has emerged as a new research focus~\citep{PROLINK,PPMT,KRLM}. By leveraging the strong knowledge emergence and in-context learning capabilities of LLMs~\citep{Roadmap}, these approaches uncover deeper facts in KGs. For example, methods such as KICGPT~\citep{KICGPT} and ChatRule~\citep{ChatRule} adopt the \textit{LLM-planning and KG-retrieval} paradigm to prompt LLMs in performing structural reasoning over given sub-KGs. Meanwhile, approaches such as KoPA~\citep{KoPA}, MKGL~\citep{MKGL}, and KRLM~\citep{KRLM} rely on fine-tuning techniques to inject structural knowledge into the pre-trained space of LLMs, aligning their parametric knowledge with the structural knowledge of KGs. Despite significant development, existing LLM-based methods still face the challenge in handling the widely observed KG-LLM representation gaps. Under this condition, structural knowledge that is beneficial for reasoning in KGs is often regarded by LLMs as irrelevant evidence, leading to a phenomenon known as \textit{evidence perception drift}. 
\par The aforementioned challenge have prompted us to propose a more refined mechanism for internalizing structural knowledge, which uses structural rule generation as a feedback to guide LLMs in perceiving grounded structural logic over KGs, thereby improving model reasoning.
\vspace{-2mm}
\section{Preliminaries}
\label{Preliminaries}
\vspace{-2mm}
In this section, we introduce the background and key conceptions of our study.
\vspace{-2mm}
\subsection{Knowledge Graph Reasoning}
\label{Preliminaries 1}
\vspace{-2mm}
We define a KG as $\mathcal{G}=(\mathcal{E},\mathcal{R},\mathcal{T})$, where $\mathcal{E}$ and $\mathcal{R}$ denote the sets of entities and relations, respectively. $\mathcal{T}=\{<e_h,r_q,e_t>|e_h,e_t\in\mathcal{E},r_q\in\mathcal{R}\}$ is the set of triples. Each triple represents a factual statement indicating that the head entity $e_h$ and the tail entity $e_t$ are connected via relation $r_q$.
\par KGR aims to infer new facts of the form $<e_h,r_q,?>\notin\mathcal{T}$. Given a training KG $\mathcal{G}_{tr}=(\mathcal{E}_{tr},\mathcal{R}_{tr},\mathcal{T}_{tr})$ and a inference KG $\mathcal{G}_{inf}=(\mathcal{E}_{inf},\mathcal{R}_{inf},\mathcal{T}_{inf})$, KGR tasks are typically categorized into \textbf{transductive} ($\mathcal{E}_{tr}=\mathcal{E}_{inf}$, $\mathcal{R}_{tr}=\mathcal{R}_{inf}$, and $\mathcal{T}_{tr}\neq\mathcal{T}_{inf}$) and \textbf{inductive} setting 
($\mathcal{E}_{tr}\neq\mathcal{E}_{inf}$ or $\mathcal{R}_{tr}\neq\mathcal{R}_{inf}$, and $\mathcal{T}_{tr}\neq\mathcal{T}_{inf}$).
The inductive setting requires KGR models to generalize to previously unseen entities or relations, thereby ensuring strong extrapolation and generalization capabilities.
\vspace{-2mm}
\subsection{Structural Invariance Learning of KG}
\label{Preliminaries 2}
\vspace{-2mm}
SIL of KG derives universal structural representations for unseen entities and relations through a relational graph and a Neural Bellman-Ford Network (NBFNet)~\citep{NBFNet}.
\par \textbf{The relational graph} is designed to characterize the structural motifs between relations in a KG. Let the relational graph be defined as $\mathcal{G}_r=(\mathcal{R},\mathcal{R}^*,\mathcal{T}^*)$, where $\mathcal{R}$ corresponds to the set of relations in the original KG $\mathcal{G}$ and serves as the node set of $\mathcal{G}_r$. $\mathcal{R}^*$ denotes the set of motif edges to connect relation nodes. The detailed definition of motif edges is provided in \textbf{Appendix~\ref{appendixRelationalGraph}}.
\par Under this inductive paradigm, any relation in an arbitrary KG can be mapped into $\mathcal{G}_r$. By aggregating these motif edges, we can obtain the universal structural representation for any relation.
\par \textbf{NBFNet} is a GNN that generalizes path formulations over graph nodes. Given a query triple $<e_h,r_q,?>$, we first perform $\text{NBFNet}(U,\bm{R}^*, \mathcal{G}_r)$, an $N$-layer NBFNet, over $\mathcal{G}_r$:
\begin{equation}
\label{eq1}
\small
\begin{aligned}
\text{NBFNet}(U,\bm{R}^*, \mathcal{G}_r)=
\begin{cases}
\bm{r}^{(0)}_{j|q} = \text{\textsc{Init}}(r_j, U),\hspace{2mm}\text{where }U=\{<r_q, \mathbf{1}^d>\}, r_q,r_j\in\mathcal{R},\\
\bm{r}^{(n)}_{j|q} = \text{\textsc{Up}}\Big(\bm{r}_{j|q}^{(n-1)},\text{\textsc{Agg}}\big(\{\text{\textsc{Msg}}(\bm{r}_{z|q}^{(n-1)},\bm{r}^*)|r_z\in\mathcal{N}_{r^*}(r_j)\}|r^*\in\mathcal{R}^*,\bm{r}^*\in\bm{R}^*\big)\Big)\\
\end{cases}
\end{aligned}
\end{equation}
where $U$ is a tuple set that stores conditions that satisfy non-zero initialization nodes, $\textbf{1}^d$ is a $d$-dimensional full-one vector, $n\in[1, N]$ is the layer index, and $\bm{R}^*\in\mathbb{R}^{|\mathcal{R}^*|\times d}$ is the randomly initialized embeddings of motif edges. $\text{\textsc{Init}}(\cdot)$, $\text{\textsc{Msg}}(\cdot)$, $\text{\textsc{Agg}}(\cdot)$, and $\text{\textsc{Up}}(\cdot)$ are an indicator function, the DistMult message propagation~\citep{DistMult}, the summation aggregation operation, and a Multi-Layer Perceptron (MLP), respectively. The details of Eq. (\ref{eq1}) is provided in Eq. (\ref{eq17}).
\par According to Eq. (\ref{eq1}), we obtain the structural representations of all relations conditioned on $r_q$, denoted as $\bm{R}_{|q}=\text{NBFNet}(U,\bm{R}^*, \mathcal{G}_r)$. Based on this, we further compute the structural representations of all entities conditioned on $e_h$ as $\bm{E}_{|h}$:
\begin{equation}
\label{eq2}
\small
\begin{aligned}
\bm{E}_{|h}=\text{NBFNet}(\{<e_h, \bm{r}_{q|q}^{(N)}>\},\bm{R}_{|q}, \mathcal{G}), \hspace{2mm}\bm{r}_{q|q}^{(N)}\in\bm{R}_{|q}.
\end{aligned}
\end{equation}
The details of Eq. (\ref{eq2}) is provided in Eq. (\ref{eq18}).
\par Finally, we compute scores for the candidate tail entities of the query triple $<e_h,r_q,?>$:
\begin{equation}
\label{eq3}
\small
\begin{aligned}
s_{\text{KG}}^{(i)}=\mathcal{S}_{\text{KG}}(\bm{e}_{i|h}^{(N)}||\bm{r}_{q|q}^{(N)})\hspace{2mm}i\in[1, |\mathcal{E}|],
\end{aligned}
\end{equation}
where $\bm{e}_{i|h}\in\bm{E}_{|h}$ and $\mathcal{S}_{\text{KG}}:\mathbb{R}^{2d}\rightarrow\mathbb{R}^{1}$ is a scoring function.
\vspace{-2mm}
\subsection{In-context learning and next-token prediction of LLM}
\label{Preliminaries 3}
\vspace{-2mm}
\textbf{In context learning:} Let $X$ be an instruction and $\text{TKN}_{\text{LLM}}$ denote the pre-trained token embedding table of LLM. By looking up $X$ in $\text{TKN}_{\text{LLM}}$, the instruction is transformed into a sequence $\text{TKN}_{\text{LLM}}[X]=\bm{X}\in\mathbb{R}^{L\times F}$ with $L$ $F$-dimensional token embeddings. The LLM aggregates each token embedding $\bm{x}_i\in\bm{X}$ together with the embeddings of all preceding tokens through the self-attention mechanism to perform in-context learning, thereby producing the hidden state of $\bm{x}_i$, defined as $\bm{h}_i=\text{InCon}(\bm{x}_i,\bm{X}_{\le i})$:
\vspace{-2mm}
\begin{equation}
\label{eq4}
\small
\begin{aligned}
\text{InCon}(\bm{x}_i,\bm{X}_{\le i})=\text{FFN}\Big(\text{softmax}\big(\frac{f_Q(\bm{x}_i)[f_K(\bm{X}_{\le i})]^{\text{T}}}{\sqrt{F}}\big)f_V(\bm{X}_{\le i})\Big),
\end{aligned}
\end{equation}
where $f_{\{Q,K,V\}}(\cdot):\mathbb{R}^F\rightarrow\mathbb{R}^F$ are pre-trained linear layers and $\text{FFN}(\cdot):\mathbb{R}^F\rightarrow\mathbb{R}^F$
is a pre-trained Feed Forward Network (FFN).
\par \textbf{Next-token prediction:} Based on Eq. (\ref{eq4}), the new token generated by LLM conditioned on the input instruction $X$ can be expressed as:
\begin{equation}
\label{eq5}
\small
\begin{aligned}
\hat{\bm{x}}_{z+1} \leftarrow \underset{\hat{\bm{x}}\in \text{TKN}_{\text{LLM}}}{\text{arg}\max}P_{\bm{\Theta}}(\hat{\bm{x}}|\text{InCon}(\hat{\bm{x}}_z,[\bm{X}:\hat{\bm{X}}_{\le z}])),
\end{aligned}
\end{equation}
where $\hat{\bm{X}}_{\le z}\in\mathbb{R}^{z\times F}$ is the embeddings of the first $z$ generated tokens and $[:]$ denotes a row-wise concatenation operation. $\text{InCon}(\hat{\bm{x}}_z,[\bm{X}:\hat{\bm{X}}_{\le z}])$ represents the hidden state of the $z$-th generated token $\hat{\bm{x}}_z\in\hat{\bm{X}}_{\le z}$, which is projected through a MLP module parameterized by $\bm{\Theta}$ to produce the generation probability of the ($z$+1)-th token.
\begin{figure*}[t!]
\vspace{-10pt}
  \centering
  \includegraphics[width=\linewidth]{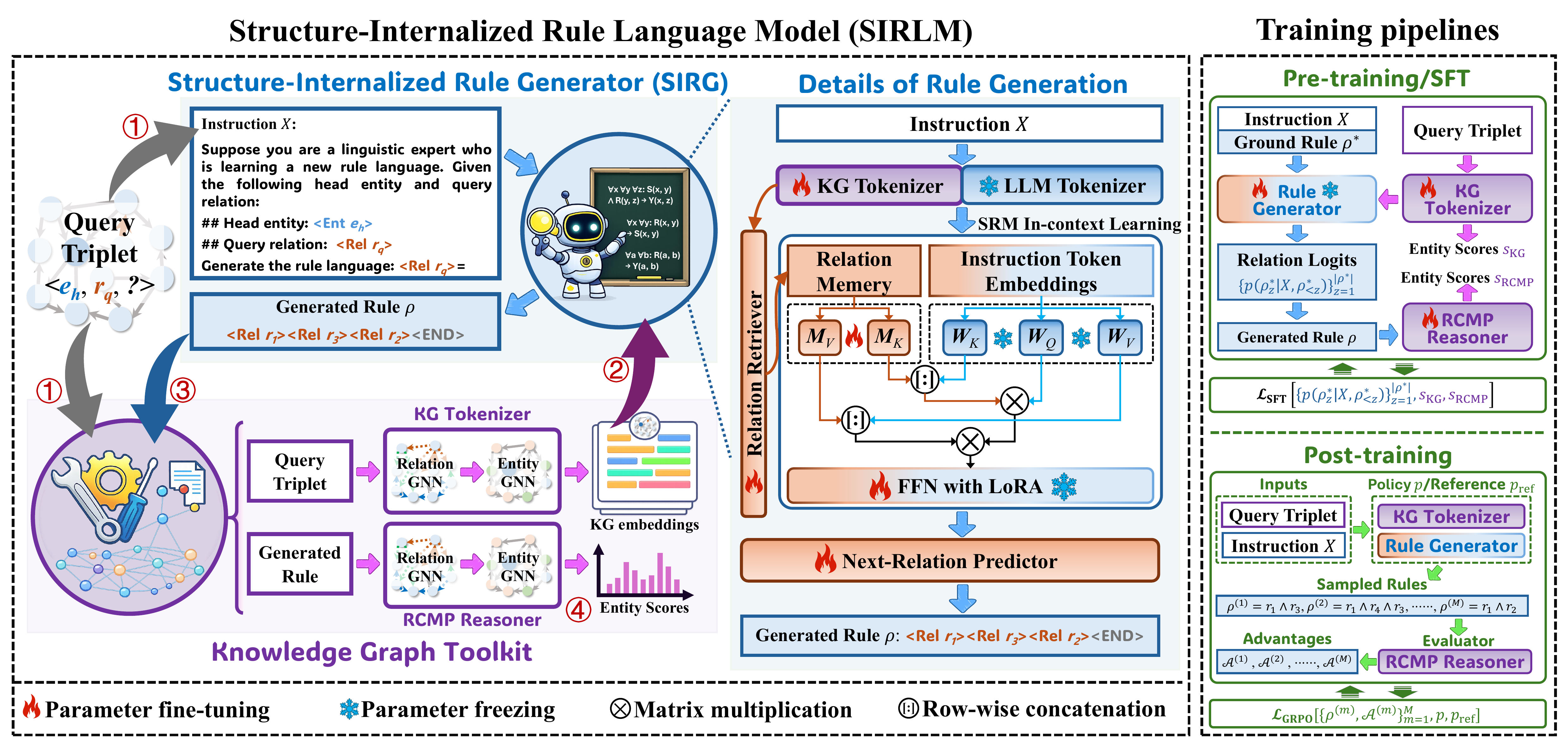}
\vspace{-6mm}
  \caption{\label{fig_model}Overall architecture and training framework of SILRM. Given a query triplet, we first \textcolor{darkred}{\scalebox{0.8}{\Circled{1}}} convert it into an instruction $X$ and input the triplet into the KG tokenizer to obtain structural embeddings of entities and relations. These structural representations are then \textcolor{darkred}{\scalebox{0.8}{\Circled{2}}} tokenized along with the LLM pre-trained tokenizer for $X$. Next, SIGR reads the aforementioned tokens and \textcolor{darkred}{\scalebox{0.8}{\Circled{3}}} generates structural rules for the RCMP reasoner. Afterwards, \textcolor{darkred}{\scalebox{0.8}{\Circled{4}}} the reasoner executes rules and performs neuro-symbolic reasoning, ultimately obtaining scores for candidate entities.}
\vspace{-4mm}
\end{figure*}
\vspace{-2mm}
\section{Methodology}
\label{Methodology}
\vspace{-2mm}
In this section, we elaborate on the proposed SILRM in detail. As illustrated in Figure~\ref{fig_model}, SILRM consists of a structure-internalized rule generator and a KG toolkit. In the following, we describe our method from four aspects: the construction of the query instruction and the KG tokenizer (\textbf{Section~\ref{Methodology1}}), LLM-based in-context learning and rule generation (\textbf{Section~\ref{Methodology2}}), and rule reasoning (\textbf{Section~\ref{Methodology3}}). Finally, we introduce how our framework seamlessly integrates with LLM SFT and post-training strategies in \textbf{Section~\ref{Methodology4}}.
\vspace{-2mm}
\subsection{Construction of Query instruction and KG tokenizer}
\label{Methodology1}
\vspace{-3mm}
Given a query triplet $<e_h,r_q,?>$, we design a query instruction $X$ suitable for generating structural rules using LLMs, whose detailed schema is provided in \textbf{Appendix~\ref{appendixInstruction}}.
\par In the instruction schema, strings of the form $<\text{Ent ID}>$ and $<\text{Rel ID}>$ are treated as indivisible structural token identifiers, which denote specific entities and relations in a KG, respectively. $<\text{END}>$ serves as a rule body termination. These structural tokens are embedded via a KG tokenizer $\text{TKN}_{\text{KG}}$ and are jointly processed with the pre-trained LLM tokenizer $\text{TKN}_{\text{LLM}}$ to encode the remaining text in instruction $X$. To unify the processing of structural and textual tokens in $X$, the tokenizer of the rule generator can be expressed as the combination of $\text{TKN}_{\text{KG}}$ and $\text{TKN}_{\text{LLM}}$:
\begin{equation}
\label{eq6}
\small
\begin{aligned}
\text{TKN}_{\text{RG}}=[\text{TKN}_{\text{LLM}}:\text{TKN}_{\text{KG}}],\hspace{2mm}
\text{TKN}_{\text{KG}}=f_{\text{up}}([\bm{r}_{\text{end}}:\bm{E}_{|h}:\bm{R}_{|q}]),
\end{aligned}
\end{equation}
where $\bm{r}_{\text{end}}\in\mathbb{R}^d$ is a randomly initialized embedding for $<\text{END}>$, $f_{\text{up}}(\cdot):\mathbb{R}^d\rightarrow\mathbb{R}^F$ is a trainable linear layer, and $\bm{R}_{|q}\in\mathbb{R}^{|\mathcal{R}|\times d}$ and $\bm{E}_{|h}\in\mathbb{R}^{|\mathcal{E}|\times d}$ are the structural embeddings of entities and relations obtained by Eqs. (\ref{eq1}) and (\ref{eq2}), respectively. It is important to note that the pre-trained $\text{TKN}_{\text{LLM}}$ remain frozen and only $\text{TKN}_{\text{KG}}$ is trained. This allows LLMs to perceive the structural representations of the KG through newly introduced tokens.
\vspace{-2mm}
\subsection{In-context Learning and Rule Generation}
\label{Methodology2}
\vspace{-2mm}
To enable LLMs to generate executable rules, we design an in-context learning strategy with the SRM mechanism and subsequently perform autoregressive relation token prediction, which ensures that the generated rules can be grounded in the corresponding KG, providing interpretable evidence for subsequent knowledge reasoning.
\par \textbf{In-context learning with structural relation memory:} To guide the LLM toward structure-internalized rule generation, we adapt the pre-trained attention layer by integrating a structural relation memory mechanism. Given a query triplet $<e_h,r_q,?>$, we first compute a relevance score for each relation $r_j\in\mathcal{R}$ and select the top-$K$ relation as the memory $\bm{R}_{\text{mem}}$:
\begin{equation}
\label{eq7}
\small
\begin{aligned}
\bm{R}_{\text{mem}}={\{\bm{r}_k\in\bm{R}_{|q}|s_{\text{rel}}^{(k)}\in \text{TopK}({\{s_{\text{rel}}^{(j)}\}}_{j=1}^{|\mathcal{R}|})\}}_{k=1}^K,\hspace{2mm}s_{\text{rel}}^{(j)}=\mathcal{S}_{\text{rel}}(\bm{e}_{h|h}^{(N)}||\bm{r}_{q|q}^{(N)}||\bm{r}_{j|q}^{(N)}),
\end{aligned}
\end{equation}
Here, $\bm{e}_{h|h}^{(N)}\in\bm{E}_{|h}$; $\bm{r}_{q|q}^{(N)},\bm{r}_{j|q}^{(N)}\in\bm{R}_{|q}$; and $\mathcal{S}_{\text{rel}}(\cdot):\mathbb{R}^{3d}\rightarrow\mathbb{R}^d$ is a MLP scorer. 
\par Let the instruction $X$ be tokenized by $\text{TKN}_{\text{RG}}$ into the embedding sequence $\bm{X}={\{\bm{x}_i\in\mathbb{R}^d\}}_{i=1}^L$, the in-context learning operation in Eq. (\ref{eq4}) can be improved to:
\begin{equation}
\label{eq8}
\small
\begin{aligned}
\hspace{-2mm}\text{InCon}(\bm{x}_i,\bm{R}_{\text{mem}},\bm{X}_{\le i})=\text{FFN}\Big(\text{softmax}\big(\frac{f_Q(\bm{x}_i)[f_K(\bm{X}_{\le i}):m_K(\bm{R}_{\text{mem}})]^{\text{T}}}{\sqrt{F}}\big)[f_V(\bm{X}_{\le i}):m_V(\bm{R}_{\text{mem}})]\Big),
\end{aligned}
\end{equation}
where $\text{FFN}(\cdot)$ incorporates a LoRA fine-tuning block~\citep{LoRA} in practical modeling and $m_{\{K,V\}}(\cdot):\mathbb{R}^d\rightarrow\mathbb{R}^F$ are trainable linear layers. Eq. (\ref{eq8}) allows the hidden state of last token $x_L$ to fully integrate the multi-modal token representations in $X$ and the relation memory, providing semantically complete activation for the subsequent rule generation. Its effectiveness analysis is provided in \textbf{Appendix~\ref{appendixSRMAttention}}.
\par \textbf{Next-relation prediction for rule generation:} Let the hidden state of the last token be $\bm{h}_L=\text{InCon}(\bm{x}_L,\bm{R}_{\text{mem}},\bm{X})$, we formulate rule generation as an autoregressive next-relation prediction process according to Eq. (\ref{eq5}). First, the distribution $P_{\bm{\Theta}}$ in Eq. (\ref{eq5}) can be concretized as:
\begin{equation}
\label{eq9}
\small
\begin{aligned}
P_{\bm{\Theta}}(\bm{r}|\bm{h}_L)=f_{\bm{\Theta}}(\bm{h}_L||\bm{r}),\hspace{2mm}\bm{r}\in[\bm{r}_{\text{end}}:\bm{R}_{|q}],
\end{aligned}
\end{equation}
where $f_{\bm{\Theta}}(\cdot):\mathbb{R}^{F+d}\rightarrow\mathbb{R}^1$ is a trainable MLP module, which scores each candidate relation embedding to determine its likelihood as the next generated token.
\par Then, relation tokens are generated sequentially in an autoregressive manner:
\begin{equation}
\label{eq10}
\small
\begin{aligned}
\bm{r}_{z+1}\leftarrow \underset{\bm{r}\in [\bm{r}_{\text{end}}:\bm{R}_{|q}]}{\text{arg}\max}P_{\bm{\Theta}}(\bm{r}|\bm{h}_{L+z}),\hspace{2mm}
\bm{h}_{L+z}=\text{InCon}\Big(f_{\text{up}}(\bm{r}_z),\bm{R}_{\text{mem}},[\bm{X}:f_{\text{up}}(\bm{R}_{\le z})]\Big),
\end{aligned}
\end{equation}
where $z\in[0,\varepsilon-1]$, $\varepsilon$ denotes the maximum rule length, $\bm{R}_{\le z}$ consists of the structural embeddings of the previously generated $z$ relations, and $f_{\text{up}}(\cdot):\mathbb{R}^d\rightarrow\mathbb{R}^F$ is a linear layer defined in Eq. (\ref{eq6}). The generation process terminates when $\bm{r}_{z+1}=\bm{r}_{\text{end}}$ or $z+1=\varepsilon$.
\par Through Eq. (\ref{eq10}), we can obtain the rule body $\rho=\bigwedge_{z=1}^{\varepsilon} r_z$ together with their hidden states ${\{\bm{h}_{L+z}\}}_{z=1}^{\varepsilon}$, which are used for subsequent rule reasoning for $<e_h, r_q, ?>$.
\vspace{-2mm}
\subsection{Reasoning with Rule-constraint Message Propagation}
\label{Methodology3}
\vspace{-2mm}
Unlike conventional SIL methods that initialize relational graph nodes solely with a full-one vector of the query relation, our RCMP mechanism incorporates the hidden states ${\{\bm{h}_{L+z}\}}_{z=1}^{\varepsilon}$ of the rule path $\bigwedge_{z=1}^{\varepsilon} r_z$ generated in Eq. (\ref{eq10}) as additional initialization signals. Specifically, we modify the initialization step of relation nodes in Eq. (\ref{eq1}) as follows:
\begin{equation}
\label{eq11}
\small
\begin{aligned}
\hat{\bm{R}}_{|q}=\text{NBFNet}(\{<r_q, \mathbf{1}^d>\}\cup{\{<r_z, f_{\text{down}}(\bm{h}_{L+z})>\}}_{z=1}^\varepsilon,\hat{\bm{R}}^*),
\end{aligned}
\end{equation}
where $f_{\text{down}}(\cdot):\mathbb{R}^F\rightarrow\mathbb{R}^d$ is a trainable linear layer. Similar to Eq. (\ref{eq1}), $\hat{\bm{R}}^*$ is a randomly initialized embeddings of motif edges. Eq. (\ref{eq11}) transforms the implicit semantic signals of the rule sequence into explicit structural constraints, which enhances the contextual representation and structural discriminability of the query relation. The details of Eq. (\ref{eq11}) is provided in Eq. (\ref{eq19}) and a theoretical analysis of RCMP is provided in \textbf{Appendix~\ref{appendixRCMP}}.
\par Based on Eq. (\ref{eq11}), we obtain $\hat{\bm{R}}_{|q}$ as the relation representations. Then, we compute entity representations $\hat{\bm{E}}_{|h}$ and score the candidate tail entities of $<e_h, r_q, ?>$:
\begin{equation}
\label{eq12}
\small
\begin{aligned}
\hat{\bm{E}}_{|h}=\text{NBFNet}(\{<e_h, \hat{\bm{r}}_{q|q}^{(N)}>\},\hat{\bm{R}}_{|q}, \mathcal{G}),
\end{aligned}
\end{equation}
\begin{equation}
\label{eq13}
\small
\begin{aligned}
s_{\text{RCMP}}^{(i)}=\mathcal{S}_{\text{RCMP}}(\hat{\bm{e}}_{i|h}^{(N)}||\hat{\bm{r}}_{q|q}^{(N)}), \hat{\bm{e}}_{i|h}\in\hat{\bm{E}}_{|h}^{(N)}, \hat{\bm{r}}_{q|q}^{(N)}\in\hat{\bm{R}}_{|q},
\end{aligned}
\end{equation}
where $\mathcal{S}_{\text{RCMP}}(\cdot):\mathbb{R}^{2d}\rightarrow\mathbb{R}^1$ is a MLP scorer. The details of Eq. (\ref{eq12}) is provided in Eq. (\ref{eq20}).
\vspace{-2mm}
\subsection{Training Frameworks}
\label{Methodology4}
\vspace{-2mm}
\textbf{During the pre-training or SFT phase}, SIRLM incorporates the KG tokenizer and RCMP reasoner into the standard LLM training framework with supervision signals. Let $\rho^*$ and $e_t$ represent the ground truth rule and target entity of $<e_h, r_q, ?>$, respectively. The pre-training/SFT loss for SIRLM can be expressed as:
\begin{equation}
\label{eq14}
\small
\begin{aligned}
\mathcal{L}_{\text{SFT}}=-\Big[\sum\limits_{z=1}^{|\rho^*|}\log{\big(p(\rho^*_z|X,\rho^*_{<z})\big)}+\log{(sc_{\text{KG}}^{(t)}sc_{\text{RCMP}}^{(t)})}-\frac{1}{|\mathcal{N}|}\sum\limits_{e_i\in\mathcal{N}}\log{\big((1-sc_{\text{KG}}^{(i)})(1-sc_{\text{RCMP}}^{(i)})\big)}\Big],
\end{aligned}
\end{equation}
where $p(\rho^*_z|X,\rho^*_{<z})$ represents the generation probability of the $z$-th relation token given the instruction $X$ and the first $z$-1 generated relation tokens. $sc_{\text{KG}}^{(t)}$ and $sc_{\text{RCMP}}^{(t)}$ are calculated by Eqs. (\ref{eq7}) and (\ref{eq13}), respectively, and $\mathcal{N}$ is the negative target set of $<e_h, r_q, ?>$.
\par \textbf{SIRLM is compatible with the post-training paradigm of existing LLMs}. When transferring the pre-trained SIRLM on a downstream KG, we use the GRPO framework to train SIRLM:
\begin{equation}
\label{eq15}
\small
\begin{aligned}
\hspace{-2mm}\mathcal{L}_{\text{GRPO}}=-\frac{1}{\sum\limits_{m=1}^M|\rho^{(m)}|}\sum\limits_{z=1}^{|\rho^{(m)}|}\Big[p(\rho^{(m)}_z|X,\rho^{(m)}_{<z})\mathcal{A}^{(m)}-\beta\text{KL}\big(p(\rho^{(m)}_z|X,\rho^{(m)}_{<z})|p_{\text{ref}}(\rho^{(m)}_z|X,\rho^{(m)}_{<z})\big)\Big],
\end{aligned}
\end{equation}
where $p$ and $p_{\text{ref}}$ represent the policy model and the reference model initialized by the pre-trained SIRLM, respectively. ${\{\rho^{(m)}\}}_{m=1}^M$ denotes the $M$ rules generated by SIRLM for $<e_h, r_q, ?>$ through a sampling strategy and $\beta$ is a fixed weight. The advantage $\mathcal{A}^{(m)}=\frac{\mathcal{W}^{(m)}-\text{mean}({\{\mathcal{W}^{(m)}\}}_{m=1}^M)}{\text{std}({\{\mathcal{W}^{(m)}\}}_{m=1}^M)}$ of each rule is represented as the standardization of the corresponding reward $\mathcal{W}^{(m)}$:
\begin{equation}
\label{eq16}
\small
\begin{aligned}
\mathcal{W}^{(m)}=\frac{\mathbb{I}\big(\phi(\rho^{(m)})\big)}{\text{Rank}\big(sc_{\text{RCMP}}^{(t)}\big)}-\mathbb{I}\big(\neg\phi(\rho^{(m)})\big),
\hspace{2mm}\phi(\rho)=
\begin{cases}
\text{True}, \rho_{<|\rho|}\in\mathcal{R} \text{ and } \rho_{|\rho|}=<\text{END}>\\
\text{False}, \text{else}
\end{cases}
\end{aligned}.
\end{equation}
Here, $\text{Rank}\big(sc_{\text{RCMP}}^{(t)}\big)$ denotes the score ranking of $e_t$ among all entities when $<e_h, r_q, ?>$ and the generated rule $\rho^{(m)}$ are given, where the score $sc_{\text{RCMP}}^{(t)}$ is obtained by Eq. (\ref{eq13}). $\phi(\cdot)$ is a boolean function used to examine the structure compliance of the generated rule. It should be noted that at this time, the RCMP reasoner, as a evaluation model in GRPO, does not participate in the parameter optimization of Eq. (\ref{eq15}).
\vspace{-3mm}
\section{Experiments}
\label{Experiments}
\vspace{-2mm}
In this section, we demonstrate SIRLM from the following research question: \textbf{RQ1}. Can SIRLM achieve significant performance across a wide range of transductive and inductive KGR scenarios? \textbf{RQ2}. Does the core modules of SIRLM play a crucial role in effect enhancement, including the Multi-Modal Query Instrction (MMQI) containing textual and structural tokens, the SRM mechanism in the SIRG module, and the RCMP reasoner? \textbf{RQ3}. Is SIRLM sensitive to hyperparameter settings, including the number of layers $N$ of NBFNet, the size $K$ of structural relation memory, and the number of rule samples $M$ of GRPO? \textbf{RQ4}. Is SIRLM applicable to different LLM backbones? 
\vspace{-2mm}
\subsection{Datasets, Baselines, and Experimental Settings}
\label{Setting}
\vspace{-2mm}
\textbf{Datasets}. To comprehensively evaluate SIRLM, we conduct experiments on 36 KGR datasets including four \textbf{Transductive} datasets (FB15k-237~\citep{FB15k237}, WN18RR~\citep{WN18RR}, CoDEx-M~\citep{CoDEx}, and NELL995~\citep{DeepPathNELL995}), 12 \textbf{Ind}uctive \textbf{E}ntity (\textbf{IndE}) datasets~\citep{GraIL} created from FB15k237, WN18RR, and NELL995, and 20 \textbf{Full}y \textbf{Ind}uctive (\textbf{FullInd}) datasets created from FB15k237, NELL995, Wikidata68K~\citep{RAILD}, and MTDEA~\citep{MTDEA}. Detailed dataset descriptions are provided in \textbf{Appendix~\ref{appendixDataset}}.
\par We perform SIRLM on the aforementioned datasets using four paradigms: End-to-End (E2E) training from scratch, Pre-Training (PT), SFT, and GRPO post-training. The training settings for each paradigm are detailed in \textbf{Appendix~\ref{appendixDataset}}.
\par \textbf{Baselines}. We compare SIRLM with \textbf{(1) conventional KG embedding methods} (TransE~\citep{TransE}, RotatE~\citep{RotatE}, and TuckER~\citep{TuckER}), \textbf{(2) rule-based methods} (NeuralLP~\citep{NeuralLP}, DRUM~\citep{DRUM}, and RNNLogic~\citep{RNNLogic}), \textbf{(3) GNN-based methods} (NBFNet~\citep{NBFNet}, RED-GNN~\citep{RED-GNN}, ULTRA~\citep{ULTRA}, and MOTIF~\citep{MOTIF}), and \textbf{(4) LLM-based methods} (KICGPT~\citep{KICGPT}, LSA~\citep{LSA}, ChatRule~\citep{ChatRule}, MKGL~\citep{MKGL}, KG-FIT~\citep{KG-FIT}, FtG~\citep{FtG}, and KRLM~\citep{KRLM}).
\begin{table}[t]
\setlength\tabcolsep{3.5pt}
\renewcommand{\arraystretch}{1.1}
\tiny
\vspace{-10pt} 
\newcommand{\tabincell}[2]{\begin{tabular}{@{}#1@{}}#2\end{tabular}}
\caption{\label{TabMainResults}
The overall performance of various methods on different datasets, where the MRR and Hit10 in the IndE and FullInd scenarios are summarized as average values. The colored cells represent the \colorbox{myRed!80}{best}, \colorbox{myOrange!80}{second-best}, and \colorbox{myYellow!80}{third-best} values, respectively. “–” indicates that the experimental results are unavailable, and “NA” indicates that a model is not applicable to a KGR task.
}
\centering
\begin{threeparttable}
\begin{tabular}{ c | c | c  c  c  c  c  c  c  c | c  c  c  c }
\hline
\multirow{2}{*}[-0.05in]{\tabincell{c}{\textbf{Type}}} & \multirow{2}{*}[-0.05in]{\tabincell{c}{\textbf{Methods}}} & \multicolumn{2}{c}{\raisebox{-0.5ex}{\textbf{FB15k237}}} & \multicolumn{2}{c}{\raisebox{-0.5ex}{\textbf{CoDEx-M}}} & \multicolumn{2}{c}{\raisebox{-0.5ex}{\textbf{WN18RR}}} & \multicolumn{2}{c}{\raisebox{-0.5ex}{\textbf{NELL995}}} & \multicolumn{2}{|c}{\raisebox{-0.5ex}{\textbf{12 IndE Datasets}}} & \multicolumn{2}{c}{\raisebox{-0.5ex}{\textbf{20 FullInd Datasets}}} \\
\cmidrule(r){3-4}\cmidrule(r){5-6}\cmidrule(r){7-8}\cmidrule(r){9-10}\cmidrule(r){11-12}\cmidrule(r){13-14}
 &  & \textbf{MRR}$\uparrow$ & \textbf{Hit10}$\uparrow$ & \textbf{MRR}$\uparrow$ & \textbf{Hit10}$\uparrow$ & \textbf{MRR}$\uparrow$ & \textbf{Hit10}$\uparrow$ & \textbf{MRR}$\uparrow$ & \textbf{Hit10}$\uparrow$ & \textbf{MRR}$\uparrow$ & \textbf{Hit10}$\uparrow$ & \textbf{MRR}$\uparrow$ & \textbf{Hit10}$\uparrow$\\
\hline
\multirow{3}{*}[-0.02in]{\tabincell{c}{\textbf{Embedding}\\\textbf{methods}}} & TransE & 0.313 & 0.495 & 0.320 & 0.481 & 0.226 & 0.501 & 0.401 & 0.501 & NA & NA & NA & NA\\
& RotatE & 0.338 & 0.533 & 0.325 & 0.466 & 0.476 & 0.571 & 0.483 & 0.565 & NA & NA & NA & NA\\
& TuckER & 0.358 & 0.544 & 0.328 & 0.458 & 0.470 & 0.526 & 0.520 & 0.624 & NA & NA & NA & NA\\
\hline
\multirow{3}{*}[-0.02in]{\tabincell{c}{\textbf{Rule-based}\\\textbf{methods}}} & NeuralLP & 0.237 & 0.362 & 0.334 & 0.444 & 0.435 & 0.566 & 0.394 & 0.482 & 0.449 & 0.623 & NA & NA\\
 & DRUM & 0.343 & 0.516 & 0.312 & 0.438 & 0.486 & 0.586 & 0.532 & \cellcolor{myYellow!80}\textbf{0.662} & 0.458 & 0.621 & NA & NA\\
 & RNNLogic & 0.344 & 0.530 & 0.310 & 0.445 & 0.483 & 0.558 & 0.416 & 0.478 & - & - & NA & NA\\
\hline
\multirow{4}{*}[-0.02in]{\tabincell{c}{\textbf{GNN-based}\\\textbf{methods}}} & NBFNet & \cellcolor{myYellow!80}\textbf{0.415} & \cellcolor{myRed!80}\textbf{0.599} & 0.343 & 0.509 & 0.551 & \cellcolor{myOrange!80}\textbf{0.666} & 0.525 & 0.639 & 0.527 & 0.670 & 0.122 & 0.259\\
& RED-GNN & 0.374 & 0.558 & 0.342 & 0.499 & 0.533 & 0.624 & \cellcolor{myOrange!80}\textbf{0.543} & 0.651 & 0.504 & 0.648 & 0.151\tnote{$\bullet$} & 0.242\tnote{$\bullet$}\\
& ULTRA & 0.368 & 0.564 & 0.372 & 0.525 & 0.480 & 0.614 & 0.509 & \cellcolor{myOrange!80}\textbf{0.660} & 0.565 & 0.724 & 0.366 & 0.529\\
& MOTIF & 0.357 & 0.550 & 0.361 & 0.517 & 0.529 & 0.628 & 0.514 & 0.655 & 0.582 & 0.739 & 0.373 & 0.535\\
\hline
\multirow{7}{*}[-0.05in]{\tabincell{c}{\textbf{LLM-based}\\\textbf{methods}}} & KICGPT & 0.412 & 0.554 & - & - & 0.549 & 0.641 & - & - & NA & NA & NA & NA\\
& LSA & 	0.356 & 0.538 & - & - & 0.469 & 0.542 & 0.469 & 0.649 & - & - & - & -\\
& KG-FIT & 0.362 & 0.572 & - & - & \cellcolor{myOrange!80}\textbf{0.553} & \cellcolor{myRed!80}\textbf{0.695} & - & - & NA & NA & NA & NA\\
& ChatRule\tnote{$\dagger$} & 0.383 & 0.564 & 0.320 & 0.481 & 0.445 & 0.518 & 0.511 & 0.627 & 0.535 & 0.660 & 0.319 & 0.472\\
& FtG & 0.392 &  0.542 & \cellcolor{myRed!80}\textbf{0.395} & 0.473 & - & - & 0.538 & 0.626 & NA & NA & NA & NA\\
& MKGL\tnote{$\dagger$} & \cellcolor{myYellow!80}\textbf{0.415} & 0.591 & 0.355 & 0.518 & \cellcolor{myYellow!80}\textbf{0.552} & 0.656 & 0.530 & 0.649 & 0.577 & 0.721 & NA & NA\\
& KRLM\tnote{$\dagger$} & 0.394 & 0.568 & 0.367 & \cellcolor{myYellow!80}\textbf{0.526} & \cellcolor{myYellow!80}\textbf{0.552} & \cellcolor{myYellow!80}\textbf{0.659} & 0.533 & 0.651 & 0.578 & 0.737 & 0.379 & 0.540\\
\hline
\multirow{5}{*}[-0.03in]{\tabincell{c}{\textbf{Ours}}} & $\text{SIRLM}_{\text{E2E}}$ & \cellcolor{myOrange!80}\textbf{0.427} & \cellcolor{myRed!80}\textbf{0.599} & 0.368 & 0.525 & \cellcolor{myYellow!80}\textbf{0.552} & 0.644 & \cellcolor{myYellow!80}\textbf{0.541} & 0.642 & \cellcolor{myYellow!80}\textbf{0.588} & \cellcolor{myYellow!80}\textbf{0.741} & \cellcolor{myOrange!80}\textbf{0.383} & \cellcolor{myOrange!80}\textbf{0.545}\\
& $\text{SIRLM}_{\text{PT}}$ & 0.408 & \cellcolor{myYellow!80}\textbf{0.593} & 0.367 & 0.522 & 0.461 & 0.556 & 0.520 & 0.638 & 0.580 & 0.738 & 0.374 & 0.538\\
& $\text{SIRLM}_{\text{SFT}}$ & \cellcolor{myRed!80}\textbf{0.429} & \cellcolor{myRed!80}\textbf{0.599} & \cellcolor{myYellow!80}\textbf{0.379} & \cellcolor{myOrange!80}\textbf{0.531} & \cellcolor{myRed!80}\textbf{0.556} & 0.648 & \cellcolor{myRed!80}\textbf{0.549} & \cellcolor{myRed!80}\textbf{0.668} & \cellcolor{myRed!80}\textbf{0.599} & \cellcolor{myRed!80}\textbf{0.749} & \cellcolor{myRed!80}\textbf{0.388} & \cellcolor{myRed!80}\textbf{0.550}\\
& $\text{SIRLM}_{\text{GRPO}}$ & 0.414 & \cellcolor{myOrange!80}\textbf{0.595} & \cellcolor{myOrange!80}\textbf{0.383} & \cellcolor{myRed!80}\textbf{0.542} & 0.523 & 0.629 & \cellcolor{myYellow!80}\textbf{0.541} & 0.644 & \cellcolor{myOrange!80}\textbf{0.593} & \cellcolor{myOrange!80}\textbf{0.744} & \cellcolor{myYellow!80}\textbf{0.381} & \cellcolor{myYellow!80}\textbf{0.544}\\
\cline{2-14}
& \textcolor{upColor}{\textit{Avg. gain (\%)}}\tnote{$\star$} & \textcolor{upColor}{\textit{+6.66}} & \textcolor{upColor}{\textit{+5.90}} & \textcolor{upColor}{\textit{+4.13}} & \textcolor{upColor}{\textit{+5.77}} & \textcolor{upColor}{\textit{+6.92}} & \textcolor{upColor}{\textit{+5.11}} & \textcolor{upColor}{\textit{+5.45}} & \textcolor{upColor}{\textit{+6.01}} & \textcolor{upColor}{\textit{+6.84}} & \textcolor{upColor}{\textit{+6.64}} & \textcolor{upColor}{\textit{+10.30}} & \textcolor{upColor}{\textit{+12.05}}\\
\hline
\end{tabular}
\end{threeparttable}
\begin{tablenotes}
\item[$\bullet$] $\bullet$ RED-GNN can only obtain experimental results on 12 FullInd datasets partitioned from FB15k237, NELL995, and Wikidata68K.
\item[$\dagger$] $\dagger$ We reproduce ChatRule, MKGL, and KRLM on all datasets. 
\item[$\star$] $\star$ We calculate the average gain of the optimal value in the four training modes of SIRLM compared to all baselines.
\end{tablenotes}
\vspace{-7mm}
\end{table}
\begin{wrapfigure}[15]{r}{0.55\linewidth}
    \begin{centering}
      \includegraphics[width=\linewidth]{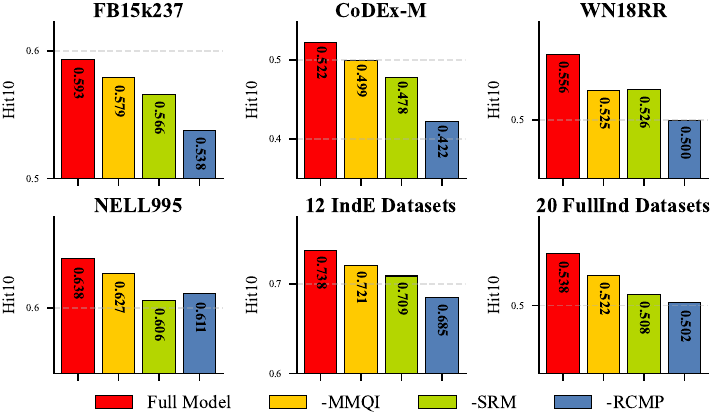}
\vspace{-6mm}
      \caption{\label{FigAblationResults}Hit10 of $\text{SIRLM}_{\text{PT}}$ variants on different datasets, where the results in the IndE and FullInd scenarios are summarized as average values.}
    \end{centering}
\vspace{-4mm}
\end{wrapfigure}
\par \textbf{Experimental settings}. Based on previous work~\citep{ULTRA}, we adopt Mean Recurrent Rank (MRR) and top-10 Hit rate (Hit10) as evaluation metrics. In the main experiment, we use Qwen2.5-1.5b as the backbone of the SIRG module. We pre-train and fine-tune SIRLM using 4 A100 (40GB) GPUs. The more detailed settings of model hyperparameters are provided in \textbf{Appendix~\ref{appendixHyperparameter}}.
\vspace{-2mm}
\subsection{Main Results (RQ1)}
\label{MainResults}
\vspace{-2mm}
Table~\ref{TabMainResults} reports the overall performance of various methods across 36 KGR datasets under transductive, IndE, and FullInd settings. Overall, our proposed SIRLM achieves consistently strong and competitive results, outperforming prior methods on most datasets and metrics. Compared with embedding and rule-based approaches, LLM-based methods demonstrate clear advantages. This is because LLMs are better at understanding structural context and leveraging it for effective fact inference. GNN-based methods, supported by their ability to learn invariant graph representations, can capture general structural semantics across different KGs and thus achieve meaningful progress on inductive tasks.
\par In comparison to other baselines, LLM-based methods generally provide a holistic performance. However, most of them primarily focus on the transductive setting. Consequently, we reproduce open-source LLM-based methods (ChatRule, MKGL, and KRLM) for inductive reasoning. These methods are built upon GPT-4o mini or LLaMA2-7B and achieve competitive performance. In contrast, our SIRLM, despite being based on a 1.5B-scale LLM, surpasses these LLM-based approaches. This is primarily attributed to its effective internalization of structural knowledge and stricter adherence to logical semantics. More detailed experimental analysis can be found in Appendixes~\ref{appendixMainResults}.
\begin{figure}
    \vspace{-10pt}  
    \begin{centering}
      \includegraphics[width=\linewidth]{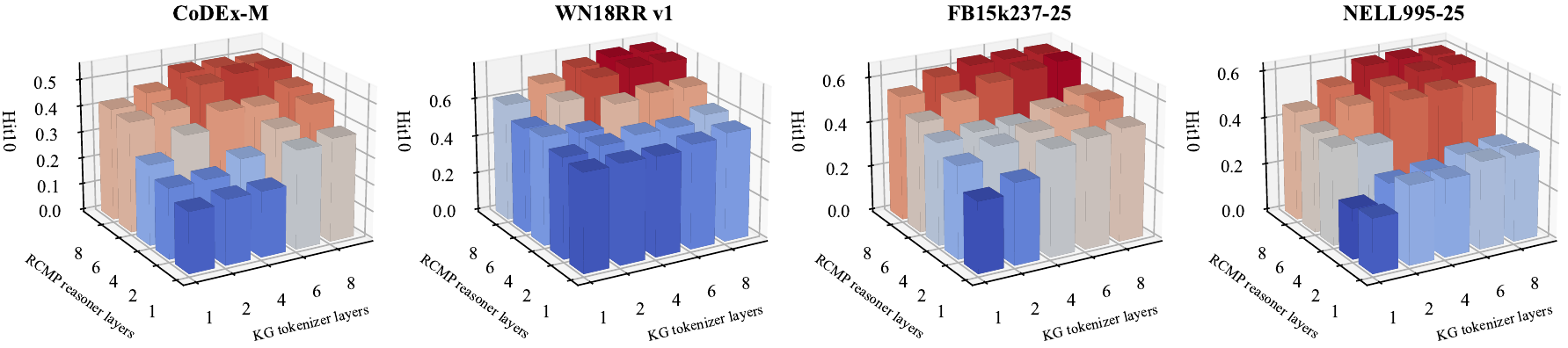}
\vspace{-6mm}
      \caption{\label{FigGNNLayers}Hit10 of $\text{SIRLM}_{\text{E2E}}$ with different NBFNet layers $N$ in KG tokenizer and RCMP reasoner.}
    \end{centering}
\vspace{-6mm}
\end{figure}
\vspace{-5mm}
\subsection{Ablation Experiments (RQ2)}
\label{Ablation Experiments}
\vspace{-3mm}
This section discusses the effectiveness of different modules in SIRLM. The experimental results are shown in Figure~\ref{FigAblationResults}. Overall, the effectiveness of each ablation variant is inferior to that of the full model, especially in structural knowledge learning modules such as “SRM” and “RCMP”. \textbf{Appendix~\ref{appendixMainResultsAbl}} provides detailed ablation variant settings and experimental results.
\vspace{-1mm}
\subsection{Parameter Analysis (RQ3)}
\label{Parameter Analysis}
\vspace{-2mm}
\begin{wrapfigure}[12]{r}{0.6\linewidth}
    \vspace*{-19pt}  
    \begin{centering}
      \includegraphics[width=\linewidth]{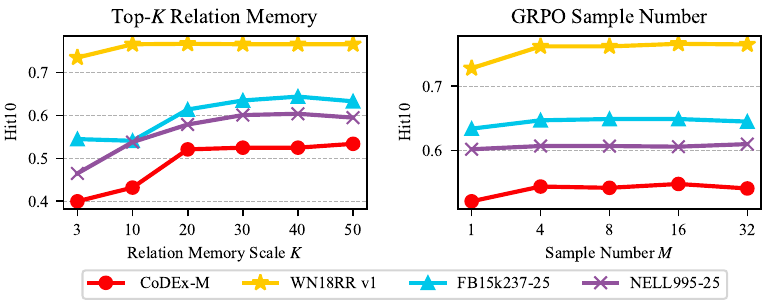}
\vspace{-7mm}
      \caption{\label{FigMemGRPOParameter}Hit10 of $\text{SIRLM}_{\text{E2E}}$ with different relation memory scale $K$ and $\text{SIRLM}_{\text{GRPO}}$ with different GRPO sample number $M$.}
    \end{centering}
\end{wrapfigure}
Figures~\ref{FigGNNLayers} and~\ref{FigMemGRPOParameter} illustrate the performance of SIRLM under different configurations of the NBFNet layers $N$, relation memory scale $	K$, and the number of GRPO samples $N$. When the layers in both the KG tokenizer and the RCMP reasoner is set to $N=6$, SIRLM achieves a good balance between reasoning accuracy and memory cost. $K$ is related to the relation scale in the KG. Therefore, we uniformly set $K=30$ in our experiments. When $M=1$, the group relative advantage in Eq. (\ref{eq15}) becomes ineffective, making it difficult for SIRLM to learn high-quality logical patterns, a phenomenon that is particularly pronounced in sparse KGs such as WN18RR v1 and CodeX-M. Therefore, we consistently set $M=8$ in our experiments.
\vspace{-2mm}
\subsection{Adaptability Analysis of LLM backbones (RQ4)}
\vspace{-2mm}
\label{LLMBackboneAnalysis}
\begin{wraptable}[8]{r}{0.54\textwidth}
\setlength\tabcolsep{1.5pt}
\scriptsize
\vspace*{-17pt}
\newcommand{\tabincell}[2]{\begin{tabular}{@{}#1@{}}#2\end{tabular}}
\caption{\label{TabLLMBackboneResults}The performance of $\text{SIRLM}_{\text{E2E}}$ with different LLM backbones.}
\vspace*{3pt}
\centering
\begin{threeparttable}
\begin{tabular}{ c | c  c  c  c  c  c  c  c }
\hline
\multirow{2}{*}[-0.05in]{\tabincell{c}{\textbf{LLM}\\\textbf{Backbone}}} & \multicolumn{2}{c}{\raisebox{-0.5ex}{\textbf{CoDEx-M}}} & \multicolumn{2}{c}{\raisebox{-0.5ex}{\textbf{WN18RR v1}}} & \multicolumn{2}{c}{\raisebox{-0.5ex}{\textbf{FB15k237-25}}} & \multicolumn{2}{c}{\raisebox{-0.5ex}{\textbf{NELL995-25}}} \\
\cmidrule(r){2-3}\cmidrule(r){4-5}\cmidrule(r){6-7}\cmidrule(r){8-9}
 & \textbf{MRR} & \textbf{Hit10} & \textbf{MRR} & \textbf{Hit10} & \textbf{MRR} & \textbf{Hit10} & \textbf{MRR} & \textbf{Hit10}\\
\hline
\textbf{Qwen2.5-0.5b} & 0.351 & 0.511 & 0.692 & 0.749 & 0.382 & 0.637 & 0.389 & 0.564\\
\textbf{Qwen2.5-1.5b} & 0.368 & 0.525 & 0.705 & 0.766 & 0.386 & 0.635 & 0.407 & 0.601\\
\textbf{Qwen2.5-7b} & 0.369 & 0.531 & 0.708 & 0.771 & 0.387 & 0.635 & 0.408 & 0.606\\
\textbf{Llama-2-7b} & 0.362 & 0.528 & 0.711 & 0.781 & 0.387 & 0.632 & 0.405 & 0.602\\
\hline
\end{tabular}
\end{threeparttable}
\end{wraptable}
This section evaluates the adaptability of $\text{SIRLM}_{\text{E2E}}$ across different LLM backbones. As shown in Table~\ref{TabLLMBackboneResults}, SIRLM remains relatively stable across most backbones, with only a slight performance decline on Qwen2.5-0.5B. We attribute this gap to the limited capacity of smaller LLMs to assimilate structural knowledge. As shown in Figure~\ref{FigLLMBackboneAnalysis} of \textbf{Appendix~\ref{appendixLLMBackboneAnalysis}}, Qwen2.5-0.5B exhibits the slowest convergence in rule-token generation accuracy, suggesting greater difficulty in aligning structural representations within its limited parameter space. As model capacity increases, SIRLM can more effectively internalize structural knowledge, leading to more stable reasoning performance.
\vspace{-3mm}
\section{Conclusion}
\vspace{-3mm}
This paper identifies a pervasive issue in LLM-based KGR, termed reasoning evidence perception drift, caused by the misalignment between KG structural representations and LLM parametric knowledge, which undermines both reasoning effectiveness and faithfulness. To address this issue, we propose the Structure-Internalized Rule Language Model (SIRLM), which integrates a structure-internalized rule generator with a KG tokenizer and a neuro-symbolic reasoner. SIRLM aligns parametric knowledge with KG structure through structural representation learning, rule generation, and faithfulness feedback, enabling evidence-grounded reasoning over KGs. Extensive experiments with 17 baselines on 36 KGR benchmarks show that SIRLM consistently outperforms existing methods across end-to-end training, pre-training, supervised fine-tuning, and post-training settings. \textbf{Appendix}~\ref{appendixLimitations} discusses its limitations and future directions.

\subsection*{AI use statement}
In this work, we have not used generative AI tools for any tasks requiring disclosure, and the remaining required disclosure tasks are not applicable to this work. We used generative AI tools only to edit the manuscript for grammar and readability. All AI-assisted edits were reviewed by the authors to ensure that they did not alter the technical content, scientific claims, or intended meaning. We take responsibility for the final content of this work, including text, claims, or artifacts produced with the aid of generative AI.



\subsection*{Reproducibility statement}
We confirm that our study has reproducibility. Specifically, we have first submitted our desensitized project on anonymous GitHub (\url{https://github.com/lazyloafer/SIRLM}). The detailed pseudocode of the algorithm is provided in \textbf{Appendix~\ref{appendixAlgorithm}}. In addition, we provide specific details of the experimental conclusions in the main text, including dataset partitioning (\textbf{Appendix~\ref{appendixDataset}}), hyperparameter settings (\textbf{Appendix~\ref{appendixHyperparameter}}), and ablation variant settings (\textbf{Appendix~\ref{appendixMainResultsAbl}}).

\bibliography{ref.bib}
\bibliographystyle{iclr2027_conference}
\newpage
\appendix
\section{Design Details of Query Instructions}
\label{appendixInstruction}
Given a query triplet $(e_h, r_q, ?)$, we first provide its schema of a query instruction below:
\begin{tcolorbox}[
  title=Schema of the Query Instruction for Rule Generation,
]
\centering
\begin{lstlisting}[
  basicstyle=\small,
  breaklines=true,
  breakindent=0pt,
  linewidth=\columnwidth,
  columns=fullflexible,
  escapeinside={(*}{*)},
  %aboveskip=0pt,   % 去掉listing上间距
  %belowskip=0pt    % 去掉listing下间距
]
Suppose you are a linguistic expert who is learning a new rule language. Given the following head entity and query relation:

### Head entity: (*$<$Ent $e_h$$>$*)

### Query relations: (*$<$Rel $r_q$$>$*)

Generate the following rule language: (*$<$Rel $r_q$$>$=$\underbrace{<\text{Rel } r_1><\text{Rel } r_3><\text{Rel } r_2><\text{END}>}_{\text{mask for generation}}$*)
\end{lstlisting}
\end{tcolorbox}
\par Our query instruction consists of a fixed text prompt along with structural representation placeholders that vary depending on the query triplet. Specifically, $<\text{Ent }e_h>$ and $<\text{Rel }r_q>$ correspond to the structural representations of $e_h$ and $r_q$, respectively. The masked portion includes the structural representations of atomic relations in the rule body, $<\text{Rel }r_1><\text{Rel }r_3><\text{Rel }r_2>$, as well as the rule termination token $<\text{END}>$, all of which are derived from $\text{TKN}_{\text{KG}}$ in Eq. (\ref{eq6}).
\par This instruction format integrates textual tokes with the structural representations of a KG, enabling the LLM to further learn the contextual semantics of entities and relations within the structural space. Moreover, since the structural representations provided by NBFNet inherently encode high-order graph information, we can deliver richer structural information to the LLM using fewer tokens. In addition, this instruction format does not rely on explicit textual names of entities and relations, making it suitable for anonymized KGR scenarios.
\par We compare the instruction lengths designed by several recent LLM-based KGR methods in Table~\ref{TabInstructionTokens}, demonstrating that our constructed instructions achieve more efficient utilization of computational resources.
\begin{table}[h]
\setlength\tabcolsep{10pt}
\renewcommand{\arraystretch}{1.1}
\small
\newcommand{\tabincell}[2]{\begin{tabular}{@{}#1@{}}#2\end{tabular}}
\caption{\label{TabInstructionTokens}
Average instruction length of MKG, KRLM, and SIRLM in 36 KGR datasets.
}
\vspace{2mm}
\centering
\begin{threeparttable}
\begin{tabular}{ c  c  c  c}
\hline
\tabincell{c}{\textbf{Model}} & MKGL~\citep{MKGL} & KRLM~\citep{KRLM} & SIRLM\\
\hline
\textbf{Avg. Length} & 115.66$\pm$3.65 & 120.50$\pm$4.12 & 50.73$\pm$0.67\\
\hline
\end{tabular}
\end{threeparttable}
\end{table}
\section{Relational Graph Construction}
\label{appendixRelationalGraph}
Unlike a typical KG, a relational graph~\citep{ULTRA,MOTIF} is used to describe the relative states between relations. As shown in Figure~\ref{FigRelationalGraph}, in our study, the motifs connecting relation nodes in a relational graph are essentially a set of relation-oriented hyperedges that is defined as $\mathcal{R}^*$ in Eq. (\ref{eq1}). 
\begin{figure}
    \begin{centering}
      \includegraphics[width=\linewidth]{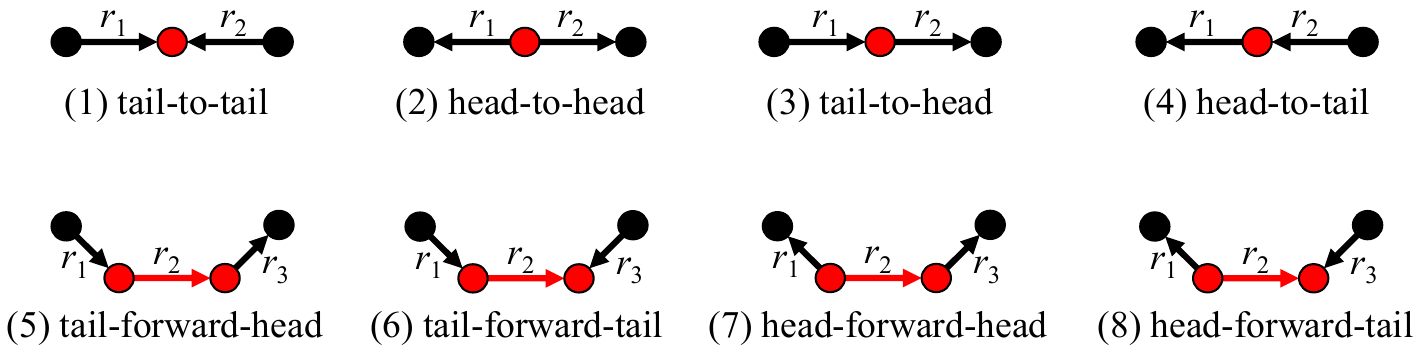}
\vspace{-5mm}
      \caption{\label{FigRelationalGraph}Motif edges of relations in a KG, whcih can be grouped into three binary edges and four ternary edges. A binary edge represents the state of the \textcolor{red}{entity} shared by two ordered relations, including \textbf{(1) tail-to-tail}: a shared tail entity, \textbf{(2) head-to-head}: a shared head entity, \textbf{(3) tail-to-head}: the tail entity of the previous relation being the head entity of the next relation, and \textbf{(4) head-to-tail}: the head entity of the previous relation being the tail entity of the next relation. A ternary edge represents the state of a \textcolor{red}{forward triplet} shared by two ordered relations, which is a hyperedge containing three relation nodes, including \textbf{(5) tail-forward-head}: the forward triplet as the tail of the previous relation and the head of the next relation, \textbf{(6) tail-forward-tail}: the forward triplet as the shared tail of two relations, \textbf{(7) head-forward-head}: the forward triplet as the shared head of two relations, and \textbf{(8) head-forward-tail}: the forward triplet as the head of the previous relation and the tail of the next relation.}
    \end{centering}
\end{figure}
\par In practical implementation, we use sparse matrices to construct the adjacency matrix of the relational graph. Giving a KG $\mathcal{G}=(\mathcal{E}, \mathcal{R}, \mathcal{T})$ defined as a graph with multiple directed edges, and its adjacency matrix can be represented as $\bm{A}\in\mathbb{R}^{\mathcal{E}\times\mathcal{R}\times\mathcal{E}}$. Subsequently, we construct the adjacency matrix by performing the maximum scatter operation on the head and tail node dimensions, resulting in two sparse matrices $\bm{A}_h\in\mathbb{R}^{\mathcal{E}\times\mathcal{R}}$ and $\bm{A}_t\in\mathbb{R}^{\mathcal{R}\times\mathcal{E}}$. $\bm{A}_h$ represents the out-degree edge of any relation from any node and $\bm{A}_t$ represents the in-degree edge of any relation pointing to any node. Then, the eight types of motif edges shown in Figure~\ref{FigRelationalGraph} can be obtained through a sparse matrix multiplication (spmm) operator:
\begin{equation}
\small
\nonumber
\begin{aligned}
&\text{Binary edges}=
\begin{cases}
\text{\textbf{Tail-to-tail} (\textit{t2t}): }\bm{A}_{t2t}=\text{spmm}(\bm{A}_t,\bm{A}_t^{\text{T}})\in\mathbb{R}^{\mathcal{R}\times\mathcal{R}}\\
\text{\textbf{Head-to-head} (\textit{h2h}): }\bm{A}_{h2h}=\text{spmm}(\bm{A}_h^{\text{T}},\bm{A}_h)\in\mathbb{R}^{\mathcal{R}\times\mathcal{R}}\\
\text{\textbf{Tail-to-head} (\textit{t2h}): }\bm{A}_{t2h}=\text{spmm}(\bm{A}_t,\bm{A}_h)\in\mathbb{R}^{\mathcal{R}\times\mathcal{R}}\\
\text{\textbf{Head-to-tail} (\textit{h2t}): }\bm{A}_{h2t}=\text{spmm}(\bm{A}_h^{\text{T}},\bm{A}_t^{\text{T}})\in\mathbb{R}^{\mathcal{R}\times\mathcal{R}}\\
\end{cases}\\
&\text{Ternary edges}=
\begin{cases}
\text{\textbf{Tail-forward-head} (\textit{tfh}): }\bm{A}_{tfh}=\text{spmm}(\bm{A}_t,\bm{A},\bm{A}_h)\in\mathbb{R}^{\mathcal{R}\times\mathcal{R}\times\mathcal{R}}\\
\text{\textbf{Tail-forward-tail} (\textit{tft}): }\bm{A}_{tft}=\text{spmm}(\bm{A}_t,\bm{A},\bm{A}_t^{\text{T}})\in\mathbb{R}^{\mathcal{R}\times\mathcal{R}\times\mathcal{R}}\\
\text{\textbf{Head-forward-heaed} (\textit{hfh}): }\bm{A}_{tft}=\text{spmm}(\bm{A}_h^{\text{T}},\bm{A},\bm{A}_h)\in\mathbb{R}^{\mathcal{R}\times\mathcal{R}\times\mathcal{R}}\\
\text{\textbf{Head-forward-tail} (\textit{hft}): }\bm{A}_{hft}=\text{spmm}(\bm{A}_h^{\text{T}},\bm{A},\bm{A}_t^{\text{T}})\in\mathbb{R}^{\mathcal{R}\times\mathcal{R}\times\mathcal{R}}\\
\end{cases}
\end{aligned}
\end{equation}
\section{Construction Details of the KG tokenizer and the RCMP Reasoner}
\label{appendixNBFNet}
Given that relational graphs contain ternary edges, there are certain differences in the construction details of NBFNet between the original KGs and the relational graphs in the KG tokenizer and the RCMP Reasoner.
\par For NBFNet on a relational graph, we need to consider the positional information of the relational nodes in the hyperedges. Therefore, Eq. (\ref{eq1}) can be concretized as
\begin{equation}
\label{eq17}
\small
\begin{aligned}
&\bm{r}^{(0)}_{j|q} = \mathbb{I}(r_j = r_q) * \bm{1}^d,\\
&\bm{r}^{(n)}_{j|q} = \sigma\Big(\bm{W}_r^{(n-1)}\Big[\bm{r}_{j|q}^{(n-1)}||\sum\limits_{r^*\in\mathcal{R}^*}\bm{r}^*\odot\big(\underset{r_z\in\mathcal{N}_{r^*}(r_j)}{\odot}(\bm{r}_{z|q}^{(n-1)}+\bm{p}_{z})\big)\Big]\Big),\bm{r}^*\in\bm{R}^*,
\end{aligned}
\end{equation}
where $\sigma(\cdot)$ is a ReLU activation function, $\odot$ is the Hadamard product operator, and $\bm{W}_r^{(n-1)}\in\mathbb{R}^{d\times d}$ and $\bm{p}_z\in\mathbb{R}^d$ is a trainable parameter matrix and a position embedding, respectively.
\par For the original KG, NBFNet only processes directed binary edges, so position information can be ignored. Therefore, Eq. (\ref{eq2}) can be concretized as
\begin{equation}
\label{eq18}
\small
\begin{aligned}
&\bm{e}^{(0)}_{i|h} = \mathbb{I}(e_i = e_h) * \bm{r}^{(N)}_{q|q},\\
&\bm{e}^{(n)}_{i|h} = \sigma\Big(\bm{W}_e^{(n-1)}\Big[\bm{e}_{i|h}^{(n-1)}||\sum\limits_{r_j\in\mathcal{R}}\sum\limits_{e_z\in\mathcal{N}_{r_j}(e_i)}\bm{r}_{j|q}^{(N)}\odot\bm{e}_{z|h}^{(n-1)}\Big]\Big),
\end{aligned}
\end{equation}
where $\bm{W}_e^{(n-1)}\in\mathbb{R}^{d\times d}$ is a trainable parameter matrix.
\par Similarly, Eqs. (\ref{eq11}) and (\ref{eq12}) can be concretized as Eqs. (\ref{eq19}) and (\ref{eq20}), respectively:
\begin{equation}
\label{eq19}
\small
\begin{aligned}
&\hat{\bm{r}}_{j|q} = \sum\limits_{<r, \bm{v}>\in U}\mathbb{I}(r_j = r) * \bm{v},\text{  where }U=\{<r_q, \mathbf{1}^d>\}\cup{\{<r_z, f_{\text{down}}(\bm{h}_{L+z})>\}}_{z=1}^\varepsilon\\
&\hat{\bm{r}}^{(n)}_{j|q} = \sigma\Big(\hat{\bm{W}}_r^{(n-1)}\Big[\hat{\bm{r}}_{j|q}^{(n-1)}||\sum\limits_{r^*\in\mathcal{R}^*}\hat{\bm{r}}^*\odot\big(\underset{r_z\in\mathcal{N}_{r^*}(r_j)}{\odot}(\hat{\bm{r}}_{z|q}^{(n-1)}+\bm{p}_{z})\big)\Big]\Big),\hat{\bm{r}}^*\in\hat{\bm{R}}^*,
\end{aligned}
\end{equation}
\begin{equation}
\label{eq20}
\small
\begin{aligned}
&\hat{\bm{e}}^{(0)}_{i|h} = \mathbb{I}(e_i = e_h) * \hat{\bm{r}}^{(N)}_{q|q},\\
&\hat{\bm{e}}^{(n)}_{i|h} = \sigma\Big(\hat{\bm{W}}_e^{(n-1)}\Big[\hat{\bm{e}}_{i|h}^{(n-1)}||\sum\limits_{r_j\in\mathcal{R}}\sum\limits_{e_z\in\mathcal{N}_{r_j}(e_i)}\hat{\bm{r}}_{j|q}^{(N)}\odot\hat{\bm{e}}_{z|h}^{(n-1)}\Big]\Big).
\end{aligned}
\end{equation}
Here, $\hat{\bm{W}}_r^{(n-1)},\hat{\bm{W}}_e^{(n-1)}\in\mathbb{R}^{d\times d}$ are two trainable parameter matrix.
\section{Discussion of the SRM In-context Layer}
\label{appendixSRMAttention}
This section discusses the effectiveness of the proposed SRM mechanism from the perspectives of structural representation alignment and semantic drift suppression. First, we provide the following explanation of the assumption and definitions required for subsequent analysis.
\begin{definition}[\textbf{In-context learning for the last token}]
\label{def:SRM}
Given a query triplet $<e_h, r_q, ?>$ and the corresponding instruction representation $\bm{X}\in\mathbb{R}^{L\times F}$, according to Eq. (\ref{eq4}), the hidden state of the last token $\bm{x}_L\in\bm{X}$ in the vanilla in-context learning module can be expressed as:
\begin{equation}
\small
\nonumber
\begin{aligned}
\bm{h}_L^{\mathrm{base}}=&\sum\limits_{i\le L}\alpha_if_{V}(\bm{x}_i),\\
\alpha_i=\frac{\exp{(a_i)}}{\sum\limits_{j\le L}\exp{(a_j)}}&,\hspace{2mm}a_i=\frac{f_Q(\bm{x}_L){[f_K(\bm{x}_i)]}^{\text{T}}}{\sqrt{F}}.
\end{aligned}
\end{equation}
\par Then, we introduce the structural relation memory $\bm{R}_{\mathrm{mem}}={\{\bm{r}_k\in\bm{R}_{|q}\}}_{k=1}^K$ from Eq. (\ref{eq7}) into the in-context learning module, the hidden state of the last token in Eq. (\ref{eq8}) can be expressed as:
\begin{equation}
\small
\nonumber
\begin{aligned}
\bm{h}_L^{\mathrm{SRM}}=&\sum\limits_{i\le L}\hat{\alpha}_if_{V}(\bm{x}_i)+\sum\limits_{k\le K}\beta_km_{V}(\bm{r}_k),\\
\hat{\alpha}_i=\frac{\exp{(a_i)}}{\sum\limits_{j\le L}\exp{(a_j)}+\sum\limits_{z\le K}\exp{(b_z)}}&,\hspace{2mm}\beta_k=\frac{\exp{(b_k)}}{\sum\limits_{j\le L}\exp{(a_j)}+\sum\limits_{z\le K}\exp{(b_z)}},\hspace{2mm}b_k=\frac{f_Q(\bm{x}_L){[m_K(\bm{r}_k)]}^{\text{T}}}{\sqrt{F}}.
\end{aligned}
\end{equation}
\end{definition}
\begin{definition}[\textbf{Structural subspace and structural observables}]
\label{def:StructuralSubspace}
Given a query triplet $<e_h, r_q, ?>$, define the structural subspace constructed by SIRLM as:
\begin{equation}
\small
\nonumber
\begin{aligned}
\mathcal{S}=\mathrm{span}(\{f_{\mathrm{up}}(\bm{r})|\bm{r}\in\bm{R}_{|q}\}\cup{\{m_V(\bm{r}_k)|\bm{r}_k\in\bm{R}_{\mathrm{mem}}\}}_{k=1}^K).
\end{aligned}
\end{equation}
\par Let $\bm{\mathrm{\Pi}}_{\mathcal{S}}$ denote the orthogonal projection operator onto $\mathcal{S}$. Then any $\bm{x}\in\mathbb{R}^F$ can be uniquely decomposed in $\mathcal{S}$ as:
\begin{equation}
\small
\nonumber
\begin{aligned}
\bm{x}=\bm{\mathrm{\Pi}}_{\mathcal{S}}\bm{x}+\bm{\mathrm{\Pi}}_{\mathcal{S}}^{\perp}\bm{x},
\end{aligned}
\end{equation}
where $\bm{\mathrm{\Pi}}_{\mathcal{S}}\bm{x}\in\mathcal{S}$ represents the structural component of $\bm{x}$ and $\bm{\mathrm{\Pi}}_{\mathcal{S}}^{\perp}\bm{x}$ is the semantic drift component orthogonal to $\mathcal{S}$.
\par Furthermore, define: 
\begin{equation}
\small
\nonumber
\begin{aligned}
\bm{u}_q=\frac{\bm{\mathrm{\Pi}}_{\mathcal{S}}f_{\mathrm{up}}(\bm{r}_{q|q})}{{\|\bm{\mathrm{\Pi}}_{\mathcal{S}}f_{\mathrm{up}}(\bm{r}_{q|q})\|}_2},\hspace{2mm}\bm{r}_{q|q}\in\bm{R}_{|q},
\end{aligned}
\end{equation}
which represents extracting a unit direction from the structural representation of $r_q$ within the structural subspace $\mathcal{S}$, serving as the structural anchor of $r_q$.
\par Based on $\bm{u}_q$ and $\bm{\mathrm{\Pi}}_{\mathcal{S}}$, two structural observables can be defined:
\begin{enumerate}[left=0pt]
\item[$\bullet$] \textbf{Structural alignment score}: $A_q(\bm{x})=\bm{u}_q^{\mathrm{T}}\bm{x}$, which measures the relevance of any $\bm{x}$ to $\bm{r}_q$ within the structural subspace $\mathcal{S}$.
\item[$\bullet$] \textbf{Semantic drift magnitude}: $D_\mathcal{S}(\bm{x})=\bm{u}_q^{\mathrm{T}}\bm{x}$, which measures the relevance of any $\bm{x}$ to $\bm{r}_q$ within the structural subspace $\mathcal{S}={\|\bm{\mathrm{\Pi}}_{\mathcal{S}}^{\perp}\bm{x}\|}_2$, which quantifies the component of any $\bm{x}$ orthogonal to the structural subspace $\mathcal{S}$.
\end{enumerate}
\end{definition}
\begin{assumption}
\label{assume:A}
Given a strictly increasing function $\phi(\cdot)$ and a strictly decreasing function $\psi(\cdot)$, for all Top-$K$ relations selected in Eq. (\ref{eq7}), we have:
\begin{equation}
\small
\nonumber
\begin{aligned}
A_q(m_V(\bm{r}_k))=\phi(s_{\mathrm{rel}}^{(k)}),\hspace{2mm}D_{\mathcal{S}}(m_V(\bm{r}_k))=\psi(s_{\mathrm{rel}}^{(k)}).
\end{aligned}
\end{equation}
\par Assume there exists a threshold $s_0$ such that ${\{s_{\mathrm{rel}}^{(k)}>s_0\}}_{k=1}^K$. Then for $\bm{h}^{\mathrm{base}}_L$, we have:
\begin{equation}
\small
\nonumber
\begin{aligned}
A_q(\bm{h}^{\mathrm{base}}_L)\le\phi(s_0),\hspace{2mm}D_{\mathcal{S}}(\bm{h}^{\mathrm{base}}_L)\ge\psi(s_0).
\end{aligned}
\end{equation}
That is, the aggregate result of the vanilla in-context learning module has a clear separation boundary from the structural reasoning subspace of the query triplet.
\end{assumption}
\par Based on the above definitions and assumptions, we present the following proposition.
\begin{proposition}
Given a query triplet $<e_h, r_q, ?>$ and its corresponding structural subspace $\mathcal{S}$, let $\bm{h}^{\mathrm{SRM}}_L$ and $\bm{h}^{\mathrm{base}}_L$ be the outputs obtained from the SRM-based and vanilla in-context learning modules, respectively. Then, the former exhibits stronger structural representation alignment and better suppression of semantic drift than the latter, i.e.,
\begin{equation}
\small
\nonumber
\begin{aligned}
A_q(\bm{h}^{\mathrm{SRM}}_L)>A_q(\bm{h}^{\mathrm{base}}_L),\hspace{2mm}D_{\mathcal{S}}(\bm{h}^{\mathrm{SRM}}_L)<D_{\mathcal{S}}(\bm{h}^{\mathrm{base}}_L).
\end{aligned}
\end{equation}
\end{proposition}
\begin{proof}
We prove the two claims separately.
\par \textbf{(I) Structural alignment improvement.} By Definitions \ref{def:SRM} and \ref{def:StructuralSubspace}, we have
\begin{equation}
\small
\nonumber
\begin{aligned}
&\begin{cases}
A_q(\bm{h}^{\mathrm{base}}_L)=\frac{\sum\limits_{i\le L}\exp{(a_i)}A_q\big(f_V(\bm{x}_i)\big)}{\sum\limits_{j\le L}\exp{(a_j)}}\\
A_q(\bm{h}^{\mathrm{SRM}}_L)=\frac{\sum\limits_{i\le L}\exp{(a_i)}A_q\big(f_V(\bm{x}_i)\big)}{\sum\limits_{j\le L}\exp{(a_j)}+\sum\limits_{z\le K}\exp{(b_z)}}+\sum\limits_{k\le K}\beta_kA_q\big(m_V(\bm{r}_k)\big)
\end{cases}\\
\Rightarrow&A_q(\bm{h}^{\mathrm{SRM}}_L)=\frac{\sum\limits_{j\le L}\exp{(a_j)}}{\sum\limits_{j\le L}\exp{(a_j)}+\sum\limits_{z\le K}\exp{(b_z)}}A_q(\bm{h}^{\mathrm{base}}_L)+\sum\limits_{k\le K}\beta_kA_q\big(m_V(\bm{r}_k)\big).
\end{aligned}
\end{equation}
\par Let $c=\frac{\sum\limits_{j\le L}\exp{(a_j)}}{\sum\limits_{j\le L}\exp{(a_j)}+\sum\limits_{z\le K}\exp{(b_z)}}$, then $\sum\limits_{k\le K}\beta_k=1-c$. Further, define $\overline{A}_{\mathrm{mem}}=\frac{\sum\limits_{k\le K}\beta_kA_q\big(m_V(\bm{r}_k)\big)}{\sum\limits_{k\le K}\beta_k}$, $A_q(\bm{h}^{\mathrm{SRM}}_L)$ can be rewritten as:
\begin{equation}
\small
\nonumber
\begin{aligned}
A_q(\bm{h}^{\mathrm{SRM}}_L)=A_q(\bm{h}^{\mathrm{base}}_L)+(1-c)(\overline{A}_{\mathrm{mem}}-A_q(\bm{h}^{\mathrm{base}}_L)).
\end{aligned}
\end{equation}
Therefore, it suffices to prove that $\overline{A}_{\mathrm{mem}}>A_q(\bm{h}^{\mathrm{base}}_L)$.
\par According to \textbf{Assumption~\ref{assume:A}} and the definition of $\overline{A}_{\mathrm{mem}}$, we have:
\begin{equation}
\small
\nonumber
\begin{aligned}
\begin{cases}
A_q(m_V(\bm{r}_k))=\phi(s_{\mathrm{rel}}^{(k)})>\phi(s_0)\\
A_q(\bm{h}^{\mathrm{base}}_L)\le\phi(s_0)
\end{cases}
\Rightarrow\overline{A}_{\mathrm{mem}}>A_q(\bm{h}^{\mathrm{base}}_L).
\end{aligned}
\end{equation}
Therefore, $A_q(\bm{h}^{\mathrm{SRM}}_L)>A_q(\bm{h}^{\mathrm{base}}_L)$ is proven.
\par \textbf{(II) Suppression of semantic drift.}
Following a similar procedure as in the previous proof, and based on Definitions \ref{def:SRM} and \ref{def:StructuralSubspace}, we obtain:
\begin{equation}
\small
\nonumber
\begin{aligned}
&\begin{cases}
D_{\mathcal{S}}(\bm{h}^{\mathrm{base}}_L)={\Bigg\|\frac{\sum\limits_{i\le L}\exp{(a_i)}\bm{\mathrm{\Pi}}_{\mathcal{S}}^{\perp}f_V(\bm{x}_i)}{\sum\limits_{j\le L}\exp{(a_j)}}\Bigg\|}_2\\
D_{\mathcal{S}}(\bm{h}^{\mathrm{SRM}}_L)={\Bigg\|\frac{\sum\limits_{i\le L}\exp{(a_i)}\bm{\mathrm{\Pi}}_{\mathcal{S}}^{\perp}f_V(\bm{x}_i)}{\sum\limits_{j\le L}\exp{(a_j)}+\sum\limits_{z\le K}\exp{(b_z)}}+\sum\limits_{k\le K}\beta_k\bm{\mathrm{\Pi}}_{\mathcal{S}}^{\perp}m_V(\bm{r}_k)\Bigg\|}_2
\end{cases}\\
\Rightarrow&D_{\mathcal{S}}(\bm{h}^{\mathrm{SRM}}_L)\le cD_{\mathcal{S}}(\bm{h}^{\mathrm{base}}_L)+\sum\limits_{k\le K}\beta_kD_{\mathcal{S}}\big(m_V(\bm{r}_k)\big).
\end{aligned}
\end{equation}
\par Let $\overline{D}_{\mathrm{mem}}=\frac{\sum\limits_{k\le K}\beta_kD_{\mathcal{S}}\big(m_V(\bm{r}_k)\big)}{\sum\limits_{k\le K}\beta_k}$, then $D_{\mathcal{S}}(\bm{h}^{\mathrm{SRM}}_L)$ admits the following upper bound:
\begin{equation}
\small
\nonumber
\begin{aligned}
D_{\mathcal{S}}(\bm{h}^{\mathrm{SRM}}_L)\le D_{\mathcal{S}}(\bm{h}^{\mathrm{base}}_L)+(1-c)(\overline{D}_{\mathrm{mem}}-D_{\mathcal{S}}(\bm{h}^{\mathrm{base}}_L)).
\end{aligned}
\end{equation}
Thus, it suffices to prove that $\overline{D}_{\mathrm{mem}}<D_{\mathcal{S}}(\bm{h}^{\mathrm{base}}_L)$.
\par According to \textbf{Assumption~\ref{assume:A}} and the definition of $\overline{D}_{\mathrm{mem}}$, we have:
\begin{equation}
\small
\nonumber
\begin{aligned}
\begin{cases}
D_{\mathcal{S}}(m_V(\bm{r}_k))=\psi(s_{\mathrm{rel}}^{(k)})<\psi(s_0)\\
D_{\mathcal{S}}(\bm{h}^{\mathrm{base}}_L)\ge\psi(s_0)
\end{cases}
\Rightarrow\overline{D}_{\mathrm{mem}}<D_{\mathcal{S}}(\bm{h}^{\mathrm{base}}_L).
\end{aligned}
\end{equation}
Therefore, $D_{\mathcal{S}}(\bm{h}^{\mathrm{SRM}}_L)<D_{\mathcal{S}}(\bm{h}^{\mathrm{base}}_L)$ is proven.
\end{proof}
\par \textbf{Discussion.}  The proof above establishes the effectiveness of Eq. (\ref{eq8}). The SRM branch does not merely append extra structural representation, it actively reorganizes the hidden state by aligning its structural component and suppressing text-only semantic bias. Because the next-relation predictor in Eqs. (\ref{eq9}) and (\ref{eq10}) operates directly on this hidden state, the hidden-state improvement transfers into a larger score margin for the correct structural relation, thereby promoting the generation of KG-grounded rules.
\section{Discussion of the RCMP Mechanism}
\label{appendixRCMP}
This section discusses the effectiveness of the proposed RCMP mechanism. Based on the theoretical foundation of the previous studies~\citep{CMPNN,HCNet}, we start with the relational Weisfeiler-Leman (WL) test~\citep{WLTest} to analyze the effectiveness of RCMP in terms of structural invariance learning, lower bound of model expression, and stricter reasoning scenarios.
\subsection{A Two-stage WL Test for the SIL Mechanism}
\label{appendixWLTestSIL}
Before formally discussing the effectiveness of RCMP, we first provide a theoretical background for the WL test of the primary SIL mechanism~\citep{MOTIF} without RCMP. Giving a KG $\mathcal{G}=(\mathcal{E},\mathcal{R},\mathcal{T})$ and a corresponding relational graph $\mathcal{G}_r=(\mathcal{R},\mathcal{R}^*,\mathcal{T}^*)$, the WL test of the SIL mechanism can be defined as a \textit{two-stage coloring scheme}.
\par \textbf{Stage 1: relation coloring on $\mathcal{G}_r$}. According to Eq. (\ref{eq1}), giving a query relation $r_q\in \mathcal{R}$, the color of relation $r_j\in \mathcal{R}$ at iteration $n$ is
denoted by $\mathrm{col}_r^{(n)}(r_j|r_q)$ and initialized by
\begin{equation}
\small
\nonumber
\begin{aligned}
\mathrm{col}_r^{(0)}(r_j|r_q)=\mathbb{I}(r_j=r_q) \cdot \mathbf{1}.
\end{aligned}
\end{equation}
\par It is then updated as
\begin{equation}
\small
\nonumber
\begin{aligned}
\mathrm{col}_r^{(n+1)}&(r_j|r_q)=\\
&\mathrm{HASH}\!\left(
\mathrm{col}^{(n)}_r(r_j|r_q),
\{\!\!\{
(\{(\mathrm{col}^{(n)}_r(r_s|r_q),p_s)\mid (r_s,p_s)\in \mathcal{N}_{r^*}(r_j)\}, r^*)
\mid r^*\in \mathcal{R}^*
\}\!\!\}
\right),
\end{aligned}
\end{equation}
where $\mathcal{N}_{r^*}(r_j)$ represents the set of all neighbors of $r_j$ on $r^*$. Since $r^*$ is a hyperedge, $\mathcal{N}_{r^*}(r_j)$ stores each neighbor $r_s$ and its position $p_s$ on $r^*$. $\text{HASH}(\cdot)$ is an abstract injective coloring function~\citep{WLTest} that compresses old colors and neighborhood information into a new color label. As long as the input is different, the output color will be different.
\par \textbf{Stage 2: entity coloring on $\mathcal{G}$}.
Given a fixed number $N$ of first-stage iterations, the color of an entity $e_v$ in a candidate triplet
$<e_u,r_q,e_v>$ at iteration $\ell$ is denoted by $\mathrm{col}^{(\ell)}_{e}(e_v|e_u,r_q)$.
It is initialized by
\begin{equation}
\small
\nonumber
\begin{aligned}
\mathrm{col}_e^{(0)}(e_v|e_u,r_q)=\mathbb{I}(e_v=e_u) \cdot \mathrm{col}_r^{(N)}(r_q|r_q),
\end{aligned}
\end{equation}
and updated as
\begin{equation}
\small
\nonumber
\begin{aligned}
\mathrm{col}_e^{(\ell+1)}&(e_v|e_u,r_q)=\\
&\mathrm{HASH}\!\left(
\mathrm{col}_e^{(\ell)}(e_v|e_u,r_q),
\{\!\!\{
(\mathrm{col}_e^{(\ell)}(e_w|e_u,r_q), \mathrm{col}_r^{(N)}(r_j|r_q))
\mid e_w\in \mathcal{N}_{r_j}(e_v),\ r_j\in \mathcal{R}
\}\!\!\}
\right).
\end{aligned}
\end{equation}
This two-stage WL test characterizes the separating power of
the SIL mechanism~\citep{MOTIF}: it is both an upper bound and a tight characterization under injective message-passing choices.
\subsection{A Two-stage WL Test for the RCMP Mechanism}
\label{appendixWLTestRCMP}
This section provides a detailed elaboration on the effectiveness of the proposed RCMP based on \textbf{Appendix~\ref{appendixWLTestSIL}}. First, we outline the overall pipeline of the proposed SIRLM as \textbf{Definition~\ref{def:SIRLM}} for subsequent analysis.
\begin{definition}
\label{def:SIRLM}
Giving a KG $\mathcal{G}=(\mathcal{E},\mathcal{R},\mathcal{T})$ and a corresponding relational graph $\mathcal{G}_r=(\mathcal{R},\mathcal{R}^*,\mathcal{T}^*)$, SIRLM processes a query $<e_u, r_q, ?>$ through the following stages:
\par \textbf{Stage 1: structural tokenization.}
We first run Eqs. (\ref{eq1}) and (\ref{eq2}) to obtain structural token embeddigns of relation $\bm{r}_{j|q}\in\bm{R}_{|q}$ and entity $\bm{e}_{i|u}\in\bm{E}_{|u}$. \textbf{These embeddings are used as structure-only tokens; no entity names,
relation names, or textual attributes are used}. In particular, only $\bm{e}_{u|u}$ and $\bm{r}_{q|q}$ are exposed to the rule generator as query context.
\par \textbf{Stage 2: structure-aware rule generation.} Let $\mathcal{P}_{\mathrm{sys}}$ be a fixed system prompt. For a query $r_q(e_u,?)$, define the input prompt as
\begin{equation}
\small
\nonumber
\begin{aligned}
\Pi(e_u,r_q) = [\mathcal{P}_{\mathrm{sys}};\ \bm{e}_{u|u};\
\bm{r}_{q|q}\,].
\end{aligned}
\end{equation}
LLM generates a rule consisting of a sequence of atomic relation
\begin{equation}
\small
\nonumber
\begin{aligned}
\rho(e_u,r_q)=r_{1|\Pi(e_u,r_q)},\ldots,r_{Z|\Pi(e_u,r_q)}\in \mathcal{R}
\end{aligned}
\end{equation}
and corresponding hidden states
\begin{equation}
\small
\nonumber
\begin{aligned}
\bm{h}_{1|\Pi(e_u,r_q)},\ldots,\bm{h}_{Z|\Pi(e_u,r_q)}.
\end{aligned}
\end{equation}
\par At $z$-th generation iteration, the last hidden state $\bm{h}_{z-1|\Pi(e_u,r_q)}$ is computed by
a LLM fed with $\Pi(e_u,r_q)$ and the previously $z-1$ generated atomic relations. The next atomic relation is then selected by a lookup over the structured
relation-token bank $\bm{R}_{|q}$:
\begin{equation}
\small
\nonumber
\begin{aligned}
r_{z|\Pi(e_u,r_q)}
=
\arg\max_{r_j\in R}\ \mathrm{Score}\bigl(\bm{h}_{z-1|\Pi(e_u,r_q)},\bm{r}_{j|q}\bigr),
\end{aligned}
\end{equation}
where $\mathrm{Score}(\cdot)$ is a
relation-token lookup scorer.
\par\textbf{Stage 3: RCMP initialization of $\mathcal{G}_r$.} For each relation $r\in \mathcal{R}$, define the accumulated rule representation as
\begin{equation}
\small
\nonumber
\begin{aligned}
\bm{x}(r|e_u,r_q)
=
\sum_{z\,:\,r_{z|\Pi(e_u,r_q)}=r}\bm{h}_{z|\Pi(e_u,r_q)},
\end{aligned}
\end{equation}
with the convention that $\bm{x}(r|e_u,r_q)=\mathbf{0}$ if $r$ does not occur in
$\rho(e_u,r_q)$. This additive definition directly handles repeated relations, that is,
every occurrence contributes its own autoregressive hidden state, and all such
occurrence-specific hidden states are accumulated.
\par Giving the generated rule $\rho(e_u,r_q)$ and the accumulated rule representation $\bm{x}(\cdot|e_u,r_q)$, the RCMP initialization of $\mathcal{G}_r$ can be concretized as
\begin{equation}
\small
\nonumber
\begin{aligned}
\widetilde{\mathrm{INIT}}_1(r_j|e_u,r_q)
=
\mathbb{I}(r_j=r_q) \cdot \mathbf{1} + \bm{x}(r_j|e_u,r_q),
\end{aligned}
\end{equation}
where $\mathbf{1}\neq \mathbf{0}$ are fixed vectors. Equivalently, if the query
relation $r_q$ appears in the generated rule body, then its node in the relational graph
is initialized by the anchor vector $\mathbf{1}$ plus the sum of all hidden states
associated with occurrences of $r_q$ in $\rho(e_u,r_q)$. Likewise, if any non-query relation $r_j$ appears
multiple times in $\rho(e_u,r_q)$, then its initialization is the sum of the hidden
states associated with all of its occurrences.
\par It should be emphasized that RCMP only changes the initialization strategy of the relational graph $\mathcal{G}_r$ in SIL with $\widetilde{\mathrm{INIT}}_1(r_j|e_u,r_q)$, while keeping $\mathcal{G}_r$ and all relation and entity message-passing operators unchanged.
\end{definition}
\begin{assumption}
\label{assume:B}
We make the following assumptions.
\begin{itemize}[leftmargin=2.2em]
    \item[\textbf{1.}] \textbf{Structure-token invariance}. The base embeddings $\bm{r}_{j|q}\in\bm{R}_{|q}$ and $\bm{e}_{i|u}\in\bm{E}_{|u}$ are structurally invariants in \textbf{Definition~\ref{def:SIRLM}}.
    \item[\textbf{2.}] \textbf{Equivariance of LLM generation for structural relation tokens}. For every isomorphism $h=(\pi,\phi):\mathcal{G}\to \mathcal{G}'$, the autoregressive decoder and the lookup rule are equivariant with respect to the induced renaming of the structural tokens, \textit{i.e.},
\begin{equation}
\small
\nonumber
\begin{aligned}
\bm{h}_{z|\Pi(e_u,r_q)}=\bm{h}_{z|\Pi(\pi(e_u),\phi(r_q))},
    \qquad
    r_{z|\Pi(\pi(e_u),\phi(r_q))}=\phi(r_{z|\Pi(e_u,r_q)}).
\end{aligned}
\end{equation}
    \item[\textbf{3.}] \textbf{Anchor separation of relations}. Query relations are disjoint from non-query relations, \textit{i.e.},
\begin{equation}
\small
\nonumber
\begin{aligned}
\mathbf{1} + \bm{x}(r_j|e_u,r_q)\neq\bm{x}(r'_j|e'_u,r'_q)
    \qquad\text{whenever } r'_j\neq r'_q.
\end{aligned}
\end{equation}
    Equivalently, $\widetilde{\mathrm{INIT}}_1(r_j|e_u,r_q)=\widetilde{\mathrm{INIT}}_1(r'_j|e'_u,r'_q)$ implies $r_j=r_q$ if and only if $r'_j=r'_q$.
    \item[\textbf{4.}] \textbf{Injective operator in RCMP}. The update, aggregation, and message functions in both the relation encoder and the entity encoder are injective, exactly as in the constructive direction of \textbf{Appendix~\ref{appendixWLTestSIL}}.
\end{itemize}
\end{assumption}
\par The next analysis shows that the proposed RCMP preserves structural invariance
and refines the expressive power of the original SIL model.
\begin{proposition}
Under \textbf{Assumptions \ref{assume:B}1-\ref{assume:B}4}, the following hold for RCMP.
\begin{enumerate}[label=(\arabic*)]
    \item[\textbf{I.}] The initialization $\widetilde{\mathrm{INIT}}_1(u,q,r)$ has relation invariance. Consequently, RCMP computes relation invariants in the relation encoder and link invariants in the entity encoder.

    \item[\textbf{II.}] Let
\begin{equation}
\small
\nonumber
\begin{aligned}
\widetilde{\mathrm{col}}_r^{(0)}(r_j|e_u,r_q)
    =
    \chi\!\left(\widetilde{\mathrm{INIT}}_1(r_j|e_u,r_q)\right),
\end{aligned}
\end{equation}
    where $\chi$ is injective, and let
    $(\widetilde{\mathrm{col}}_r,\widetilde{\mathrm{col}}_e)$
    be the two-stage WL test of RCMP, while keeping all update rules unchanged. Then for all $n,\ell\ge 0$, we have
\begin{equation}
\small
\nonumber
\begin{aligned}
\widetilde{\mathrm{col}}_r^{(n)}(r_j|e_u,r_q)
    =
    \widetilde{\mathrm{col}}_r^{(n)}(r'_j|e'_u,r'_q)
    \Longrightarrow
    \mathrm{col}_r^{(n)}(r_j|r_q)
    =
    \mathrm{col}_r^{(n)}(r'_j|r'_q)
\end{aligned}
\end{equation}
    and
\begin{equation}
\small
\nonumber
\begin{aligned}
\widetilde{\mathrm{col}}_e^{(\ell)}(e_v|e_u,r_q)
    =
    \widetilde{\mathrm{col}}_e^{(\ell)}(e'_v|e'_u,r'_q)
    \Longrightarrow
    \mathrm{col}_e^{(\ell)}(e_v|e_u,r_q)
    =
    \mathrm{col}_e^{(\ell)}(e'_v|e'_u,r'_q).
\end{aligned}
\end{equation}
    Hence RCMP is at least as expressive as the original SIL model.

    \item[\textbf{III.}] Assume there exist a query $<e_u, r_q, ?>$ and two relations $r_j,r_s\in \mathcal{R}$ such that
\begin{equation}
\small
\nonumber
\begin{aligned}
\mathrm{col}_r^{(n)}(r_j,r_q)=\mathrm{col}_r^{(n)}(r_s,r_q)
    \quad \text{for all } n\ge 0,
\end{aligned}
\end{equation}
    but
\begin{equation}
\small
\nonumber
\begin{aligned}
\widetilde{\mathrm{INIT}}_1(r_j|e_u,r_q)
    \neq
    \widetilde{\mathrm{INIT}}_1(r_s|e_u,r_q).
\end{aligned}
\end{equation}
Then the refinement is strict, that is, there exists a KG and two query links that cannot be separated by any instance of SIL, but can be separated by some instance of RCMP.
\end{enumerate}
\end{proposition}
\begin{proof}
We prove the three claims as follows.
\par \textbf{I. $\widetilde{\mathrm{INIT}}_1$ is a relation invariant.} By \textbf{Assumption~\ref{assume:B}1}, $\bm{R}_{|q}$ and
$\bm{E}_{|u}$ depend only on KG structure, not on concrete entity or relation names. Since the prompt $\Pi(e_u,r_q)$ is formed only by these structural
tokens and a fixed system prompt, \textbf{Assumption~\ref{assume:B}2} implies that for every graph isomorphism
$h=(\pi,\phi):\mathcal{G}\to \mathcal{G}'$, we have
\begin{equation}
\small
\nonumber
\begin{aligned}
\bm{h}_{z|\Pi(e_u,r_q)}=\bm{h}_{z|\Pi(\pi(e_u),\phi(r_q))}
    \qquad\text{and}\qquad
    r_{z|\Pi(\pi(e_u),\phi(r_q))}=\phi(r_{z|\Pi(e_u,r_q)})
\end{aligned}
\end{equation}
for all decoding steps $z$. Therefore, for every relation $r_j\in \mathcal{R}$, the accumulated rule signal satisfies
\begin{equation}
\small
\nonumber
\begin{aligned}
\bm{x}(r_j|e_u,r_q)
=
\bm{x}(\phi(r_j)|\pi(e_u),\phi(r_q)).
\end{aligned}
\end{equation}
Since the query-anchor case $r_j=r_q$ is mapped to $\phi(r_j)=\phi(r_q)$ under the same
isomorphism, the initialization satisfies
\begin{equation}
\small
\nonumber
\begin{aligned}
\widetilde{\mathrm{INIT}}_1(r_j|e_u,r_q)
=
\widetilde{\mathrm{INIT}}_1(\phi(r_j)|\pi(e_u),\phi(r_q)).
\end{aligned}
\end{equation}
Hence $\widetilde{\mathrm{INIT}}_1$ is a relation invariant.
\par Now the relation encoder follows exactly the same inductive argument used in the
main manuscript. Once the initialization $\widetilde{\mathrm{INIT}}_1$ is a relation invariant, every subsequent layer of
the relation encoder remains a relation invariant because message, aggregation, and
update only combine invariant quantities over the relational graph $\mathcal{G}_r$.
Likewise, the entity encoder receives invariant relation representations and uses the same message-passing structure as the original SIL model, so it computes link invariants.
\par \textbf{II. RCMP is at least as expressive as the original SIL model.}
We define the first-stage colors of RCMP by
\begin{equation}
\small
\nonumber
\begin{aligned}
\widetilde{\mathrm{col}}_r^{(0)}(r_j|e_u,r_q)
    =
    \chi\!\left(\widetilde{\mathrm{INIT}}_1(r_j|e_u,r_q)\right),
\end{aligned}
\end{equation}
Then, for $n\ge 0$, we use the same first-stage update as in \textbf{Appendix~\ref{appendixWLTestSIL}}, but with the richer initial coloring:
\begin{equation}
\small
\nonumber
\begin{aligned}
\widetilde{\mathrm{col}}_r^{(n+1)}&(r_j|e_u,r_q)=\\
&\mathrm{HASH}\left(
\widetilde{\mathrm{col}}_r^{(n)}(r_j|e_u,r_q), \right.\left.
\{\!\!\{
(\{(\widetilde{\mathrm{col}}_r^{(n)}(r_s|e_u,r_q),p_s)\mid (r_s,p_s)\in \mathcal{N}_{r^*}(r_j)\},r^*)
\mid r^*\in \mathcal{R}^*
\}\!\!\}
\right).
\end{aligned}
\end{equation}
Similarly, define the second-stage colors by
\begin{equation}
\small
\nonumber
\begin{aligned}
\widetilde{\mathrm{col}}_e^{(0)}(e_v|e_u,r_q)=\mathbb{I}(e_v=e_u) \cdot \widetilde{\mathrm{col}}_r^{(N)}(r_q|e_u,r_q),
\end{aligned}
\end{equation}
and for $\ell\ge 0$,
\begin{equation}
\small
\nonumber
\begin{aligned}
\widetilde{\mathrm{col}}_e^{(\ell+1)}&(e_v|e_u,r_q)=\\
&\mathrm{HASH}\!\left(
\widetilde{\mathrm{col}}_e^{(\ell)}(e_v|e_u,r_q),
\{\!\!\{
(\widetilde{\mathrm{col}}_e^{(\ell)}(e_w|e_u,r_q), \widetilde{\mathrm{col}}_r^{(N)}(r_j|e_u,r_q))
\mid e_w\in \mathcal{N}_{r_j}(e_v),\ r_j\in \mathcal{R}
\}\!\!\}
\right).
\end{aligned}
\end{equation}
\textbf{(1) We first prove by induction on $n$ that}
\begin{equation}
\small
\nonumber
\begin{aligned}
\widetilde{\mathrm{col}}_r^{(n)}(r_j|e_u,r_q)
    =
    \widetilde{\mathrm{col}}_r^{(n)}(r'_j|e'_u,r'_q)
    \Longrightarrow
    \mathrm{col}_r^{(n)}(r_j|r_q)
    =
    \mathrm{col}_r^{(n)}(r'_j|r'_q).
\end{aligned}
\end{equation}
\par For $n=0$, injectivity of $\chi$ gives $\widetilde{\mathrm{INIT}}_1(r_j|e_u,r_q)=\widetilde{\mathrm{INIT}}_1(r'_j|e'_u,r'_q)$. According to \textbf{Assumption~\ref{assume:B}3}, such an equation preserves whether the anchor contribution
$\mathbb{I}(r_j=r_q)\cdot\mathbf{1}$ is present. Therefore, $\mathbb{I}(r_j=r_q)=\mathbb{I}(r'_j=r'_q)$, which is exactly 
\begin{equation}
\small
\nonumber
\begin{aligned}
\mathrm{col}_r^{(0)}(r_j|r_q)=\mathrm{col}_r^{(0)}(r'_j|r'_q).
\end{aligned}
\end{equation}
\par Assume the implication holds at layer $n$ and suppose
\begin{equation}
\small
\nonumber
\begin{aligned}
\widetilde{\mathrm{col}}_r^{(n+1)}(r_j|e_u,r_q)
    =
    \widetilde{\mathrm{col}}_r^{(n+1)}(r'_j|e'_u,r'_q).
\end{aligned}
\end{equation}
\par Since $\mathrm{HASH}$ is injective, both the previous colors and the corresponding
position-aware neighbor multisets of SIL must agree. Applying the induction hypothesis to all
first-stage colors appearing inside those multisets shows that the original colors in
\textbf{Appendix~\ref{appendixWLTestSIL}} also agree. Therefore the multisets defining
$\mathrm{col}_r^{(n+1)}(r_j|r_q)$ and $\mathrm{col}_r^{(n+1)}(r'_j|r'_q)$ coincide, thus
\begin{equation}
\small
\nonumber
\begin{aligned}
\mathrm{col}_r^{(n+1)}(r_j|r_q)=\mathrm{col}_r^{(n+1)}(r'_j|r'_q)
\end{aligned}
\end{equation}
is proven.
\par \textbf{(2) We next prove by induction on $\ell$ that}
\begin{equation}
\small
\nonumber
\begin{aligned}
\widetilde{\mathrm{col}}_e^{(\ell)}(e_v|e_u,r_q)=\widetilde{\mathrm{col}}_e^{(\ell)}(e'_v|e'_u,r'_q)
\Longrightarrow
\mathrm{col}_e^{(\ell)}(e_v|e_u,r_q)=\mathrm{col}_e^{(\ell)}(e'_v|e'_u,r'_q).
\end{aligned}
\end{equation}
For $\ell=0$, equality of the RCMP colors implies
\begin{equation}
\small
\nonumber
\begin{aligned}
\mathbb{I}(e_v=e_u)=\mathbb{I}(e'_v=e'_u)
\qquad\text{and}\qquad
\widetilde{\mathrm{col}}_r^{(N)}(r_j|e_u,r_q)
    =
    \widetilde{\mathrm{col}}_r^{(N)}(r'_j|e'_u,r'_q).
\end{aligned}
\end{equation}
By the first-stage result proved above, we obtain
\begin{equation}
\small
\nonumber
\begin{aligned}
\mathrm{col}_r^{(N)}(r_q|r_q)=\mathrm{col}_r^{(N)}(r'_q|r'_q),
\end{aligned}
\end{equation}
hence
\begin{equation}
\small
\nonumber
\begin{aligned}
\mathrm{col}_e^{(0)}(e_v|e_u,r_q)=\mathrm{col}_e^{(0)}(e'_v|e'_u,r'_q).
\end{aligned}
\end{equation}
\par Assume the implication holds at layer $\ell$ and suppose
\begin{equation}
\small
\nonumber
\begin{aligned}
\widetilde{\mathrm{col}}_e^{(\ell+1)}(e_v|e_u,r_q)=\widetilde{\mathrm{col}}_e^{(\ell+1)}(e'_v|e'_u,r'_q).
\end{aligned}
\end{equation}
Injectivity of $\mathrm{HASH}$ yields equality of the previous second-stage colors of SIL and
of the corresponding message multisets. By the induction hypothesis on $\ell$ and the
first-stage implication already proved, the original second-stage colors and original
first-stage relation colors also agree. Therefore, the multisets defining
$col^{(\ell+1)}_{F,T}(q(u,v))$ and
$col^{(\ell+1)}_{F,T}(q'(u',v'))$ are identical, and the assumption
\begin{equation}
\small
\nonumber
\begin{aligned}
\mathrm{col}_e^{(\ell+1)}(e_v|e_u,r_q)=\mathrm{col}_e^{(\ell+1)}(e'_v|e'_u,r'_q)
\end{aligned}
\end{equation}
is proven.
\par Therefore, the two-stage WL test of RCMP is a refinement of the original two-stage test
of SIL on \textbf{Appendix~\ref{appendixWLTestSIL}}. Under \textbf{Assumption~\ref{assume:B}4}, the same constructive argument as in \textbf{Appendix~\ref{appendixWLTestSIL}} applies to the present model, because only the initialization changes while the message-passing architecture over $\mathcal{G}_r$ and over the original KG remains unchanged. Therefore, the refined two-stage test characterizes RCMP, and the refinement
of the WL tests implies that RCMP is at least as expressive as SIL.
\par \textbf{III. Strictness.}
Assume there exist a query triplet $<e_u, r_q, ?>$ and two relations $r_j,r_s\in \mathcal{R}$ that satisfy 
\begin{equation}
\small
\nonumber
\begin{aligned}
\mathrm{col}_r^{(n)}(r_j|r_q)=\mathrm{col}_r^{(n)}(r_s|r_q)\qquad\text{but}\qquad\widetilde{\mathrm{INIT}}_1(r_j|e_u,r_q)\neq\widetilde{\mathrm{INIT}}_1(r_s|e_u,r_q)\qquad\text{for all }n\ge 0.
\end{aligned}
\end{equation}
\par We then construct a new KG $\mathcal{G}^\star$ with fresh entities $\{a,b,c,d,x,y\}$ and triplets $\{<x, r_j, b>,<y, r_s, d>\}$, and no other edges incident to $b$ or $d$.
\par Consider two query triplets
\begin{equation}
\small
\nonumber
\begin{aligned}
<a, r_q, b> \qquad\text{and}\qquad <c, r_q, d>,
\end{aligned}
\end{equation}
for the original SIL model, the incoming relation colors attached to $b$ and $d$ are the same because $r_j$ and $r_s$ are indistinguishable by the first-stage WL test. Moreover, the local entity neighborhoods of $b$ and $d$ are isomorphic because each target node has exactly one incoming edge from a fresh non-source node. Based on the above conditions, we can obtain
\begin{equation}
\small
\nonumber
\begin{aligned}
\mathrm{col}_e^{(\ell)}(b|a,r_q)=\mathrm{col}_e^{(\ell)}(d|c,r_q) \qquad\text{for all } \ell\ge 0.
\end{aligned}
\end{equation}
\par For RCMP, we already have $\widetilde{\mathrm{col}}_r^{(0)}(r_j|e_u,r_q)\neq\widetilde{\mathrm{col}}_r^{(0)}(r_s|e_u,r_q)$. 
Therefore, after a first-stage update, the message multiset received by $b$ contains
\begin{equation}
\small
\nonumber
\begin{aligned}
\bigl(
\widetilde{\mathrm{col}}^{(0)}_e(x|a,r_q),
\widetilde{\mathrm{col}}^{(N)}_r(r_j|a,r_q)
\bigr),
\end{aligned}
\end{equation}
whereas the one received by $d$ contains
\begin{equation}
\small
\nonumber
\begin{aligned}
\bigl(
\widetilde{\mathrm{col}}^{(0)}_e(y|c,r_q),
\widetilde{\mathrm{col}}^{(N)}_r(r_s|c,r_q)
\bigr).
\end{aligned}
\end{equation}
Based on $\widetilde{\mathrm{col}}_e^{(0)}(e_v|e_u,r_q)=\mathbb{I}(e_v=e_u) \cdot \widetilde{\mathrm{col}}_r^{(N)}(r_q|e_u,r_q)$, the first components are equal, but the second components are different, so the
multisets in the second-stage test are different. Finally, we can obtain
\begin{equation}
\small
\nonumber
\begin{aligned}
\widetilde{\mathrm{col}}_e^{(\ell)}(b|a,r_q)\neq\widetilde{\mathrm{col}}_e^{(\ell)}(d|c,r_q) \qquad\text{for all } \ell\ge 0.
\end{aligned}
\end{equation}
Again using the constructive direction of \textbf{Assumption~\ref{assume:B}4}, there exists an
instance where RCMP can separate but SIL cannot. This proves the strict refinement.
\end{proof}
\par \textbf{Discussion.} The proposition identifies a novel perspective of expressiveness beyond changing the motif set $\mathcal{R}^*$ itself. The original theory strengthens SIL by enriching the relational hypergraph through richer motifs. Our construction keeps $\mathcal{R}^*$ fixed and instead strengthens the base coloring of the first-stage relation process by injecting a structure-only autoregressive rule signal. Therefore, the gain comes from refining the initial partition of relations rather than from
adding new motifs. Moreover, the structure-only prompting design is crucial. The invariance statement depends on the equivariance of the generator with respect to isomorphisms, which would
generally fail if one injected relation names, entity names, or arbitrary textual
attributes into the prompt.
\begin{algorithm}[t]
\caption{\label{algorithm1}Pre-training framework of SIRLM}
\raggedright
{\bf Input:} KG $\mathcal{G}=(\mathcal{E}, \mathcal{R}, \mathcal{T})$; relational graph $\mathcal{G}_r=(\mathcal{R}, \mathcal{R}^*, \mathcal{T}^*)$; trainable parameter set of SIRLM $\Omega$; learning rate $\eta$; max training step $s$; batch size $b$.\\
{\bf Output:}
Optimized parameter set $\Omega$.
\begin{algorithmic}[1]
\State $step=0$
\For{$step<s$}
    \State Randomly select $b$ query triplets from $\mathcal{T}$ to form $\mathcal{T}_q$
    \State $\mathcal{L}_{total}=0$
    \For{$<e_h, r_q, ?>$ in $\mathcal{T}_q$}
        \State Construct the query instruction $X$ of $<e_h, r_q, ?>$ according to Appendix~\ref{appendixInstruction}
        \State Obtain structural embeddings $\bm{R}_{|q}$ and $\bm{E}_{|h}$ according to Eqs (\ref{eq1}) and (\ref{eq2}), respectively
        \State Obtain the tokenizer $\text{TKN}_{\text{RG}}$ according to Eq. (\ref{eq6})
        \State Obtain structural relation memory $\bm{R}_{\text{mem}}$ by Eq. (\ref{eq7})
        \State Convert $X$ into token sequence using $\text{TKN}_{\text{RG}}$ and conduct in-context learning with
        \Statex \hspace{10mm}$\bm{R}_{\text{mem}}$ according to Eq. (\ref{eq8})
        \State Generate the rule body $\rho=\bigwedge_{z=1}^{\varepsilon} r_z$ with the hidden states ${\{\bm{h}_{L+z}\}}_{z=1}^{\varepsilon}$ using Eq. (\ref{eq10})
        \State Reason the missing entity in the query triplet according to Eqs. (\ref{eq11})-(\ref{eq13})
        \State Calculate the loss $\mathcal{L}_{\text{SFT}}$ using Eq. (\ref{eq14})
        \State $\mathcal{L}_{total}\leftarrow\mathcal{L}_{total}+\mathcal{L}_{\text{SFT}}$
    \EndFor
    \State \textbf{end for}
    \State Optimize trainable parameters according to $\Omega\leftarrow\Omega-\eta\nabla(\mathcal{L}_{total})$
    \State $step\leftarrow step+1$
\EndFor
\State \textbf{end for}
\State {\bf return} $\Omega$
\end{algorithmic}
\end{algorithm}
\section{Training Algorithm}
\label{appendixAlgorithm}
Algorithm~\ref{algorithm1} provides a complete training process for $\text{SIRLM}_{\text{PT}}$. In each training iteration, we first convert a query triplet into a instruction form consisting of textual and structural tokens (\textbf{Step 6}). Next, we obtain the structural representaions of entities and relations to construct a tokenizer for the query instruciton (\textbf{Steps 7 and 8}). We then obtain the structural relation memory and conduct in-context learning for the query instruciton (\textbf{Steps 9 and 10}), which provides the hidden states for rule generation (\textbf{Step 11}) and reasoning (\textbf{Step 12}). Finally, we using the reasoning loss to update the model parameters (\textbf{Step 16}) and return a pre-trained parameter set for downstream KGR tasks.
\section{Computational Complexity}
\label{appendixComplexity}
\begin{table}[t]
\scriptsize
\centering
\scriptsize
\setlength{\tabcolsep}{2pt}
\renewcommand{\arraystretch}{1.15}
\caption{Comparison of KRLM and SIRLM in terms of TFLOPs, memory footprint, and wall-clock time under pre-training and fine-tuning settings.}
\label{TabTrainingCost}

\begin{tabular}{c cc cc cc}
\toprule
\multirow{2}{*}{\textbf{Metric}} 
& \multicolumn{2}{c}{\makecell{\textbf{Pre-training}\\(3 transductive datasets)}} 
& \multicolumn{2}{c}{\makecell{\textbf{Fine-tuning}\\ (FB15k237 v1)}} 
& \multicolumn{2}{c}{\makecell{\textbf{Fine-tuning}\\ (FB15k237-25)}} \\
\cmidrule(lr){2-3} \cmidrule(lr){4-5} \cmidrule(lr){6-7}
& \textbf{KRLM} & \textbf{$\text{SIRLM}_{\text{PT}}$} & \textbf{KRLM} & \textbf{$\text{SIRLM}_{\text{SFT}}$} & \textbf{KRLM} & \textbf{$\text{SIRLM}_{\text{SFT}}$} \\
\midrule
\makecell{\textbf{Avg. TFLOPs}\\(forward propagation 100 steps)}
& 3.3436$\pm$0.4540 & \textbf{2.0397$\pm$0.3841}
& 3.2755$\pm$0.5208 & \textbf{2.0683$\pm$0.0175}
& 3.3312$\pm$0.4859 & \textbf{2.2287$\pm$0.0000}\\
\hline
\textbf{Training memory} 
& 36.12 GB / GPU & \textbf{27.58 GB / GPU}
& 32.57 GB / GPU & \textbf{11.38 GB / GPU}
& 32.67 GB / GPU & \textbf{13.12 GB / GPU} \\
\hline
\textbf{Wall-clock time }
& \makecell{3h10m / epoch} & \makecell{\textbf{2h38m / epoch}}
& \makecell{7m28s / epoch} & \makecell{\textbf{3m39s / epoch}}
& \makecell{12m13s / epoch} & \makecell{\textbf{10m45s / epoch}} \\
\bottomrule
\end{tabular}
\end{table}
The computational complexity of SIRLM consists of three parts. 
For the construction of the relational graph $\mathcal{G}^*$, since we need to dynamically construct it based on negative sampling during the training process, the time complexity of this part needs to be considered~\citep{MOTIF}. To construct motif edges in $\mathcal{G}^*$, we first need to compress $\bm{A}\in\mathbb{R}^{|\mathcal{E}|\times|\mathcal{R}|\times|\mathcal{E}|}$ into two sparse matrices $\bm{A}_h\in\mathbb{R}^{|\mathcal{E}|\times|\mathcal{R}|}$ and $\bm{A}_t\in\mathbb{R}^{|\mathcal{R}|\times|\mathcal{E}|}$, with a complexity of $\mathcal{O}(|\mathcal{E}|^2|\mathcal{R}|)$. Next, the construction process of the four types of binary edges theoretically requires a complexity of $\mathcal{O}(|\mathcal{E}||\mathcal{R}|^2)$. In practical situations, we use the sparse  operator $\text{spmm}(\cdot)$ to reduce the complexity of constructing binary edges to $\mathcal{O}(\text{nnz}(\bm{X}|\bm{X}\in\{\bm{A}_h, \bm{A}_t\})\times\text{nnz}(\bm{Y}|\bm{Y}\in\{\bm{A}_h^{\text{T}}, \bm{A}_t^{\text{T}}\}))$, where $\text{nnz}(\cdot)$ is a operator finding the number of non-zero element in $\bm{X}$ and 
$\bm{Y}$. Similarly, the theoretical complexity and actual sparse computational complexity of constructing ternary edges are $\mathcal{O}(|\mathcal{E}|^2|\mathcal{R}|^2+|\mathcal{E}||\mathcal{R}|^3)$ and $\mathcal{O}(\text{nnz}(\bm{X}|\bm{X}\in\{\bm{A}_h, \bm{A}_t\})\times\text{nnz}(\bm{A})\times\text{nnz}(\bm{Y}|\bm{Y}\in\{\bm{A}_h^{\text{T}}, \bm{A}_t^{\text{T}}\}))$, respectively
\par From the perspective of the KG tokenizer and RCMP reasoner, the time complexity is upper-bounded by the NBFNet executed on $\mathcal{G}$, as $|\mathcal{R}|\ll|\mathcal{E}|$. For each layer on NBFNet, the reasoning time complexity is calculated as $\mathcal{O}(|\mathcal{T}|d+|\mathcal{E}|d^2)$, where $d$ is the embedding dimension. Therefore, for a $N$-layer NBFNet executed on $\mathcal{G}$, its overall time complexity is $\mathcal{O}(N(|\mathcal{T}|d+|\mathcal{E}|d^2))$. Furthermore, by using the efficient relational messaging kernel in the Pytorch-geometric library, the complexity of the NBFNet is optimized to $\mathcal{O}(N|\mathcal{E}|d)$~\citep{ULTRA}.
\par The complexity of the in-contex learning module with the sturctural relation memory can be divided into the self-attention matrix calculation ($\mathcal{O}(L^2F)$), the memory key-value calculation ($\mathcal{O}(LKd)$), and the final aggregation ($\mathcal{O}(L(L+K)F)$), where $F$ is the hidden dimensions of LLM. Because $L\gg K$, the complexity upper bound of the in-contex learning module can be represented as $\mathcal{O}(L(L+K)F)$.
\par We further provide Table~\ref{TabTrainingCost}, which shows the TFLOPs in forward propagation, memory footprint, and wall-clock time of $\text{SIRLM}_{\text{PT}}$ and $\text{SIRLM}_{\text{SFT}}$ under the condition of batch\_size = 4 per GPU × 4 GPUs. For the SFT paradigm, we include results on the largest inductive dataset (FB15k237-25) and the smallest inductive dataset (FB15k237 v1) to provide a boundary of the computational cost.
\section{Datasets}
\label{appendixDataset}
\begin{table}
\setlength\tabcolsep{1pt}
\tiny
\newcommand{\tabincell}[2]{\begin{tabular}{@{}#1@{}}#2\end{tabular}}
\caption{\label{TabDatasets}
Element statistics of KGR dataset. “\textbf{Triplets}” represents the number of total triplets contained in a training/validation/testing graph.  “\textbf{\#Valid}” and  “\textbf{\#Test}” are the number of evaluation triplets in the validation and testing graph, respectively.
}
\centering
\begin{threeparttable}
\begin{tabular}{ c  c  c  c  c  c  c  c  c  c  c  c  c  c  c  c}
\hline
\multirow{2}{*}[-0.05in]{\tabincell{c}{\textbf{Type}}} & \multirow{2}{*}[-0.05in]{\tabincell{c}{\textbf{Dataset}}} & \multicolumn{3}{c}{\textbf{Training graph}} & \multicolumn{4}{c}{\textbf{Validation Graph}} & \multicolumn{4}{c}{\textbf{Testing Graph}}\\
\cmidrule(r){3-5}\cmidrule(r){6-9}\cmidrule(r){10-13}
 & & \textbf{Entities} & \textbf{Relations} & \textbf{Triplets} & \textbf{Entities} & \textbf{Relations} & \textbf{Triplets} & \textbf{\#Valid} & \textbf{Entities} & \textbf{Relations} & \textbf{Triplets} & \textbf{\#Test}\\
\hline
\multirow{4}{*}[-0.01in]{\tabincell{c}{\textbf{Transductive}}} & FB15k237~\citep{FB15k237} & 14541 & 237 & 272115 & 14541 & 237 & 272115 & 17535 & 14541 & 237 & 272115 & 20466\\
& CoDEx-M~\citep{CoDEx} & 17050 & 51 & 185584 & 17050 & 51 & 185584 & 10310 & 17050 & 51 & 185584 & 10311\\
& WN18RR~\citep{WN18RR} & 40943 & 11 & 86835 & 40943 & 11 & 86835 & 3034 & 40943 & 11 & 86835 & 3134\\
& NELL995~\citep{DeepPathNELL995} & 74536 & 200 & 149678 & 74536 & 200 & 149678 & 543 & 74536 & 200 & 149678 & 2818\\
\hline
\multirow{12}{*}[-0.01in]{\tabincell{c}{\textbf{IndE}}} & FB15k237 v1~\citep{GraIL} & 1594 & 180 & 4245 & 1594 & 180 & 4245 & 489 & 1093 & 180 & 1993 & 411\\
& FB15k237 v2~\citep{GraIL} & 2608 & 200 & 9739 & 2608 & 200 & 9739 & 1166 & 1660 & 200 & 4145 & 947\\
& FB15k237 v3~\citep{GraIL} & 3668 & 215 & 17986 & 3668 & 215 & 17986 & 2194 & 2501 & 215 & 7406 & 1731\\
& FB15k237 v4~\citep{GraIL} & 4707 & 219 & 27203 & 4707 & 219 & 27203 & 3352 & 3051 & 219 & 11714 & 2840\\
& WN18RR v1~\citep{GraIL} & 2746 & 9 & 5410 & 2746 & 9 & 5410 & 630 & 922 & 9 & 1618 & 373\\
& WN18RR v2~\citep{GraIL} & 6954 & 10 & 15262 & 6954 & 10 & 15262 & 1838 & 2757 & 10 & 4011 & 852\\
& WN18RR v3~\citep{GraIL} & 12078 & 11 & 25901 & 12078 & 11 & 25901 & 3097 & 5084 & 11 & 6327 & 1143\\
& WN18RR v4~\citep{GraIL} & 3861 & 9 & 7940 & 3861 & 9 & 7940 & 934 & 7084 & 9 & 12334 & 2823\\
& NELL995 v1~\citep{GraIL} & 3103 & 14 & 4687 & 3103 & 14 & 4687 & 414 & 225 & 14 & 833 & 201\\
& NELL995 v2~\citep{GraIL} & 2564 & 88 & 8219 & 2564 & 88 & 8219 & 922 & 2086 & 88 & 4586 & 935\\
& NELL995 v3~\citep{GraIL} & 4647 & 142 & 16393 & 4647 & 142 & 16393 & 1851 & 3566 & 142 & 8048 & 1620\\
& NELL995 v4~\citep{GraIL} & 2092 & 76 & 7546 & 2092 & 76 & 7546 & 876 & 2795 & 76 & 7073 & 1447\\
\hline
\multirow{20}{*}[-0.01in]{\tabincell{c}{\textbf{FullInd}}} & FB15k237-25~\citep{InGram} & 5190 & 163 & 91571 & 4097 & 216 & 17147 & 5716 & 4097 & 216 & 17147 & 5716\\
& FB15k237-50~\citep{InGram} & 5190 & 153 & 85375 & 4445 & 205 & 11636 & 3879 & 4445 & 205 & 11636 & 3879\\
& FB15k237-75~\citep{InGram} & 4659 & 134 & 62809 & 2792 & 186 & 9316 & 3106 & 2792 & 186 & 9316 & 3106\\
& FB15k237-100~\citep{InGram} & 4659 & 134 & 62809 & 2624 & 77 & 6987 & 2329 & 2624 & 77 & 6987 & 2329\\
& NELL995-25~\citep{InGram} & 4396 & 106 & 17578 & 2146 & 120 & 2230 & 743 & 2146 & 120 & 2230 & 744\\
& NELL995-50~\citep{InGram} & 4396 & 106 & 17578 & 2335 & 119 & 2576 & 859 & 2335 & 119 & 2576 & 859\\
& NELL995-75~\citep{InGram} & 2607 & 96 & 11058 & 1578 & 116 & 1818 & 606 & 1578 & 116 & 1818 & 607\\
& NELL995-100~\citep{InGram} & 1258 & 55 & 7832 & 1709 & 53 & 2378 & 793 & 1709 & 53 & 2378 & 793\\
& Wikidata68K-25~\citep{InGram} & 12659 & 47 & 41873 & 3228 & 74 & 3391 & 1130 & 3228 & 74 & 3391 & 1131\\
& Wikidata68K-50~\citep{InGram} & 12022 & 72 & 82481 & 9328 & 93 & 9672 & 3224 & 9328 & 93 & 9672 & 3225\\
& Wikidata68K-75~\citep{InGram} & 6853 & 52 & 28741 & 2722 & 65 & 3430 & 1143 & 2722 & 65 & 3430 & 1144\\
& Wikidata68K-100~\citep{InGram} & 9784 & 67 & 49875 & 12136 & 37 & 13487 & 4496 & 12136 & 37 & 13487 & 4496\\
& MTDEA1-tax~\citep{MTDEA} & 10000 & 10 & 17178 & 10000 & 10 & 17178 & 1908 & 10000 & 9 & 16526 & 1834\\
& MTDEA1-health~\citep{MTDEA} & 10000 & 7 & 14371 & 10000 & 7 & 14371 & 1596 & 10000 & 7 & 14110 & 1566\\
& MTDEA2-org~\citep{MTDEA} & 10000 & 10 & 23233 & 10000 & 10 & 23233 & 2581 & 10000 & 11 & 21976 & 2441\\
& MTDEA2-sci~\citep{MTDEA} & 10000 & 16 & 16471 & 10000 & 16 & 16471 & 1830 & 10000 & 16 & 14852 & 1650\\
& MTDEA3-art~\citep{MTDEA} & 10000 & 45 & 27262 & 10000 & 45 & 27262 & 3026 & 10000 & 45 & 28023 & 3113\\
& MTDEA3-infra~\citep{MTDEA} & 10000 & 24 & 21990 & 10000 & 24 & 21990 & 2443 & 10000 & 27 & 21646 & 2405\\
& MTDEA4-sci~\citep{MTDEA} & 10000 & 42 & 12576 & 10000 & 42 & 12576 & 1397 & 10000 & 42 & 12516 & 1388\\
& MTDEA4-health~\citep{MTDEA} & 10000 & 21 & 15539 & 10000 & 21 & 15539 & 1725 & 10000 & 20 & 15337 & 1703\\
\hline
\end{tabular}
\end{threeparttable}
\end{table}
In our experiments, we conduct evaluations on 36 datasets. According to the overlap level between train KG $\mathcal{G}_{train}=(\mathcal{E}_{train},\mathcal{R}_{train},\mathcal{T}_{train})$ and test KG $\mathcal{G}_{test}=(\mathcal{E}_{test},\mathcal{R}_{test},\mathcal{T}_{test})$, these datasets can be divided into the following three categories:
\begin{enumerate}[left=0pt]
\item[$\bullet$] \textbf{Transductive} datasets that $\mathcal{E}_{test}=\mathcal{E}_{train}$ and $\mathcal{R}_{test}=\mathcal{R}_{train}$: FB15k-237~\citep{FB15k237}, WN18RR~\citep{WN18RR}, CoDEx-M~\citep{CoDEx}, and NELL995~\citep{DeepPathNELL995}.
\item[$\bullet$] \textbf{Ind}uctive \textbf{E}ntity (\textbf{IndE}) datasets that $\mathcal{E}_{test}\neq\mathcal{E}_{train}$ and $\mathcal{R}_{test}=\mathcal{R}_{train}$, including 12 datasets from GraIL~\citep{GraIL} (FB15k237 V1, FB15k237 V2, FB15k237 V3, FB15k237 V4, WN18RR V1, WN18RR V2, WN18RR V3, WN18RR V4, NELL995 V1, NELL995 V2, NELL995 V3, and NELL995 V4).
\item[$\bullet$] \textbf{Full}y \textbf{Ind}uctive  (\textbf{FullInd}) datasets that $\mathcal{E}_{test}\neq\mathcal{E}_{train}$ and $\mathcal{R}_{test}\neq\mathcal{R}_{train}$, including 20 datasets from InGram~\citep{InGram} (FB15k237-25, FB15k237-50, FB15k237-75, FB15k237-100, NELL995-25, NELL995-50, NELL995-75, NELL995-100, Wikidata68K-25, Wikidata68K-50, Wikidata68K-75, and Wikidata68K-100) and MTDEA~\citep{MTDEA} (MTDEA1-tax, MTDEA1-health, MTDEA2-org, MTDEA2-sci, MTDEA3-art, MTDEA3-infra, MTDEA4-sci, and MTDEA4-health).
\end{enumerate}
\par These dataset are used to evaluate the model in the four training paradigms mentioned in Section~\ref{Setting}. Table~\ref{TabDatasets} provides detailed elemental statistics for these datasets.
\begin{figure}
    \begin{centering}
      \includegraphics[width=4.6in]{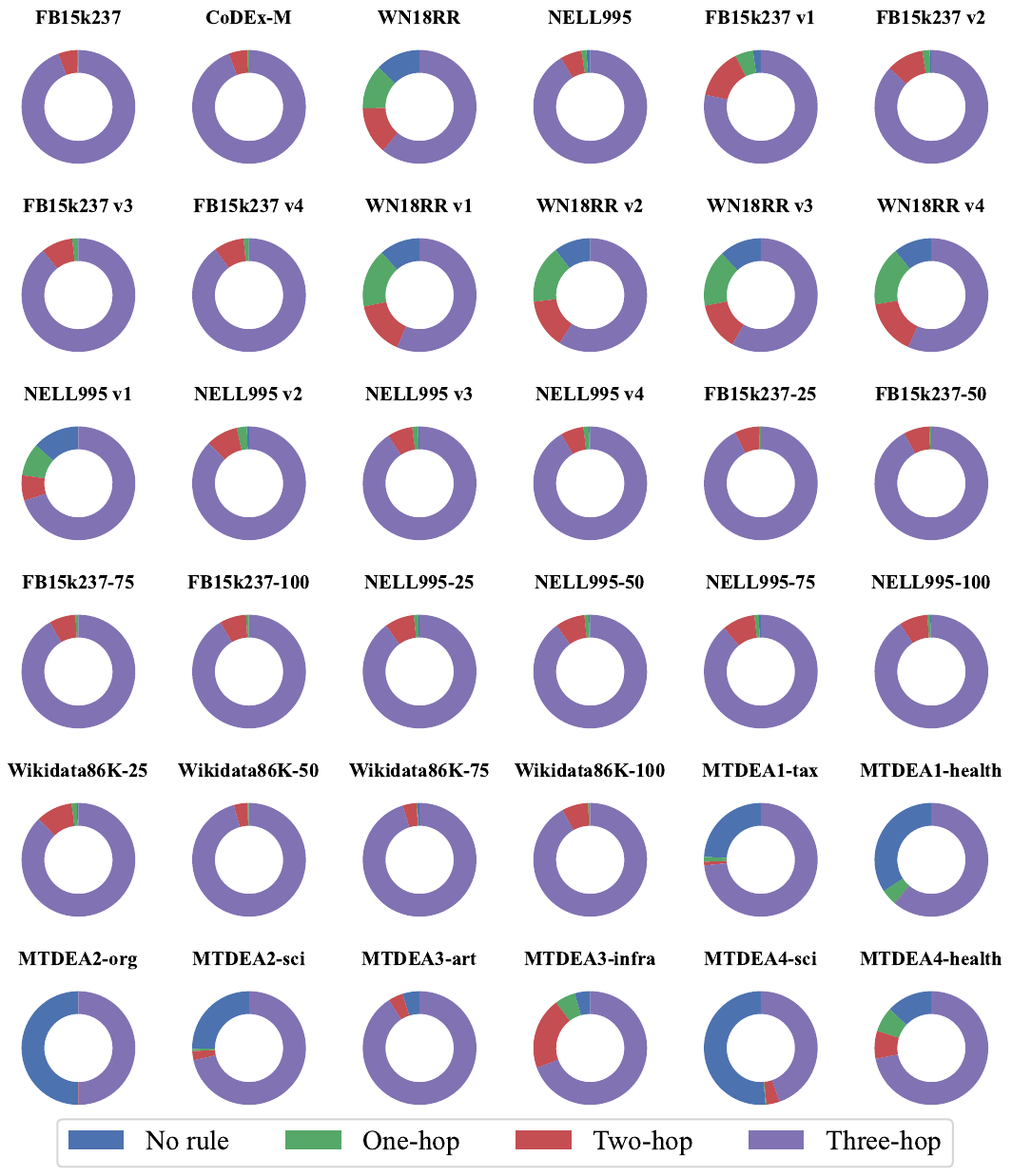}
      \caption{\label{FigRuleLenDistribution}The length distribution of rules mined from various datasets.}
\vspace{-5mm}
    \end{centering}
\end{figure}
\par Next, we discuss the rule mining approach in the data preprocessing stage. For each complete triplet in the training KG, we extract all closed paths within three hops. Due to the sparsity of the KG, some triplets may not have any closed paths of limited length that can be extracted. Figure~\ref{FigRuleLenDistribution} presents the proportion of triplets without rules (No rule) and the distribution of rules at different lengths across datasets. It can be observed that datasets such as WN18RR and the MTDEA1, MTDEA2, and MTDEA4 series exhibit a relatively high proportion of rule-less triplets. This is one of the reasons that the performance gains of SIRLM on these datasets, as shown in Tables~\ref{TabMainResults},~\ref{TabMainIndEResults}, and~\ref{TabMainFullIndResults}, are generally lower than on other datasets, \textit{i.e.}, the difficulty in obtaining comprehensive rule-based contextual information from the KG.
\par After extracting the closed paths, we handle triplets without any rules by directly using the query relation of the triplet as the rule body. For the remaining triplets that contain at least one closed path, we further construct candidate rule bodies by sequentially aggregating the directed relation sequences along each closed path. We then count the occurrence frequency of all candidate rule bodies for each triplet and select the most frequent one as the rule to be generated for that triplet, \textit{i.e.}, the masked component of the query instruction described in \textbf{Appendix~\ref{appendixInstruction}}.
\section{Experimental Hyperparameter Settings}
\label{appendixHyperparameter}
\begin{table}
\setlength\tabcolsep{2pt}
\renewcommand{\arraystretch}{1.1}
\scriptsize
\newcommand{\tabincell}[2]{\begin{tabular}{@{}#1@{}}#2\end{tabular}}
\caption{\label{TabMainHyperparameter}
Hyperperameters of KRLM used in pre-training and end-to-end training from scratch.
}
\centering
\begin{threeparttable}
\begin{tabular}{ c | c | c }
\hline
\textbf{Module} & \textbf{Component} & \textbf{Parameter} \\
\hline
\multirow{4}{*}[-0.4in]{\tabincell{c}{{\textbf{KG Tokenizer}}}} & NBFNet on $\mathcal{G}_r$ & \tabincell{c}{Layer number $N=6$\\Hidden dim $d=64$\\Message function $\textsc{Msg}(\cdot)=$ DistMult\\ Aggregation function $\textsc{Agg}(\cdot)=$ Sum\\Updating function $\textsc{Up}(\cdot)=\text{nn.Linear}(128, 64)$\\Motif Edge embeddings $\bm{R}^*=\text{nn.Embedding}(8, 64)$}\\
\cline{2-3}
 & NBFNet on $\mathcal{G}$ & \tabincell{c}{Layer number $N=6$\\Hidden dim $d=64$\\Message function $\textsc{Msg}(\cdot)=$ DistMult\\ Aggregation function $\textsc{Agg}(\cdot)=$ Sum\\Updating function $\textsc{Up}(\cdot)=\text{nn.Linear}(128, 64)$} \\
\cline{2-3}
 & Up-scale layer $f_\text{up}(\cdot)$ & \tabincell{c}{$\text{nn.Linear}(64, 1536)$} \\
\cline{2-3}
 & Score function $\mathcal{S}_\text{KG}(\cdot)$ & \tabincell{c}{$\text{nn.Linear}(128, 64)$\\$\text{nn.ReLU}(\cdot)$\\$\text{nn.Linear}(64, 1)$} \\
\hline
\multirow{4}{*}[-0.25in]{\tabincell{c}{{\textbf{In-context Learning Module}}}} & Qwen2.5-1.5b backbone & \tabincell{c}{Default configuration} \\
\cline{2-3}
 & LoRA configuration & \tabincell{c}{$r=64$\\$\alpha=32$\\$\text{droupout}=0.1$\\$\text{target module}=\text{[gate\_proj, up\_proj, down\_proj]}$}\\
\cline{2-3}
 & Memory key layer $m_K(\cdot)$ & $\text{nn.Linear}(64,1536)$ \\
\cline{2-3}
 & Memory value layer $m_V(\cdot)$ & $\text{nn.Linear}(64,1536)$ \\
\cline{2-3}
 & Scale of structural relation memory & $K=30$ \\
\hline
\textbf{Next-relation Predictor} & Porject layer $f_{\bm{\Theta}}(\cdot)$ & \tabincell{c}{$\text{nn.Linear}(1600, 1600)$\\$\text{nn.ReLU}(\cdot)$\\$\text{nn.Linear}(1600, 1)$}\\
\hline
\multirow{4}{*}[-0.4in]{\tabincell{c}{{\textbf{RCMP Reasoner}}}} & NBFNet on $\mathcal{G}_r$ & \tabincell{c}{Layer number $N=6$\\Hidden dim $d=64$\\Message function $\textsc{Msg}(\cdot)=$ DistMult\\ Aggregation function $\textsc{Agg}(\cdot)=$ Sum\\Updating function $\textsc{Up}(\cdot)=\text{nn.Linear}(128, 64)$\\Motif Edge embeddings $\hat{\bm{R}}^*=\text{nn.Embedding}(8, 64)$}\\
\cline{2-3}
 & NBFNet on $\mathcal{G}$ & \tabincell{c}{Layer number $N=6$\\Hidden dim $d=64$\\Message function $\textsc{Msg}(\cdot)=$ DistMult\\ Aggregation function $\textsc{Agg}(\cdot)=$ Sum\\Updating function $\textsc{Up}(\cdot)=\text{nn.Linear}(128, 64)$} \\
\cline{2-3}
 & Down-scale layer $f_\text{down}(\cdot)$ & \tabincell{c}{$\text{nn.Linear}(1536, 64)$} \\
\cline{2-3}
 & Score function $\mathcal{S}_\text{RCMP}(\cdot)$ & \tabincell{c}{$\text{nn.Linear}(128, 64)$\\$\text{nn.ReLU}(\cdot)$\\$\text{nn.Linear}(64, 1)$} \\
\hline
\multirow{3}{*}[0.05in]{\tabincell{c}{{\textbf{Training}}}} & Optimizer & AdamW \\
\cline{2-3}
 & Number of negative samples & 512\\
\hline
\end{tabular}
\end{threeparttable}
\vspace{-5mm}
\end{table}
\begin{table}
\setlength\tabcolsep{20pt}
\tiny
\newcommand{\tabincell}[2]{\begin{tabular}{@{}#1@{}}#2\end{tabular}}
\caption{\label{TabTrainingHyperparameter}
Detailed Training parameters, where $b$, $\eta$, and $M$ represent batch size, learning rate, and GRPO sampling number, respectively. “all” means that each epoch needs to iterate through all training queries.
}
\centering
\begin{threeparttable}
\begin{tabular}{ c | c  c  c }
\hline
\textbf{Datasets} & \tabincell{c}{\textbf{$\text{SIRLM}_{\text{E2E}}$}\\($b$, $\eta$, epoch, step)} & \tabincell{c}{\textbf{$\text{SIRLM}_{\text{SFT}}$}\\($b$, $\eta$, epoch, step)} & \tabincell{c}{\textbf{$\text{SIRLM}_{\text{GRPO}}$}\\($b$, $\eta$, $M$, epoch, step)} \\
\hline
FB15k237 & (12, 5e-4, 10, all) & (12, 1e-4, 3, all) & (8, 1e-5, 8, 1, all) \\
CoDEx-M & (24, 1e-4, 10, all) & (24, 5e-5, 3, all) & (8, 1e-5, 8, 1, all)\\
WN18RR & (12, 5e-4, 10, all) & (12, 1e-4, 3, all) &  (8, 1e-5, 8, 3, all)\\
NELL995 & (12, 5e-4, 10, all) & (12, 1e-4, 3, all) &  (8, 1e-5, 8, 1, all)\\
\hline
FB15k237 v1 & (24, 1e-4, 10, all) & (24, 5e-5, 3, all) & (16, 1e-5, 8, 3, all) \\
FB15k237 v2 & (24, 1e-4, 10, all) & (24, 5e-5, 3, all) & (16, 1e-5, 8, 3, all) \\
FB15k237 v3 & (24, 1e-4, 10, all) & (24, 5e-5, 3, all) & (16, 1e-5, 8, 3, all) \\
FB15k237 v4 & (24, 1e-4, 5, all) & (24, 5e-5, 3, all) & (16, 1e-5, 8, 3, all) \\
WN18RR v1 & (48, 1e-4, 10, all) & (48, 5e-5, 3, all) & (16, 1e-5, 8, 3, all) \\
WN18RR v2 & (48, 1e-4, 10, all) & (48, 5e-5, 3, all) & (16, 1e-5, 8, 3, all) \\
WN18RR v3 & (48, 1e-4, 20, all) & (48, 5e-5, 5, all) & (16, 1e-5, 8, 3, all) \\
WN18RR v4 & (48, 1e-4, 10, all) & (48, 5e-5, 3, all) & (16, 1e-5, 8, 3, all) \\
NELL995 v1 & (32, 1e-4, 5, all) & (32, 5e-5, 3, all) & (16, 1e-5, 8, 3, all) \\
NELL995 v2 & (32, 1e-4, 10, all) & (32, 5e-5, 3, all) & (16, 1e-5, 8, 3, all) \\
NELL995 v3 & (32, 1e-4, 5, all) & (32, 5e-5, 3, all) & (16, 1e-5, 8, 3, all) \\
NELL995 v4 & (32, 1e-4, 5, all) & (32, 5e-5, 3, all) & (16, 1e-5, 8, 3, all) \\
\hline
FB15k237-25 & (32, 1e-4, 10, all) & (32, 5e-5, 3, all) & (8, 1e-5, 8, 3, all) \\
FB15k237-50 & (32, 1e-4, 10, all) & (32, 5e-5, 3, all) & (8, 1e-5, 8, 3, all) \\
FB15k237-75 & (32, 1e-4, 10, all) & (32, 5e-5, 3, all) & (8, 1e-5, 8, 3, all) \\
FB15k237-100 & (32, 1e-4, 10, all) & (32, 5e-5, 3, all) & (8, 1e-5, 8, 3, all) \\
NELL995-25 & (32, 1e-4, 10, all) & (32, 5e-5, 3, all) & (8, 1e-5, 8, 3, all) \\
NELL995-50 & (32, 1e-4, 10, all) & (32, 5e-5, 3, all) & (8, 1e-5, 8, 3, all) \\
NELL995-75 & (32, 1e-4, 10, all) & (32, 5e-5, 3, all) & (8, 1e-5, 8, 3, all) \\
NELL995-100 & (32, 1e-4, 10, all) & (32, 5e-5, 3, all) & (8, 1e-5, 8, 3, all) \\
Wikidata68K-25 & (32, 1e-4, 10, all) & (32, 5e-5, 3, all) & (16, 1e-5, 8, 3, all) \\
Wikidata68K-50 & (32, 1e-4, 10, all) & (32, 5e-5, 3, all) & (16, 1e-5, 8, 3, all) \\
Wikidata68K-75 & (32, 1e-4, 10, all) & (32, 5e-5, 3, all) & (16, 1e-5, 8, 3, all) \\
Wikidata68K-100 & (32, 1e-4, 10, all) & (32, 5e-5, 3, all) & (16, 1e-5, 8, 3, all) \\
MTDEA1-tax & (32, 1e-4, 10, all) & (32, 5e-5, 3, all) & (16, 1e-5, 8, 3, all) \\
MTDEA1-health & (32, 1e-4, 10, all) & (32, 5e-5, 3, all) & (16, 1e-5, 8, 3, all) \\
MTDEA2-org & (32, 1e-4, 10, all) & (32, 5e-5, 3, all) & (16, 1e-5, 8, 3, all) \\
MTDEA2-sci & (32, 1e-4, 10, all) & (32, 5e-5, 3, all) & (16, 1e-5, 8, 3, all) \\
MTDEA3-art  & (32, 1e-4, 10, all) & (32, 5e-5, 3, all) & (16, 1e-5, 8, 3, all) \\
MTDEA3-infra & (32, 1e-4, 10, all) & (32, 5e-5, 3, all) & (16, 1e-5, 8, 3, all) \\
MTDEA4-sci & (32, 1e-4, 10, all) & (32, 5e-5, 3, all) & (16, 1e-5, 8, 3, all) \\
MTDEA4-health & (32, 1e-4, 10, all) & (32, 5e-5, 3, all) & (16, 1e-5, 8, 3, all) \\
\hline
\end{tabular}
\end{threeparttable}
\end{table}
In Section~\ref{MainResults}, we report four training paradigms of SIRLM, including End-to-End (E2E) training from scratch, Pre-Training (PT), SFT, and GRPO post-training. The hyperparameters of the model architecture under the four paradigms are uniformly set to the values in Table~\ref{TabMainHyperparameter}. In the PT paradigm, we set the learning rate to 1e-4 and use the AdamW optimizer with a 1\% warm-up step. The batch size per GPU is 12. Each epoch consists of 5000 iterations and the training process is conducted for a total of 20 epochs. More detailed training hyperparameters for the other training paradigms are provided in Table~\ref{TabTrainingHyperparameter}.
\section{Details Experimental Analysis}
\label{appendixExperimentalResults}
\subsection{Additional Analysis of Main Experiments}
\label{appendixMainResults}
\begin{table}
\setlength\tabcolsep{8.5pt}
\renewcommand{\arraystretch}{0.95}
\tiny
\newcommand{\tabincell}[2]{\begin{tabular}{@{}#1@{}}#2\end{tabular}}
\caption{\label{TabMainIndEResults}
The overall performance of various methods on IndE datasets. The colored cells represent the \colorbox{myRed!80}{best}, \colorbox{myOrange!80}{second-best}, and \colorbox{myYellow!80}{third-best} values, respectively.
}
\centering
\begin{threeparttable}
\begin{tabular}{ c  c  c  c  c  c  c  c  c }
\hline
\multirow{2}{*}[-0.05in]{\tabincell{c}{\textbf{Methods}}} & \multicolumn{2}{c}{\raisebox{-0.5ex}{\textbf{FB15k237 v1}}} & \multicolumn{2}{c}{\raisebox{-0.5ex}{\textbf{FB15k237 v2}}} & \multicolumn{2}{c}{\raisebox{-0.5ex}{\textbf{FB15k237 v3}}} & \multicolumn{2}{c}{\raisebox{-0.5ex}{\textbf{FB15k237 v4}}} \\
\cmidrule(r){2-3}\cmidrule(r){4-5}\cmidrule(r){6-7}\cmidrule(r){8-9}
 & \textbf{MRR}$\uparrow$ & \textbf{Hit10}$\uparrow$ & \textbf{MRR}$\uparrow$ & \textbf{Hit10}$\uparrow$ & \textbf{MRR}$\uparrow$ & \textbf{Hit10}$\uparrow$ & \textbf{MRR}$\uparrow$ & \textbf{Hit10}$\uparrow$\\
\hline
NeuralLP & 0.325 & 0.468 & 0.389 & 0.586 & 0.400 & 0.571 & 0.396 & 0.593\\
DRUM & 0.333 & 0.474 & 0.395 & 0.595 & 0.402 & 0.571 & 0.410 & 0.593\\
\hline
NBFNet & 0.422 & 0.574 & 0.514 & 0.685 & 0.476 & 0.637 & 0.453 & 0.627\\
RED-GNN & 0.369 & 0.483 & 0.469 & 0.629 & 0.445 & 0.603 & 0.442 & 0.621\\
ULTRA & 0.509 & 0.670 & 0.524 & 0.710 & 0.504 & 0.663 & 0.496 & 0.684\\
MOTIF & 0.530 & 0.702 & \cellcolor{myOrange!80}\textbf{0.557} & \cellcolor{myYellow!80}\textbf{0.744} & 0.519 & 0.684 & 0.508 & 0.695\\	
\hline
ChatRule\tnote{$\dagger$} & 0.448 & 0.546 & 0.484 & 0.628 & 0.471 & 0.653 & 0.457 & 0.598\\
MKGL\tnote{$\dagger$} & 0.475 & 0.595 & 0.508 & 0.681 & 0.486 & 0.643 & 0.471 & 0.645\\
KRLM\tnote{$\dagger$} & 0.511 & 0.668 & 0.527 & 0.712 & 0.520 & 0.681 & 0.505 & 0.693\\
\hline
$\text{SIRLM}_{\text{E2E}}$ & 0.548 & \cellcolor{myYellow!80}\textbf{0.710} & \cellcolor{myOrange!80}\textbf{0.557} & \cellcolor{myYellow!80}\textbf{0.744} & \cellcolor{myYellow!80}\textbf{0.527} & \cellcolor{myOrange!80}\textbf{0.705} & \cellcolor{myOrange!80}\textbf{0.512} & \cellcolor{myOrange!80}\textbf{0.708}\\
$\text{SIRLM}_{\text{PT}}$ & \cellcolor{myYellow!80}\textbf{0.552} & \cellcolor{myOrange!80}\textbf{0.718} & 0.543 & 0.741 & 0.510 & \cellcolor{myYellow!80}\textbf{0.702} & 0.497 & \cellcolor{myYellow!80}\textbf{0.705}\\
$\text{SIRLM}_{\text{SFT}}$ & \cellcolor{myOrange!80}\textbf{0.558} & \cellcolor{myOrange!80}\textbf{0.718} & \cellcolor{myYellow!80}\textbf{0.555} & \cellcolor{myOrange!80}\textbf{0.747} & \cellcolor{myOrange!80}\textbf{0.529} & \cellcolor{myOrange!80}\textbf{0.708} & \cellcolor{myYellow!80}\textbf{0.510} & 0.702\\
$\text{SIRLM}_{\text{GRPO}}$ &  \cellcolor{myRed!80}\textbf{0.561} & \cellcolor{myRed!80}\textbf{0.721} & \cellcolor{myRed!80}\textbf{0.559} & \cellcolor{myRed!80}\textbf{0.752} & \cellcolor{myRed!80}\textbf{0.532} & \cellcolor{myRed!80}\textbf{0.711} & \cellcolor{myRed!80}\textbf{0.515} & \cellcolor{myRed!80}\textbf{0.714}\\
\hline
\textcolor{upColor}{\textit{Avg. gain}}\tnote{$\star$} & \textcolor{upColor}{\textit{+12.50\%}} & \textcolor{upColor}{\textit{+14.50\%}} & \textcolor{upColor}{\textit{+7.38\%}} & \textcolor{upColor}{\textit{+8.87\%}} & \textcolor{upColor}{\textit{+6.28\%}} & \textcolor{upColor}{\textit{+7.70\%}} & \textcolor{upColor}{\textit{+5.52\%}} & \textcolor{upColor}{\textit{+7.52\%}}\\
\hline\hline
\multirow{2}{*}[-0.05in]{\tabincell{c}{\textbf{Methods}}} & \multicolumn{2}{c}{\raisebox{-0.5ex}{\textbf{WN18RR v1}}} & \multicolumn{2}{c}{\raisebox{-0.5ex}{\textbf{WN18RR v2}}} & \multicolumn{2}{c}{\raisebox{-0.5ex}{\textbf{WN18RR v3}}} & \multicolumn{2}{c}{\raisebox{-0.5ex}{\textbf{WN18RR v4}}} \\
\cmidrule(r){2-3}\cmidrule(r){4-5}\cmidrule(r){6-7}\cmidrule(r){8-9}
 & \textbf{MRR}$\uparrow$ & \textbf{Hit10}$\uparrow$ & \textbf{MRR}$\uparrow$ & \textbf{Hit10}$\uparrow$ & \textbf{MRR}$\uparrow$ & \textbf{Hit10}$\uparrow$ & \textbf{MRR}$\uparrow$ & \textbf{Hit10}$\uparrow$\\
\hline
NeuralLP &  0.649 & 0.772 & 0.635 & 0.749 & 0.361 & 0.476 & 0.628 & 0.706\\
DRUM & 0.666 & 0.777 & 0.646 & 0.747 & 0.380 & 0.477 & 0.627 & 0.702\\
\hline
NBFNet & \cellcolor{myOrange!80}\textbf{0.741} & \cellcolor{myRed!80}\textbf{0.826} & 0.704 & \cellcolor{myOrange!80}\textbf{0.798} & 0.432 & \cellcolor{myOrange!80}\textbf{0.568} & 0.641 & 0.694\\
RED-GNN & 0.701 & 0.799 & 0.690 & 0.780 & 0.427 & 0.524 & 0.651 & 0.721\\
ULTRA & 0.685 & 0.793 & 0.679 & 0.779 & 0.411 & 0.546 & 0.614 & 0.720\\
MOTIF & 0.703 & \cellcolor{myYellow!80}\textbf{0.806} & 0.680 & 0.781 & \cellcolor{myRed!80}\textbf{0.466} & \cellcolor{myRed!80}\textbf{0.590} & \cellcolor{myYellow!80}\textbf{0.659} & 0.733\\
\hline
ChatRule\tnote{$\dagger$} & 0.621 & 0.694 & 0.650 & 0.686 & 0.395 & 0.497 & 0.568 & 0.649\\
MKGL\tnote{$\dagger$} & \cellcolor{myRed!80}\textbf{0.746} & \cellcolor{myOrange!80}\textbf{0.822} & \cellcolor{myOrange!80}\textbf{0.712} & \cellcolor{myRed!80}\textbf{0.799} & 0.456 & 0.559 & \cellcolor{myRed!80}\textbf{0.664} & \cellcolor{myOrange!80}\textbf{0.741}\\
KRLM\tnote{$\dagger$} & 0.705 & 0.797 & 0.690 & \cellcolor{myYellow!80}\textbf{0.791} & \cellcolor{myYellow!80}\textbf{0.457} & \cellcolor{myRed!80}\textbf{0.590} & \cellcolor{myOrange!80}\textbf{0.662} & \cellcolor{myYellow!80}\textbf{0.737}\\
\hline
$\text{SIRLM}_{\text{E2E}}$ & 0.705 & 0.766 & 0.699 & 0.789 & 0.456 & 0.562 & 0.658 & 0.732\\
$\text{SIRLM}_{\text{PT}}$ & 0.668 & 0.744 & 0.670 & 0.766 & 0.433 & 0.549 & 0.627 & 0.693\\
$\text{SIRLM}_{\text{SFT}}$ & \cellcolor{myYellow!80}\textbf{0.706} & 0.778 & \cellcolor{myRed!80}\textbf{0.714} & \cellcolor{myRed!80}\textbf{0.799} & \cellcolor{myOrange!80}\textbf{0.459} & \cellcolor{myOrange!80}\textbf{0.568} & \cellcolor{myOrange!80}\textbf{0.662} & \cellcolor{myRed!80}\textbf{0.745}\\
$\text{SIRLM}_{\text{GRPO}}$ & 0.695 & 0.762 & \cellcolor{myYellow!80}\textbf{0.708} & \cellcolor{myYellow!80}\textbf{0.791} & 0.454 & \cellcolor{myYellow!80}\textbf{0.566} & \cellcolor{myOrange!80}\textbf{0.662} & \cellcolor{myOrange!80}\textbf{0.741}\\
\hline
\textcolor{upColor}{\textit{Avg. gain}}\tnote{$\star$} & \textcolor{upColor}{\textit{+0.38\%}} & \textcolor{upColor}{\textit{+0.21\%}} & \textcolor{upColor}{\textit{+3.78\%}} & \textcolor{upColor}{\textit{+3.12\%}} & \textcolor{upColor}{\textit{+3.84\%}} & \textcolor{upColor}{\textit{+3.17\%}} & \textcolor{upColor}{\textit{+2.71\%}} & \textcolor{upColor}{\textit{+3.36\%}}\\
\hline\hline
\multirow{2}{*}[-0.05in]{\tabincell{c}{\textbf{Methods}}} & \multicolumn{2}{c}{\raisebox{-0.5ex}{\textbf{NELL995 v1}}} & \multicolumn{2}{c}{\raisebox{-0.5ex}{\textbf{NELL995 v2}}} & \multicolumn{2}{c}{\raisebox{-0.5ex}{\textbf{NELL995 v3}}} & \multicolumn{2}{c}{\raisebox{-0.5ex}{\textbf{NELL995 v4}}} \\
\cmidrule(r){2-3}\cmidrule(r){4-5}\cmidrule(r){6-7}\cmidrule(r){8-9}
 & \textbf{MRR}$\uparrow$ & \textbf{Hit10}$\uparrow$ & \textbf{MRR}$\uparrow$ & \textbf{Hit10}$\uparrow$ & \textbf{MRR}$\uparrow$ & \textbf{Hit10}$\uparrow$ & \textbf{MRR}$\uparrow$ & \textbf{Hit10}$\uparrow$\\
\hline
NeuralLP & 0.610 & 0.871 & 0.361 & 0.564 & 0.367 & 0.576 & 0.261 & 0.539\\
DRUM & 	0.628 & 0.873 & 0.365 & 0.540 & 0.375 & 0.577 & 0.270 & 0.531\\
\hline
NBFNet & 0.648 & 0.862 & 0.421 & 0.599 & 0.462 & 0.578 & 0.404 & 0.588\\
RED-GNN & 0.637 & 0.866 & 0.419 & 0.601 & 0.436 & 0.594 & 0.363 & 0.556\\
ULTRA & 0.757 & 0.878 & 0.575 & 0.761 & 0.563 & 0.755 & 0.469 & 0.733\\
MOTIF & 0.712 & 0.873 & 0.566 & 0.765 & 0.580 & 0.764 & 0.507 & 0.740\\
\hline
ChatRule\tnote{$\dagger$} & 0.719 & 0.806 & 0.529 & 0.690 & 0.552  & 0.732 & 0.528 & 0.741\\
MKGL\tnote{$\dagger$} & \cellcolor{myYellow!80}\textbf{0.749} & \cellcolor{myYellow!80}\textbf{0.886} & 0.570 & 0.767 & 0.571 & 0.759 & 0.525 & 0.749\\
KRLM\tnote{$\dagger$} & 0.679 & \cellcolor{myOrange!80}\textbf{0.898} & 0.568 & 0.761 & 0.561 & 0.755 & \cellcolor{myRed!80}\textbf{0.554} & 0.762\\
\hline
$\text{SIRLM}_{\text{E2E}}$ & 0.680 & 0.871 & \cellcolor{myOrange!80}\textbf{0.580} & \cellcolor{myYellow!80}\textbf{0.769} & \cellcolor{myOrange!80}\textbf{0.585} & \cellcolor{myYellow!80}\textbf{0.769} & \cellcolor{myYellow!80}\textbf{0.547} & \cellcolor{myYellow!80}\textbf{0.767}\\
$\text{SIRLM}_{\text{PT}}$ & \cellcolor{myOrange!80}\textbf{0.762} & \cellcolor{myRed!80}\textbf{0.908} & \cellcolor{myYellow!80}\textbf{0.579} & \cellcolor{myOrange!80}\textbf{0.787} & \cellcolor{myYellow!80}\textbf{0.583} & \cellcolor{myOrange!80}\textbf{0.775} & 0.539 & 0.764\\
$\text{SIRLM}_{\text{SFT}}$ & \cellcolor{myRed!80}\textbf{0.768} & 0.879 & \cellcolor{myRed!80}\textbf{0.582} & \cellcolor{myRed!80}\textbf{0.789} & \cellcolor{myRed!80}\textbf{0.588} & \cellcolor{myRed!80}\textbf{0.779} & \cellcolor{myOrange!80}\textbf{0.552} & \cellcolor{myRed!80}\textbf{0.771}\\
$\text{SIRLM}_{\text{GRPO}}$ & 0.742 & 0.880 & 0.572 & 0.768 & 0.576 & 0.757 & 0.543 & \cellcolor{myOrange!80}\textbf{0.769}\\
\hline
\textcolor{upColor}{\textit{Avg. gain}}\tnote{$\star$} & \textcolor{upColor}{\textit{+8.59\%}} & \textcolor{upColor}{\textit{+3.99\%}} & \textcolor{upColor}{\textit{+9.60\%}} & \textcolor{upColor}{\textit{+11.70\%}} & \textcolor{upColor}{\textit{+9.17\%}} & \textcolor{upColor}{\textit{+10.23\%}} & \textcolor{upColor}{\textit{+12.08\%}} & \textcolor{upColor}{\textit{+11.11\%}}\\
\hline
\end{tabular}
\end{threeparttable}
\begin{tablenotes}
\item[$\dagger$] $\dagger$ We reproduce the experimental metrics of ChatRule, MKGL, and KRLM. 
\item[$\star$] $\star$ We calculate the average gain of the optimal value in the four training modes of SIRLM compared to all baselines.
\end{tablenotes}
\end{table}
\begin{table}
\setlength\tabcolsep{8.4pt}
\renewcommand{\arraystretch}{0.95}
\tiny
\newcommand{\tabincell}[2]{\begin{tabular}{@{}#1@{}}#2\end{tabular}}
\caption{\label{TabMainFullIndResults}
The overall performance of various methods on FullInd datasets. The colored cells represent the \colorbox{myRed!80}{best}, \colorbox{myOrange!80}{second-best}, and \colorbox{myYellow!80}{third-best} values, respectively.
}
\centering
\begin{threeparttable}
\begin{tabular}{ c  c  c  c  c  c  c  c  c  c  c  c  c  c  c  c  c }
\hline
\multirow{2}{*}[-0.05in]{\tabincell{c}{\textbf{Methods}}} & \multicolumn{2}{c}{\raisebox{-0.5ex}{\textbf{FB15k237-25}}} & \multicolumn{2}{c}{\raisebox{-0.5ex}{\textbf{FB15k237-50}}} & \multicolumn{2}{c}{\raisebox{-0.5ex}{\textbf{FB15k237-75}}} & \multicolumn{2}{c}{\raisebox{-0.5ex}{\textbf{FB15k237-100}}}\\
\cmidrule(r){2-3}\cmidrule(r){4-5}\cmidrule(r){6-7}\cmidrule(r){8-9}
 & \textbf{MRR}$\uparrow$ & \textbf{Hit10}$\uparrow$ & \textbf{MRR}$\uparrow$ & \textbf{Hit10}$\uparrow$ & \textbf{MRR}$\uparrow$ & \textbf{Hit10}$\uparrow$ & \textbf{MRR}$\uparrow$ & \textbf{Hit10}$\uparrow$\\
\hline
NBFNet & 0.224 & 0.410 & 0.130 & 0.259 & 0.089 & 0.166 & 0.072 & 0.154\\
RED-GNN & 0.145 & 0.284 & 0.129 & 0.251 & 0.107 & 0.201 & 0.121 & 0.263\\
ULTRA & 0.383 & 0.635 & 0.334 & 0.538 & 0.400 & 0.598 & 0.444 & 0.643\\
MOTIF & \cellcolor{myYellow!80}\textbf{0.388} & 0.635 & \cellcolor{myYellow!80}\textbf{0.340} & \cellcolor{myYellow!80}\textbf{0.544} & 0.399 & 0.607 & 0.439 & 0.642\\
\hline
ChatRule\tnote{$\dagger$} & 0.336 & 0.602 & 0.289 & 0.514 & 0.250 & 0.507 & 0.361 & 0.554\\
KRLM\tnote{$\dagger$} & \cellcolor{myRed!80}\textbf{0.398} & \cellcolor{myYellow!80}\textbf{0.640} & \cellcolor{myOrange!80}\textbf{0.345} & \cellcolor{myRed!80}\textbf{0.552} & \cellcolor{myOrange!80}\textbf{0.414} & \cellcolor{myYellow!80}\textbf{0.620} & \cellcolor{myRed!80}\textbf{0.455} & \cellcolor{myOrange!80}\textbf{0.655}\\
\hline
$\text{SIRLM}_{\text{E2E}}$ & 0.386 & 0.635 & \cellcolor{myOrange!80}\textbf{0.345} & 0.540 & \cellcolor{myYellow!80}\textbf{0.411} & \cellcolor{myOrange!80}\textbf{0.623} & 0.444 & 0.646\\
$\text{SIRLM}_{\text{PT}}$ & 0.376 & 0.631 & 0.315 & 0.534 & 0.396 & 0.613 & 0.442 & \cellcolor{myOrange!80}\textbf{0.655}\\
$\text{SIRLM}_{\text{SFT}}$ & \cellcolor{myYellow!80}\textbf{0.388} & \cellcolor{myOrange!80}\textbf{0.647} & \cellcolor{myOrange!80}\textbf{0.345} & 0.541 & 0.400 & 0.610 & \cellcolor{myYellow!80}\textbf{0.448} & \cellcolor{myYellow!80}\textbf{0.648}\\
$\text{SIRLM}_{\text{GRPO}}$ & \cellcolor{myOrange!80}\textbf{0.396} & \cellcolor{myRed!80}\textbf{0.649} & \cellcolor{myRed!80}\textbf{0.347} & \cellcolor{myOrange!80}\textbf{0.546} & \cellcolor{myRed!80}\textbf{0.420} & \cellcolor{myRed!80}\textbf{0.626} & \cellcolor{myOrange!80}\textbf{0.449} & \cellcolor{myRed!80}\textbf{0.656}\\
\hline
\textcolor{upColor}{\textit{Avg. gain}}\tnote{$\star$} & \textcolor{upColor}{\textit{+8.37\%}} & \textcolor{upColor}{\textit{+11.47\%}} & \textcolor{upColor}{\textit{+8.58\%}} & \textcolor{upColor}{\textit{+10.30\%}} & \textcolor{upColor}{\textit{+14.35\%}} & \textcolor{upColor}{\textit{+17.62\%}} & \textcolor{upColor}{\textit{+13.37\%}} & \textcolor{upColor}{\textit{+17.08\%}}\\
\hline\hline
\multirow{2}{*}[-0.05in]{\tabincell{c}{\textbf{Methods}}} & \multicolumn{2}{c}{\raisebox{-0.5ex}{\textbf{NELL995-25}}} & \multicolumn{2}{c}{\raisebox{-0.5ex}{\textbf{NELL995-50}}} & \multicolumn{2}{c}{\raisebox{-0.5ex}{\textbf{NELL995-75}}} & \multicolumn{2}{c}{\raisebox{-0.5ex}{\textbf{NELL995-100}}}\\
\cmidrule(r){2-3}\cmidrule(r){4-5}\cmidrule(r){6-7}\cmidrule(r){8-9}
 & \textbf{MRR}$\uparrow$ & \textbf{Hit10}$\uparrow$ & \textbf{MRR}$\uparrow$ & \textbf{Hit10}$\uparrow$ & \textbf{MRR}$\uparrow$ & \textbf{Hit10}$\uparrow$ & \textbf{MRR}$\uparrow$ & \textbf{Hit10}$\uparrow$\\
\hline
NBFNet & 0.283 & 0.417 & 0.225 & 0.346 & 0.137 & 0.255 & 0.096 & 0.199\\
RED-GNN & 0.214 & 0.266 & 0.179 & 0.115 & 0.203 & 0.353 & 0.212 & 0.385\\
ULTRA & \cellcolor{myYellow!80}\textbf{0.407} & 0.596 & \cellcolor{myOrange!80}\textbf{0.418} & \cellcolor{myOrange!80}\textbf{0.595} & \cellcolor{myOrange!80}\textbf{0.374} & \cellcolor{myOrange!80}\textbf{0.570} & 0.458 & \cellcolor{myOrange!80}\textbf{0.684}\\
MOTIF & 0.390 & 0.580 & \cellcolor{myYellow!80}\textbf{0.414} & 0.573 & 0.360 & 0.548 & 0.464 & \cellcolor{myYellow!80}\textbf{0.682}\\
\hline
ChatRule\tnote{$\dagger$} & 0.359 & 0.564 & 0.368 & 0.540 & 0.302 & 0.521 & 0.444 & 0.608\\
KRLM\tnote{$\dagger$} & 0.401 & 0.596 & \cellcolor{myRed!80}\textbf{0.432} & \cellcolor{myRed!80}\textbf{0.598} & 0.367 & 0.559 & \cellcolor{myRed!80}\textbf{0.489} & \cellcolor{myRed!80}\textbf{0.688}\\
\hline
$\text{SIRLM}_{\text{E2E}}$ & \cellcolor{myYellow!80}\textbf{0.407} & \cellcolor{myYellow!80}\textbf{0.601} & 0.413 & \cellcolor{myOrange!80}\textbf{0.595} & \cellcolor{myYellow!80}\textbf{0.371} & \cellcolor{myOrange!80}\textbf{0.570} & \cellcolor{myOrange!80}\textbf{0.477} & 0.666\\
$\text{SIRLM}_{\text{PT}}$ & 0.394 & 0.593 & 0.394 & \cellcolor{myYellow!80}\textbf{0.581} & 0.355 & 0.550 & \cellcolor{myYellow!80}\textbf{0.472} & 0.668\\
$\text{SIRLM}_{\text{SFT}}$ & \cellcolor{myRed!80}\textbf{0.412} & \cellcolor{myRed!80}\textbf{0.613}& \cellcolor{myOrange!80}\textbf{0.418} & \cellcolor{myRed!80}\textbf{0.598} & \cellcolor{myRed!80}\textbf{0.379} & \cellcolor{myRed!80}\textbf{0.578} & \cellcolor{myRed!80}\textbf{0.489} & 0.667\\
$\text{SIRLM}_{\text{GRPO}}$ & \cellcolor{myOrange!80}\textbf{0.410} & \cellcolor{myOrange!80}\textbf{0.607}  & 0.382 & 0.567 & 0.368 & \cellcolor{myYellow!80}\textbf{0.560} & \cellcolor{myOrange!80}\textbf{0.477} & 0.668\\
\hline
\textcolor{upColor}{\textit{Avg. gain}}\tnote{$\star$} & \textcolor{upColor}{\textit{+6.97\%}} & \textcolor{upColor}{\textit{+10.98\%}} & \textcolor{upColor}{\textit{+7.87\%}} & \textcolor{upColor}{\textit{+13.68\%}} & \textcolor{upColor}{\textit{+8.85\%}} & \textcolor{upColor}{\textit{+11.03\%}} & \textcolor{upColor}{\textit{+12.85\%}} & \textcolor{upColor}{\textit{+12.70\%}}\\
\hline\hline
\multirow{2}{*}[-0.05in]{\tabincell{c}{\textbf{Methods}}} & \multicolumn{2}{c}{\raisebox{-0.5ex}{\textbf{Wikidata68K-25}}} & \multicolumn{2}{c}{\raisebox{-0.5ex}{\textbf{Wikidata68K-50}}} & \multicolumn{2}{c}{\raisebox{-0.5ex}{\textbf{Wikidata68K-75}}} & \multicolumn{2}{c}{\raisebox{-0.5ex}{\textbf{Wikidata68K-100}}}\\
\cmidrule(r){2-3}\cmidrule(r){4-5}\cmidrule(r){6-7}\cmidrule(r){8-9}
 & \textbf{MRR}$\uparrow$ & \textbf{Hit10}$\uparrow$ & \textbf{MRR}$\uparrow$ & \textbf{Hit10}$\uparrow$ & \textbf{MRR}$\uparrow$ & \textbf{Hit10}$\uparrow$ & \textbf{MRR}$\uparrow$ & \textbf{Hit10}$\uparrow$\\
\hline
NBFNet & 0.154 & 0.301 & 0.062 & 0.105 & 0.072 & 0.172 & 0.014 & 0.026\\
RED-GNN & 0.170 & 0.263 & 0.058 & 0.093 & 0.172 & 0.290 & 0.096 & 0.136\\
ULTRA & 0.321 & 0.535 & 0.140 & 0.280 & 0.380 & 0.530 & 0.168 & 0.286\\
MOTIF & 0.317 & 0.505 & 0.160 & 0.304 & 0.371 & 0.535 & 0.173 & 0.284\\
\hline
ChatRule\tnote{$\dagger$} & 0.256 & 0.372 & 0.090 & 0.126 & 0.266 & 0.483 & 0.107 & 0.196\\
KRLM\tnote{$\dagger$} & \cellcolor{myRed!80}\textbf{0.332} & \cellcolor{myRed!80}\textbf{0.550} & 0.168 & \cellcolor{myYellow!80}\textbf{0.328} & 0.384 & 0.538 & \cellcolor{myOrange!80}\textbf{0.189} & \cellcolor{myOrange!80}\textbf{0.313}\\
\hline
$\text{SIRLM}_{\text{E2E}}$ & \cellcolor{myYellow!80}\textbf{0.322} & \cellcolor{myYellow!80}\textbf{0.536} & \cellcolor{myRed!80}\textbf{0.173} & \cellcolor{myOrange!80}\textbf{0.329} & \cellcolor{myOrange!80}\textbf{0.406} & \cellcolor{myRed!80}\textbf{0.568} & \cellcolor{myYellow!80}\textbf{0.186} & 0.304\\
$\text{SIRLM}_{\text{PT}}$ & 0.303 & 0.497 & 0.166 & 0.309 & \cellcolor{myYellow!80}\textbf{0.399} & \cellcolor{myYellow!80}\textbf{0.556} & \cellcolor{myOrange!80}\textbf{0.189} & 0.300\\
$\text{SIRLM}_{\text{SFT}}$ & \cellcolor{myOrange!80}\textbf{0.328} & \cellcolor{myOrange!80}\textbf{0.541} & \cellcolor{myOrange!80}\textbf{0.171} & \cellcolor{myRed!80}\textbf{0.330} & \cellcolor{myRed!80}\textbf{0.412} & \cellcolor{myOrange!80}\textbf{0.566} & \cellcolor{myRed!80}\textbf{0.191} & \cellcolor{myRed!80}\textbf{0.316}\\
$\text{SIRLM}_{\text{GRPO}}$ & 0.314 & 0.511 & \cellcolor{myYellow!80}\textbf{0.170} & 0.315 & 0.385 & 0.550 & 0.185 & \cellcolor{myYellow!80}\textbf{0.307}\\
\hline
\textcolor{upColor}{\textit{Avg. gain}}\tnote{$\star$} & \textcolor{upColor}{\textit{+6.97\%}} & \textcolor{upColor}{\textit{+12.00\%}} & \textcolor{upColor}{\textit{+5.80\%}} & \textcolor{upColor}{\textit{+12.40\%}} & \textcolor{upColor}{\textit{+13.78\%}} & \textcolor{upColor}{\textit{+14.33\%}} & \textcolor{upColor}{\textit{+6.65\%}} & \textcolor{upColor}{\textit{+10.92\%}}\\
\hline\hline
\multirow{2}{*}[-0.05in]{\tabincell{c}{\textbf{Methods}}} & \multicolumn{2}{c}{\raisebox{-0.5ex}{\textbf{MTDEA1-tax}}} & \multicolumn{2}{c}{\raisebox{-0.5ex}{\textbf{MTDEA1-health}}} & \multicolumn{2}{c}{\raisebox{-0.5ex}{\textbf{MTDEA2-org}}} & \multicolumn{2}{c}{\raisebox{-0.5ex}{\textbf{MTDEA2-sci}}}\\
\cmidrule(r){2-3}\cmidrule(r){4-5}\cmidrule(r){6-7}\cmidrule(r){8-9}
 & \textbf{MRR}$\uparrow$ & \textbf{Hit10}$\uparrow$ & \textbf{MRR}$\uparrow$ & \textbf{Hit10}$\uparrow$ & \textbf{MRR}$\uparrow$ & \textbf{Hit10}$\uparrow$ & \textbf{MRR}$\uparrow$ & \textbf{Hit10}$\uparrow$\\
\hline
NBFNet & 0.113 & 0.315 & 0.122 & 0.339 & \cellcolor{myRed!80}\textbf{0.109} & \cellcolor{myRed!80}\textbf{0.292} & 0.091 & 0.235 \\
ULTRA & 0.330 & 0.459 & 0.380 & 0.467 & \cellcolor{myYellow!80}\textbf{0.104} & \cellcolor{myOrange!80}\textbf{0.170} & 0.311 & 0.451\\
MOTIF & \cellcolor{myRed!80}\textbf{0.416} & 0.522 & \cellcolor{myYellow!80}\textbf{0.385} & \cellcolor{myYellow!80}\textbf{0.473} & \cellcolor{myOrange!80}\textbf{0.106} & \cellcolor{myOrange!80}\textbf{0.170} & 0.326 & 0.520 \\
\hline
ChatRule\tnote{$\dagger$} & 0.300 & 0.388 & 0.352 & 0.427 & 0.101 & 0.140 & 0.343 & 0.468 \\
KRLM\tnote{$\dagger$} & 0.378 & 0.508 & 0.376 & 0.457 & 0.098 & 0.140 & 0.361 & 0.526 \\
\hline
$\text{SIRLM}_{\text{E2E}}$ & 0.399 & \cellcolor{myYellow!80}\textbf{0.528} & \cellcolor{myOrange!80}\textbf{0.390} & \cellcolor{myOrange!80}\textbf{0.475} & 0.101 & 0.142 & 0.385 & 0.530\\
$\text{SIRLM}_{\text{PT}}$ & 0.395 & 0.523 & 0.383 & 0.470 & 0.083 & 0.136 & \cellcolor{myYellow!80}\textbf{0.390} & \cellcolor{myYellow!80}\textbf{0.532} \\
$\text{SIRLM}_{\text{SFT}}$ & \cellcolor{myOrange!80}\textbf{0.407} & \cellcolor{myRed!80}\textbf{0.533} & \cellcolor{myRed!80}\textbf{0.404} & \cellcolor{myRed!80}\textbf{0.481} & \cellcolor{myRed!80}\textbf{0.109} & \cellcolor{myYellow!80}\textbf{0.145} & \cellcolor{myRed!80}\textbf{0.410} & \cellcolor{myRed!80}\textbf{0.548}\\
$\text{SIRLM}_{\text{GRPO}}$ & \cellcolor{myYellow!80}\textbf{0.401} & \cellcolor{myOrange!80}\textbf{0.530} & 0.384 & \cellcolor{myYellow!80}\textbf{0.473} & 0.097 & 0.140 & \cellcolor{myOrange!80}\textbf{0.399} & \cellcolor{myOrange!80}\textbf{0.538}\\
\hline
\textcolor{upColor}{\textit{Avg. gain}}\tnote{$\star$} & \textcolor{upColor}{\textit{+9.96\%}} & \textcolor{upColor}{\textit{+9.46\%}} & \textcolor{upColor}{\textit{+8.10\%}} & \textcolor{upColor}{\textit{+4.84\%}} & \textcolor{upColor}{\textit{+0.54\%}} & \textcolor{upColor}{\textit{-3.74\%}} & \textcolor{upColor}{\textit{+12.36\%}} & \textcolor{upColor}{\textit{+10.80\%}}\\
\hline\hline
\multirow{2}{*}[-0.05in]{\tabincell{c}{\textbf{Methods}}} & \multicolumn{2}{c}{\raisebox{-0.5ex}{\textbf{MTDEA3-art}}} & \multicolumn{2}{c}{\raisebox{-0.5ex}{\textbf{MTDEA3-infra}}} & \multicolumn{2}{c}{\raisebox{-0.5ex}{\textbf{MTDEA4-sci}}} & \multicolumn{2}{c}{\raisebox{-0.5ex}{\textbf{MTDEA4-health}}}\\
\cmidrule(r){2-3}\cmidrule(r){4-5}\cmidrule(r){6-7}\cmidrule(r){8-9}
 & \textbf{MRR}$\uparrow$ & \textbf{Hit10}$\uparrow$ & \textbf{MRR}$\uparrow$ & \textbf{Hit10}$\uparrow$ & \textbf{MRR}$\uparrow$ & \textbf{Hit10}$\uparrow$ & \textbf{MRR}$\uparrow$ & \textbf{Hit10}$\uparrow$\\
\hline
NBFNet & 0.099 & 0.257 & 0.137 & 0.357 & 0.094 & 0.246 & 0.124 & 0.328\\
ULTRA & 0.306 & 0.473 & 0.657 & 0.807 & 0.303 & 0.478 & 0.704 & \cellcolor{myYellow!80}\textbf{0.785}\\
MOTIF & 0.315 & 0.469 & \cellcolor{myYellow!80}\textbf{0.683} & 0.827 & 0.309 & 0.483 & 0.703 & \cellcolor{myOrange!80}\textbf{0.787}\\
\hline
ChatRule\tnote{$\dagger$} & 0.267 & 0.432 & 0.659 & 0.811 & 0.268 & 0.449 & 0.658 & 0.740\\
KRLM\tnote{$\dagger$} & 0.315 & 0.468 & 0.670 & 0.811 & 0.303 & 0.471 & 0.695 & 0.776\\
\hline
$\text{SIRLM}_{\text{E2E}}$ & \cellcolor{myOrange!80}\textbf{0.328} & \cellcolor{myOrange!80}\textbf{0.499} & \cellcolor{myOrange!80}\textbf{0.686} & \cellcolor{myOrange!80}\textbf{0.834} & 0.311 & 0.489 & \cellcolor{myRed!80}\textbf{0.710} & \cellcolor{myRed!80}\textbf{0.789}\\
$\text{SIRLM}_{\text{PT}}$ & \cellcolor{myYellow!80}\textbf{0.324} & 0.492 & 0.677 & 0.827 & \cellcolor{myYellow!80}\textbf{0.317} & \cellcolor{myYellow!80}\textbf{0.497} & \cellcolor{myOrange!80}\textbf{0.707} & \cellcolor{myOrange!80}\textbf{0.787}\\
$\text{SIRLM}_{\text{SFT}}$ & \cellcolor{myOrange!80}\textbf{0.328} & \cellcolor{myYellow!80}\textbf{0.493} & \cellcolor{myRed!80}\textbf{0.688} & \cellcolor{myRed!80}\textbf{0.839} & \cellcolor{myRed!80}\textbf{0.330} & \cellcolor{myRed!80}\textbf{0.526} & 0.701& \cellcolor{myOrange!80}\textbf{0.787}\\
$\text{SIRLM}_{\text{GRPO}}$ &\cellcolor{myRed!80}\textbf{0.333} & \cellcolor{myRed!80}\textbf{0.510} & 0.676 & \cellcolor{myYellow!80}\textbf{0.830} & \cellcolor{myOrange!80}\textbf{0.321} & \cellcolor{myOrange!80}\textbf{0.514} & \cellcolor{myYellow!80}\textbf{0.706} & \cellcolor{myRed!80}\textbf{0.789}\\
\hline
\textcolor{upColor}{\textit{Avg. gain}}\tnote{$\star$} & \textcolor{upColor}{\textit{+7.26\%}} & \textcolor{upColor}{\textit{+9.02\%}} & \textcolor{upColor}{\textit{+12.68\%}} & \textcolor{upColor}{\textit{+11.64\%}} & \textcolor{upColor}{\textit{+7.46\%}} & \textcolor{upColor}{\textit{+10.06\%}} & \textcolor{upColor}{\textit{+13.32\%}} & \textcolor{upColor}{\textit{+10.58\%}}\\
\hline
\end{tabular}
\end{threeparttable}
\begin{tablenotes}
\item[$\dagger$] $\dagger$ We reproduce the experimental metrics of ChatRule and KRLM. 
\item[$\star$] $\star$ We calculate the average gain of the optimal value in the four training modes of SIRLM compared to all baselines.
\end{tablenotes}
\end{table}
\begin{table}
\setlength\tabcolsep{5pt}
\tiny
\newcommand{\tabincell}[2]{\begin{tabular}{@{}#1@{}}#2\end{tabular}}
\caption{\label{TabAblationResults}
The overall performance of ablation variants on different datasets, where the MRR and Hit10 in the IndE and FullInd scenarios are summarized as average values. The colored cells represent the \colorbox{myRed!80}{best}, \colorbox{myOrange!80}{second-best}, and \colorbox{myYellow!80}{third-best} values, respectively.
}
\centering
\begin{threeparttable}
\begin{tabular}{ c | c  c  c  c  c  c  c  c | c  c  c  c }
\hline
\multirow{2}{*}[-0.05in]{\tabincell{c}{\textbf{Methods}}} & \multicolumn{2}{c}{\raisebox{-0.5ex}{\textbf{FB15k237}}} & \multicolumn{2}{c}{\raisebox{-0.5ex}{\textbf{CodeX-M}}} & \multicolumn{2}{c}{\raisebox{-0.5ex}{\textbf{WN18RR}}} & \multicolumn{2}{c}{\raisebox{-0.5ex}{\textbf{NELL995}}} & \multicolumn{2}{|c}{\raisebox{-0.5ex}{\textbf{12 IndE Datasets}}} & \multicolumn{2}{c}{\raisebox{-0.5ex}{\textbf{20 FullInd Datasets}}} \\
\cmidrule(r){2-3}\cmidrule(r){4-5}\cmidrule(r){6-7}\cmidrule(r){8-9}\cmidrule(r){10-11}\cmidrule(r){12-13}
 & \textbf{MRR} & \textbf{Hit10} & \textbf{MRR} & \textbf{Hit10} & \textbf{MRR} & \textbf{Hit10} & \textbf{MRR} & \textbf{Hit10} & \textbf{MRR} & \textbf{Hit10} & \textbf{MRR} & \textbf{Hit10}\\
\hline
\textbf{Full Model} & \cellcolor{myRed!80}\textbf{0.408} & \cellcolor{myRed!80}\textbf{0.593} & \cellcolor{myRed!80}\textbf{0.367} & \cellcolor{myRed!80}\textbf{0.522} & \cellcolor{myRed!80}\textbf{0.461} & \cellcolor{myRed!80}\textbf{0.556} & \cellcolor{myRed!80}\textbf{0.520} & \cellcolor{myRed!80}\textbf{0.638} & \cellcolor{myRed!80}\textbf{0.580} & \cellcolor{myRed!80}\textbf{0.738} & \cellcolor{myRed!80}\textbf{0.374} & \cellcolor{myRed!80}\textbf{0.538}\\
\hline
w/o \textbf{MMQI} & \cellcolor{myOrange!80}\textbf{0.389} & \cellcolor{myOrange!80}\textbf{0.579} & \cellcolor{myOrange!80}\textbf{0.360} & \cellcolor{myOrange!80}\textbf{0.499} & \cellcolor{myOrange!80}\textbf{0.458} & \cellcolor{myYellow!80}\textbf{0.525} & \cellcolor{myOrange!80}\textbf{0.514} & \cellcolor{myOrange!80}\textbf{0.627} & \cellcolor{myOrange!80}\textbf{0.570} & \cellcolor{myOrange!80}\textbf{0.721} & \cellcolor{myOrange!80}\textbf{0.364} & \cellcolor{myOrange!80}\textbf{0.522}\\
w/o \textbf{SRM} & \cellcolor{myYellow!80}\textbf{0.358} & \cellcolor{myYellow!80}\textbf{0.566} & \cellcolor{myYellow!80}\textbf{0.348} & \cellcolor{myYellow!80}\textbf{0.478} & \cellcolor{myYellow!80}\textbf{0.455} & 0.526 & 0.488 & 0.606 & \cellcolor{myYellow!80}\textbf{0.558} & \cellcolor{myYellow!80}\textbf{0.709} & \cellcolor{myYellow!80}\textbf{0.348} & \cellcolor{myYellow!80}\textbf{0.508}\\
w/o \textbf{RCMP} & 0.352 & 0.538 & 0.339 & 0.422 & 0.404 & 0.500 & \cellcolor{myYellow!80}\textbf{0.501} & \cellcolor{myYellow!80}\textbf{0.611} & 0.540 & 0.685 & 0.344 & 0.502\\
\hline
\end{tabular}
\end{threeparttable}
\end{table}
Tables~\ref{TabMainIndEResults} and~\ref{TabMainFullIndResults} correspond to the detailed experimental results of each method in Table~\ref{TabMainResults} on the IndE and FullInd datasets, respectively.
\par Traditional embedding-based methods rely on initializing fixed representations of entities and relations tailored to a specific KGR scenario. 

\begin{wrapfigure}[15]{r}{0.55\linewidth}
    \vspace*{-10pt}  
    \begin{centering}
      \includegraphics[width=\linewidth]{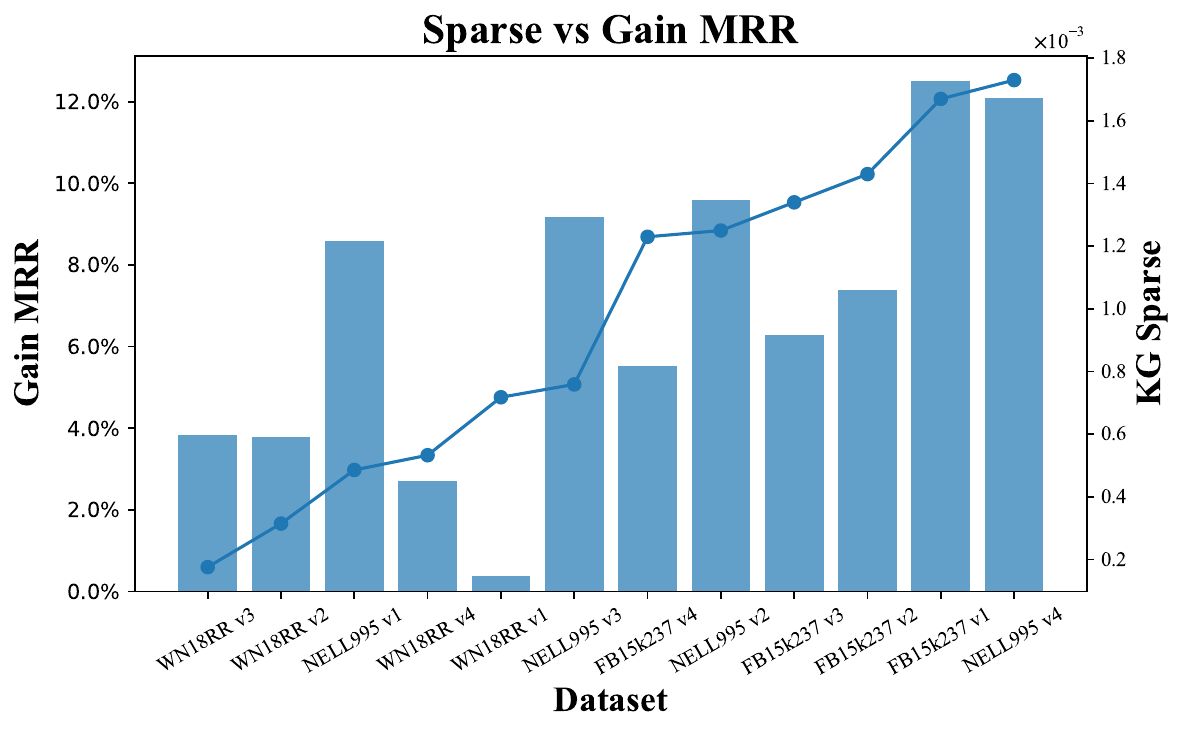}
\vspace{-7mm}
      \caption{\label{FigSparseGains}The correlation trend between KG sparsity and the MRR gain of SIRLM.}
    \end{centering}
\end{wrapfigure}
As a result, they cannot accommodate unseen entities and relations in the inference KG, making them unsuitable for inductive KGR tasks. Rule-based approaches, while capable of handling predictions involving unseen entities under a fixed set of relation types, fail to effectively generalize rule bodies when relations are dynamic, which prevents their application to FullInd KGR settings. In contrast, GNN-based methods are applicable across all KGR scenarios considered in our experiments. This advantage stems from their ability to learn structural context, enabling more structural pattern induction over the entire KG and allowing them to identify unfamiliar entities and relations through representations grounded in structural invariance.
\par LLM-based methods leverage pre-trained knowledge and strong contextual reasoning capabilities to uncover latent facts in incomplete KGs and to recognize unseen entities and relations. For instance, ChatRule utilizes the natural language understanding ability of LLMs to filter candidate textualized KG rules and select those that best support reasoning over a query triplet, while MKGL and KRLM project explicit KG structures into the implicit parametric knowledge of LLMs and ultimately generate potential facts through next-entity prediction. However, these approaches lack faithfulness constraints on the implicit reasoning process of LLMs, making them susceptible to biases arising from misalignment between structural knowledge and parametric representations, which can lead to erroneous inference outcomes.
\par In contrast, our SIRLM enforces the generation of structurally grounded rules over the KG, mitigating the tendency of LLMs to overlook the underlying structural evidence in favor of pre-trained knowledge when generating answers. Furthermore, we observe a correlation between the performance gains of SIRLM and the sparsity of the KG. Based on the results in Tables~\ref{TabMainIndEResults} and~\ref{TabMainFullIndResults}, Figure~\ref{FigSparseGains} illustrates this trend. Notably, SIRLM performs slightly worse on sparse KGs (\textit{e.g.}, the WN18RR series dataset) than on dense ones, which highlights an inherent limitation of structural rule-based reasoning under sparse conditions and points to an important direction for future research.
\subsection{Details Ablation Analysis}
\label{appendixMainResultsAbl}
Section~\ref{Ablation Experiments} analyzes the ablation components of SIRLM. To verify the universality of each component, we chose to perform ablation variants in $\text{SIRLM}_{\text{PT}}$ training mode. We first provide the design details of each ablation variant.
\begin{enumerate}[left=0pt]
\item[$\bullet$] \textbf{-MMQI}. This variant only provides the textual instructions of query triplets, without including structural representations of entities and relations. Following MKGL~\citep{MKGL}, such instruction sequences contain textual strings of entities/relations along with their corresponding textual explanations. During tokenization, we use the principal neighborhood aggregation method~\citep{PNA} to aggregate the text token sequences of entity/relation into word-level token embeddings, which are then used for subsequent rule generation and reasoning. Since the instructions do not include structural representations, the tokenizer in Eq. (\ref{eq6}) degenerates into $\text{TKN}_{\text{LLM}}$. The rest of the SIRLM architecture remains unchanged.
\item[$\bullet$] \textbf{-SRM}. This variant removes the structural relation memory in the in-context learning module, along with its corresponding linear layers ($m_K(\cdot)$ and $m_V(\cdot)$) in Eq. (\ref{eq8}). All other components remain consistent with the primary SIRLM.
\item[$\bullet$] \textbf{-RCMP}. This variant removes ${\{<r_z, f_{\text{down}}(\bm{h}_{L+z})>\}}_{z=1}^\varepsilon$ in Eq. (\ref{eq11}), reverting it back to the original SIL module.
\end{enumerate}
Table~\ref{TabAblationResults} provides the details ablation experimental results of each ablation variant.
\subsection{Additional Analysis on Different LLM Backbones}
\label{appendixLLMBackboneAnalysis}
Section~\ref{LLMBackboneAnalysis} presents the experimental results of SIRLM across four different LLM backbones. Overall, the results show a positive correlation between the parameter scale of the backbone and its reasoning performance. We attribute this to the fact that larger parameter sizes enable the model to more quickly fit the newly injected structured representation space.
\par To further illustrate this, Figure~\ref{FigLLMBackboneAnalysis} reports the convergence of rule generation accuracy for the four LLM backbones on WN18RR v1. It is evident that the two 7B backbones (Qwen2.5-7b and Llama2-7b) approach convergence at around 200 steps, whereas Qwen2.5-0.5b and Qwen2.5-1.5b backbones require approximately 400 steps to reach a similar level of convergence.
\begin{figure}
    \begin{centering}
      \includegraphics[width=\linewidth]{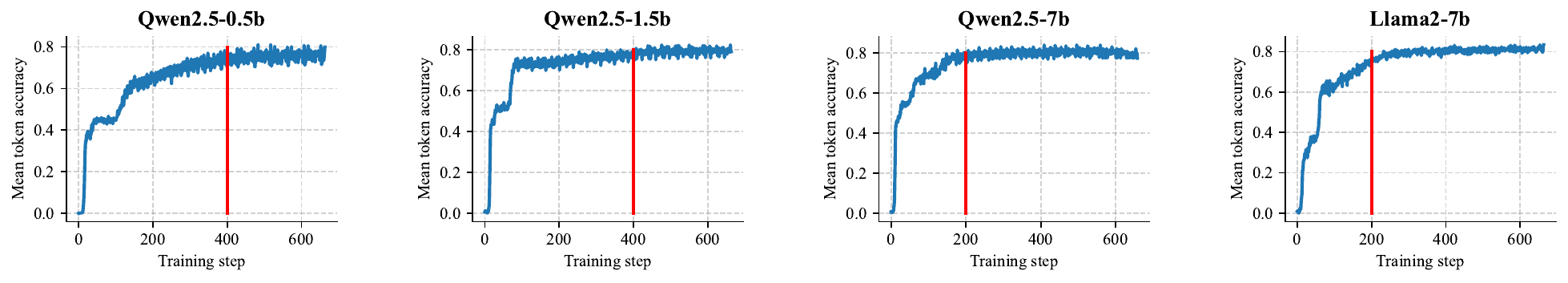}
\vspace{-6mm}
      \caption{\label{FigLLMBackboneAnalysis}Convergence of rule generation accuracy for SIRLM with different LLMs on WN18RR v1.}
    \end{centering}
\vspace{-4mm}
\end{figure}
\subsection{Case Study and Error Analysis}
\label{appendixCase}
\begin{figure}
    \begin{centering}
      \includegraphics[width=\linewidth]{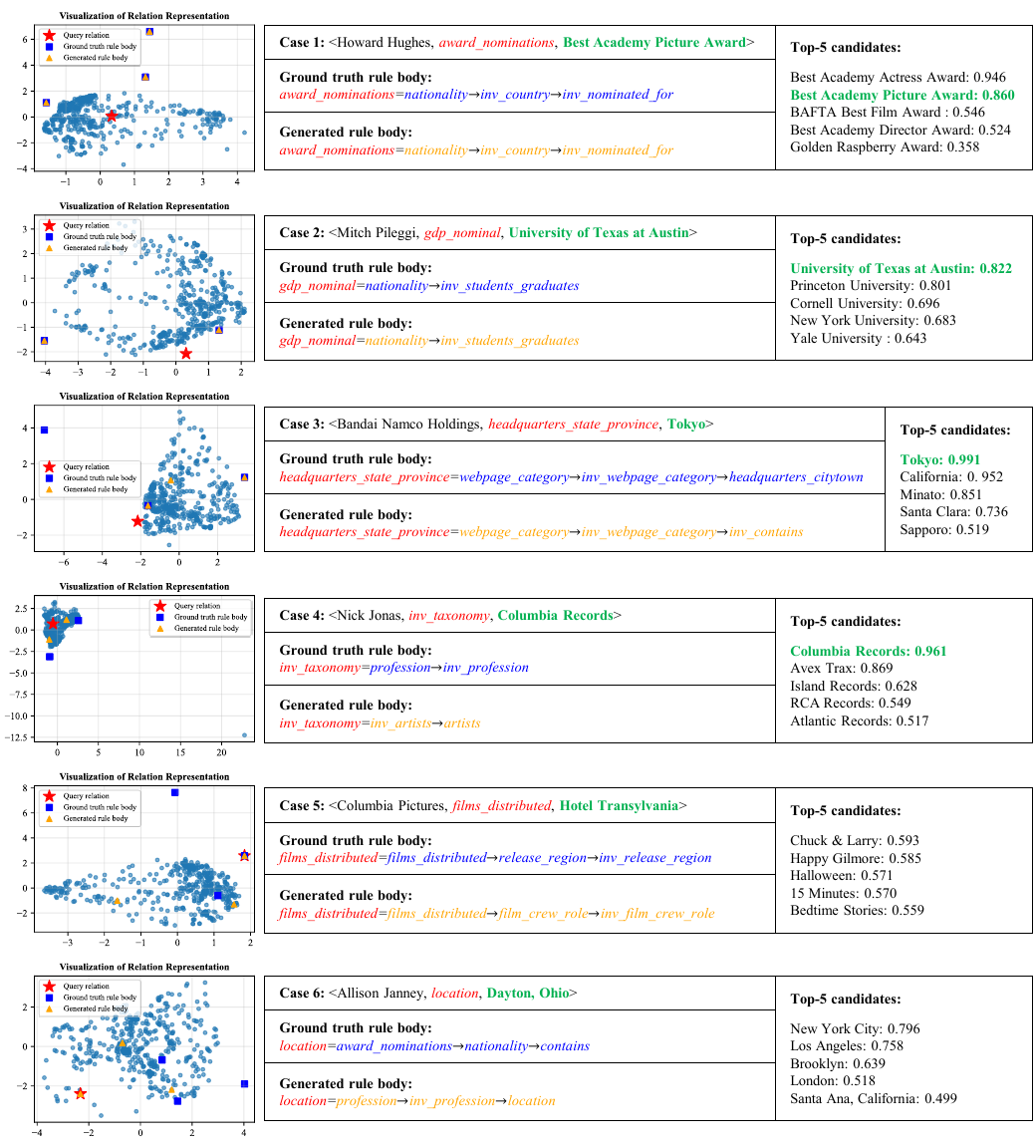}
\vspace{-6mm}
      \caption{\label{FigCaseStudy}Case studies of rule generation. \textit{\textcolor{red}{Red}}, \textit{\textcolor{blue}{blue}}, and \textit{\textcolor{orange}{orange}} denote the query relation, the atomic relations in the ground-truth rule, and the atomic relations in the predicted rule, respectively. \textbf{\textcolor{green}{Green}} denotes the target entities that need to be predicted.}
    \end{centering}
\vspace{-4mm}
\end{figure}
To better understand the behavior of our SIRLM, we visualize several representative cases of rule generation in Figure~\ref{FigCaseStudy}. \textbf{Cases 1 and 2} is a successful example where the generated rule exactly matches the ground-truth rule, indicating that the model can accurately identify the correct multi-hop reasoning pattern. As a result, the correct entities “\textbf{\textcolor{green}{Best Academy Picture Award}}” and “\textbf{\textcolor{green}{University of Texas at Austin}}” appears in the top-ranked candidates with a relatively high confidence score. 
\par \textbf{Cases 3 and 4} demonstrates that an exact match to the ground-truth rule is not necessary for correct prediction. Although the generated rule deviates from the annotated reasoning path, it still leads to the correct answers “\textbf{\textcolor{green}{Tokyo}}” and “\textbf{\textcolor{green}{Columbia Records}}”. This suggests that the model can exploit alternative reasoning paths with similar structural semantics. In these case, the generated atomic relations (\textcolor{blue}{blue square}) differ from the ground-truth atomic relations (\textcolor{orange}{orange triangle}), but remain semantically similar in the representation space.
\par In contrast, \textbf{Cases 5 and 6} reveals a limitation of the method. The generated rule introduces irrelevant atomic relations that differ substantially from the ground truth, causing the reasoning process to deviate from the correct semantic direction and resulting in incorrect top-ranked candidates.
\par These observations indicate that the effectiveness of SIRLM critically depends on the quality of the generated rule body, that is, semantically accurate rules improve both interpretability and predictive performance, whereas semantically inconsistent rules can significantly degrade performance.
\section{Limitations and Future Work}
\label{appendixLimitations}
SIRLM provides a novel modeling and training framework for LLM-based KGR research, effectively alleviating the problem of reasoning evidence perception drift caused by the knowledge representation gap between LLMs and KGs. However, SIRLM still has several potential limitations. We discuss these limitations as follows and provide potential future research directions.
\par \textbf{Computational complexity.} As shown in \textbf{Appendix~\ref{appendixComplexity}}, SIRLM needs to make real-time relational graph updates and message passing over the entire KG during training. Although sparse operators are adopted in practice to significantly reduce the computational cost of graph processing, the upper bound remains quadratic in the number of entities, which leads to a notable computational bottleneck for large-scale KGs. Therefore, in future work, we plan to introduce a query-driven evidence graph extraction mechanism~\citep{Relink}, such that the graph structure processed by SIRLM is closely aligned with the query triplets. This would help eliminate the unnecessary computational overhead incurred by processing irrelevant KG context.
\par \textbf{Sparse KG reasoning.} The comprehensive analysis of Figures 7 and 8, as well as Tables 8 and 9, shows that there is still significant room for improvement in the reasoning performance of SIRLM on sparse KGs. We consider that this limitation is attributed to the fact that SIRLM is a rule-driven KGR framework. In sparse KGs, the diversity and completeness of structural rules are inherently limited, which in turn hinders the subsequent processes of rule learning and generation by LLMs. To address this issue, we plan to incorporate a soft rule mechanism~\citep{RNNLogic} into the rule mining stage. Specifically, instead of relying solely on hard inductive methods based on observable closed-path patterns, we will jointly employ soft rule mining to construct approximate rules and corresponding soft labels for query triplets that lack closed-path evidence. These soft signals can serve as additional supervision for subsequent LLM-based rule generation. In this way, the contextual exploration space of LLMs over sparse KGs can be effectively expanded, thereby alleviating the limitations in rule learning caused by insufficient structural context.

\end{document}